\documentclass[10pt]{article}

\usepackage[NumSec,NumType]{amsthm_FR}
\usepackage{stylefile_FR}
\usepackage{tikz-cd}
\usepackage{graphicx}
\usepackage{enumerate}
\usepackage{longtable}

\newcommand{\tnlpt}{\hline&&\\[-10pt]}
\newcommand{\tnlpb}{\\[1pt]\hline}
\newcommand{\tnlp}{\\[1pt]\hline&&\\[-10pt]}

\newcommand{\selfmap}{\mathbin{\scalebox{.85}{%
\lefteqn{\scalebox{.5}{$\blacktriangleleft$}}\raisebox{.34ex}{$\supset$}}}}

\newcommand{\CondInd}[1]{\pperp_{#1}}

\newcommand{\ExpS}[2]{\mathbb{E}_{#1}\left[#2\right]}
\newcommand{\ExpC}[3]{\mathbb{E}_{#1}\left[#2\,|\,#3\right]}

\newcommand{\Per}[1]{\mathrm{Per}(#1)}
\newcommand{\APer}[1]{\mathrm{Aper}(#1)}
\newcommand{\SPer}[2]{\mathrm{Sync\text{-}Per}(#1,#2)}

\newcommand{\Pin}{\Pc_{\mathrm{in}}}
\newcommand{\Pstate}{\Pc_{\mathrm{state}}}
\newcommand{\Pout}{\Pc_{\mathrm{out}}}

\newcommand{\Pcausal}{P^{\mathrm{causal}}(\XS^- \times \US^-)}

\newcommand{\DetSol}{\Sc^{\mathrm{det}}}
\newcommand{\StochSol}{\Sc^{\mathrm{stoch}}}
\newcommand{\StochCausalSol}{\Sc^{\mathrm{c,stoch}}}

\newcommand{\DetOut}{\Oc^{\mathrm{det}}}
\newcommand{\StochOut}{\Oc^{\mathrm{stoch}}}
\newcommand{\StochCausalOut}{\Oc^{\mathrm{c,stoch}}}

\title{Stochastic dynamics learning with state-space systems}
\author{Juan-Pablo Ortega\footnote{Division of Mathematical Sciences, School of Physical and Mathematical Sciences, Nanyang Technological University, Singapore} \qquad Florian Rossmannek$^*$}
\Headings{Stochastic dynamics learning with state-space systems}{Ortega and Rossmannek}

\begin{document}

\maketitle

\begin{abstract}
	This work advances the theoretical foundations of reservoir computing (RC) by providing a unified treatment of fading memory and the echo state property (ESP) in both deterministic and stochastic settings. We investigate state-space systems, a central model class in time series learning, and establish that fading memory and solution stability hold generically --- even in the absence of the ESP --- offering a robust explanation for the empirical success of RC models without strict contractivity conditions. In the stochastic case, we critically assess stochastic echo states, proposing a novel distributional perspective rooted in attractor dynamics on the space of probability distributions, which leads to a rich and coherent theory. Our results extend and generalize previous work on non-autonomous dynamical systems, offering new insights into causality, stability, and memory in RC models. This lays the groundwork for reliable generative modeling of temporal data in both deterministic and stochastic regimes.
\end{abstract}

\section{Introduction}

Many tasks arising in engineering and the sciences involve modeling time series data, that is, data points acquired at different points in time over a certain period.
Such datasets could be gathered, for example, from observations of a biological or physical system, measurements of patients' clinical data over time, or prices of stocks, among many others.
This classical task has been treated comprehensively over the last decades, initiated with the introduction of the modern computer.
Lately, those classical methods have been replaced by machine learning methods in many applications \cite{LeCun2015,Schmidhuber2015}.

Broadly speaking, there are two main tasks one may be interested in with a given temporal dataset.
The first one is prediction \cite{HewamalageEtal2021,LimZohren2021}:
can we build a model that processes the temporal data and then predicts future data points (a weather forecast \cite{ArcomanoEtal2022,WiknerEtal2020Chaos}, prognosis of a patient's condition \cite{KourouEtal2015}, next day's stock price \cite{FrancqZakoian2019}, next quarter's GDP of a country \cite{BallarinEtal2024}, etc.\ \cite{LangkvistEtal2014})?
Forecasting requires a model to extract precise features of the temporal dataset to understand how the evolution is governed.
The second task is generation \cite{Graves2014,LuEtal2025}:
can our model generate new temporal data that behaves the same way as the original dataset without being an exact replica?

Data generation has various use cases.
If we can generate reliable data, we can conduct simulations.
There, we are not interested in the specific trajectory seen in the original dataset.
Instead, we intend to recover (physical or stochastic) features of the true underlying system.
With those, we are then able to simulate, say, different weather and climate scenarios \cite{ArcomanoEtal2022,WiknerEtal2020Chaos}.
Scenario simulation enables us to infer factors and confounders of the system.
Concretely, tweaking variables of the model (representing actual variables of the application) and conducting simulations with each constellation uncovers the influence of those variables on the output.
For example, running simulations on a patient's prognosis with varying parameters representing the prescribed dose of a medication yields insights on the implications of raising or lowering it on the probability of recovery \cite{AllamEtal2021}.

Another use case of data generation is to enrich sparse datasets.
Many numerical methods, most notably some machine learning paradigms, are data-hungry:
they require to be trained on large datasets to achieve good performance.
If data is available scarcely, one may use data generation to simulate new synthetic data that inherits the same structure and features and can complement the original dataset.
Subsequently, the data-hungry methods can be run on the new dataset, comprised of the original and the synthetic data \cite{TaylorNitschke2018,WenEtal2021}.
Consider, for example, data gathered in a financial context.
Although historical data may be available in abundance, market conditions, say, 100 years ago do not reflect today's conditions.
This renders old historical data less useful, and the data-subset of more recent, useful data becomes small.
It has been identified as an important task in financial applications to enrich the resulting small dataset for training advanced models \cite{BuehlerEtal2020, TakahashiEtal2019}.

Mathematically, many of these tasks fall in the category of (stochastic) filtering and control problems.
An early model used in that context was that of a Kalman filter \cite{Kalman1960}.
The Kalman filter is a simple yet powerful model, which represents a special linear case of the more general family of models called state-space systems \cite{Kalman1960,Sontag1990}.
These state-space systems form the basis for a range of different models, including recent machine learning ones \cite{GuDao2023,JiangLiLiWang2023JML} for long sequence modeling.
A machine learning paradigm that has proved itself successful in a range of applications is reservoir computing (RC) \cite{JaegerHaas2004, VerstraetenEtal2007, Jaeger2010, Maass2011, MaassNatschMarkram2002}.
This paradigm is based on a feedback loop, which takes a new input at each time step and transforms it together with its previous state into a new state.
Then, one extracts outputs by an additional transformation of the new state.
RC boasts several advantages over competing models.
The reservoir performing the transformation can be implemented by a physical system, outsourcing large parts of the computational costs from a classical computer.
Various physical systems have been used successfully to this end, from ripples on the surface of water \cite{FernandoSojakka2003}, to optoelectronic circuits \cite{AppeltantEtAl2011,LargerEtAl2012}, to mechanical bodies \cite{Nakajima2020}.
The additional transformation that yields the output is usually taken to be linear.
This makes it computationally cheap to train and to execute as well as makes the training of the model scale well as the size of the dataset and the dimensionality of the problem grows \cite{JaegerHaas2004, Jaeger2010}.
These two features of RC become particularly desirable in face of the rapidly growing energy demands of other modern machine learning methods \cite{BogmansEtal2025,JiangSonneEtal2024}.

In addition to its practical benefits, RC enjoys a range of theoretical guarantees.
In modern day machine learning, theoretical guarantees have become ever more desirable due to the recent focus on reliability and explainability in artificial intelligence models.
In the context of RC specifically, two main aspects build the foundation for the supporting theory.
The first, like for classical neural network theory, are universal approximation results, which ensure that the learning task is indeed accurately solvable with an RC model \cite{RC6, RC7, RC12, RC20}.
The second aspect is a reservoir's ability to create reliable state responses to its inputs, commonly known as steady states or echo states \cite{ChuaGreen1976,BoydChua1985,Jaeger2010}.
The theory of echo states of a system is intimately linked to the model's behavior in how it handles its memory of the information it processed.
Referred to as `fading memory', the model is required to forget inputs from the distant past in order to create those echo states \cite{BoydChua1985}.
Folklore has linked the fading of the memory to several related notions describing how both past states' and past inputs' impact on the present state fade over time.
A detailed discussion of the relationship between those various notions can be found in \cite{RC32}.
In this work, we focus on echo states and on fading memory.

Even the notion of fading memory is not defined consistently throughout the literature \cite{BoydChua1985,Manjunath2020ProcA,RC9}.
All commonly accepted definitions are phrased in terms of continuity of the echo states as a function of the input sequence.
But speaking of continuity requires fixing a topology on the space of sequences, which does not admit a canonical choice.
Different choices lead to different notions of fading memory \cite{RC30}.
In this work, we pose a very general definition of fading memory, which captures most other definitions as special cases, including the original one introduced in the seminal work \cite{BoydChua1985}.

Over the last decades, a lot of effort has been invested in finding necessary and sufficient conditions that guarantee the existence of uniqueness of echo states --- the so-called echo state property (ESP) --- for various families of functions, such as echo state networks (ESNs) and state-affine systems \cite{BuehnerYoung2006,YildizJaegerKiebel2012,RC9,QRC1}.
These conditions typically boil down to contractivity requirements.
If the model design of, for example, an ESN takes those findings into consideration, then the mathematical theory gives guarantees for fading memory and, hence, on its successful implementation.
In this paper, we show that \textit{successful implementation can be expected} even if the model design does \textit{not} incorporate guarantees for the ESP.
Namely, we prove that \textit{state-space systems enjoy fading memory and stability in the number of solutions for generic inputs}, independent of the specific model design.
This marks a breakthrough in the mathematical theory for the learning of dynamics.
Systems with more than one solution have previously been studied with the notion of an echo index introduced in \cite{CeniEtal2020PhysD} but the results therein also relied on contractivity assumptions.
Our new result that fading memory holds generically explains in particular that reservoirs do not generate intrinsic chaotic dynamics.
\textit{The chaos exhibited by the reservoir system is solely injected through the chaos of the input dynamics} --- a property that ought to be expected to enable successful learning since intrinsic chaos would pose a serious obstruction.

The theory for RC is well established in the context of deterministic learning tasks.
However, many applications naturally generate stochastic data, which requires the theory to be adapted accordingly, in particular for our understanding of generative models.
Stochastic counterparts of established results have begun to appear recently \cite{RC28} but are far from complete.
This work aims to largely close the gap between the deterministic and the stochastic theory.

Our findings will reveal many similarities but also striking differences between the deterministic and the stochastic settings.
Deterministically, there is a natural way to define echo states as solutions to a state equation of the form $x_t = f(x_{t-1},u_t)$, where $f$ is the state map, $x_t$ are the states, and $u_t$ are the inputs.
These solutions have been linked to the global attractor of the dynamical system encoding the dynamics of iterating the state map \cite{ManjunathJaeger2013,Manjunath2020ProcA}.
In the deterministic case, these two points of view lead to the same mathematical object.

To define stochastic solutions, one considers stochastic processes that satisfy $(X_t,U_t)_t\!=\!(f(X_{t-1},U_t),U_t)_t$.
Here, it is crucial to consider the joint stochastic process since probabilistic dependence structures between inputs and states are one of the aspects that make the stochastic theory richer than the deterministic one.
Contrary to the deterministic case, we now face a choice: requiring the equality to hold in law or almost surely.
The former choice had been made in \cite{RC27,RC28}, whereas the latter choice is prevalent in time series analysis \cite{FrancqZakoian2019}.
In this work, we make a strong case that the latter choice is more natural.
We do so by generalizing the link between solutions and the global attractor of the dynamical system that now encodes the stochastic version of the state map (on the space of probability distributions).
We find that the dynamic point of view \textit{naturally leads to solutions defined by almost sure equality}.

The newly established link between solutions and attractors in the stochastic context enable us to leverage tools from dynamical systems theory.
If the system has the deterministic ESP and fading memory, then stochastic solutions are shown to be induced as push-forwards of the inputs under the deterministic solution map, generalizing a result of \cite{RC28}.
However, the stochastic theory becomes richer in that it admits solutions that elude a functional representation.
We prove a stochastic generalization of the fundamental deterministic result that the ESP implies fading memory on compact state spaces.
Then, refining the study of stochastic solutions, we discuss in detail the notion of causality, which captures the physical intuition that future inputs cannot directly affect present states.
For such causal solutions, we prove that fading memory ensures stability in the number of solutions even in the absence of the ESP.
Our proofs will be based on results about abstract dynamical systems.
From this, it will become clear exactly which are the key properties of state-space systems that enable the interplay of the ESP and fading memory.

We treat the dynamics of state-space systems similarly as in \cite{ManjunathJaeger2013, ManjunathJaeger2014}, greatly expanding and generalizing several of the results established therein.
In particular, \cite{ManjunathJaeger2013} was one of the first works to study state-space systems through the lens of abstract non-autonomous dynamical systems theory.
The subsequent work \cite{ManjunathJaeger2014} by the same authors as \cite{ManjunathJaeger2013} extends their framework to include stochastic inputs.
In this work, we promote a concurrent treatment of the stochastic case, focusing on a distributional point of view that enables clean proofs relying on the same dynamical systems results.
In this context, such a distributional point of view was first considered in \cite{RC27,RC28} and is improved upon in this work.

Finally, the analysis of state-space systems in stochastic filtering and control tasks has previously led to the consideration of fading memory \cite{KantasEtal2015, LindstenEtal2012}.
We relate our new finding on generic fading memory back to the classical filtering task as well as explore the meaning of causality of solutions in this context.

This paper is arranged as follows.
In \cref{sec_det_SSS,sec_stoch_SSS}, we present our findings on deterministic and stochastic state-space systems, respectively.
The proofs of the results in these sections are postponed to \cref{sec_proofs}.
\cref{sec_abstract} develops results on abstract dynamical systems, on which the proofs in \cref{sec_proofs} are based.
\cref{sec_conclusion} concludes.
\cref{app_sec_technical} contains additional technical lemmas, and \cref{sec_app_inference} discusses links to classical stochastic filtering.

\subsection*{Conventions}

We denote by $\Z$ the set of integers, by $\N$ the set of strictly positive integers, $\N_0 = \N \cup \{0\}$, and $\Z_- = \Z \backslash \N$.
We understand the size $\# A$ of a set $A$ as a number in $\N_0 \cup \{ \infty \}$, that is, `$\# A$' does not distinguish between countable infinite and uncountable infinite cardinality.
Given any Cartesian product $A \times B$, the letter $\pi$ with the subscript $A$ (or $B$) denotes the natural projection $\pi_A \colon A \times B \rightarrow A$ (or $\pi_B \colon A \times B \rightarrow B$).
Subsets of topological spaces and products of topological spaces are endowed with the subspace topology and the product topology, respectively, unless explicitly stated otherwise.
All topological spaces are endowed with their Borel sigma-algebra.
For any Hausdorff space $X$, we denote by $P(X)$ the set of all Radon probability measures.
In particular, if $X$ is a Radon space (e.g.\ Polish), then $P(X)$ coincides with the set of all Borel probability measures.
The following table provides an overview of our notation.

\newpage
\centerline{Table 1. Overview of notation.}
\begin{longtable}{lcl}
\tnlpt
Function & \phantom{spacing} Letter \phantom{spacing} & Domain and codomain
\tnlpb\tnlpt
State map & $f$ & $\Xc \times \Uc \rightarrow \Xc$
\tnlp
Readout & $h$ & $\Xc \rightarrow \Yc$
\tnlp
Extended state map & $\Fc$ & $\XS^- \times \US^- \rightarrow \Xc^{\Z_-} \times \Uc^{\Z_-}$
\tnlp
Extended readout & $H$ & $\XS^- \times \US^- \rightarrow \YS^- \times \US^-$
\tnlp
Dynamical system & $\varphi$ & $\XS^- \times \US \rightarrow \XS^- \times \US$
\tnlp
Right-shift operator & $T$ & $\US^- \selfmap$ or $\XS^- \selfmap$ or $\XS^- \times \US^- \selfmap$ or $\XS^- \times \US \selfmap$
\tnlp
Left-shift operator & $\sigma$ & $\US \rightarrow \US$
\tnlp
Truncation & $\tau$ & $\US \rightarrow \US^-$ or $\XS^- \times \US \rightarrow \XS^- \times \US^-$% or $\YS^- \times \US \rightarrow \YS^- \times \US^-$
\tnlp
Right-inv.\ of truncation & $j^+$ & $\XS^- \times \US^- \rightarrow \XS^- \times \US$
\tnlp
Inclusion & $\iota$ & $\XS^- \times \US^- \rightarrow \Xc^{\Z_-} \times \Uc^{\Z_-}$
\tnlp
Projection & $\pi_{\XS^-}$ & $\XS^- \times \US^- \rightarrow \XS^-$ or $\XS^- \times \US \rightarrow \XS^-$
\tnlp
Projection & $\pi_{\US^-}$ & $\XS^- \times \US^- \rightarrow \US^-$ or $\YS^- \times \US^- \rightarrow \US^-$
\tnlp
Projection & $\pi_{\US}$ & $\XS^- \times \US \rightarrow \US$
\tnlp
Deterministic attractor & $\Sc_{\varphi}$ & $\US \rightarrow 2^{\XS^- \times \US}$
\tnlp
Deterministic solutions & $\DetSol$ & $\US^- \rightarrow 2^{\XS^- \times \US^-}$
\tnlp
Deterministic outputs & $\DetOut$ & $\US^- \rightarrow 2^{\YS^- \times \US^-}$
\tnlp
Stoch.\ dynamical system & $\varphi_*$ & $\Pstate \rightarrow \Pstate$
\tnlp
Stoch.\ left-shift operator & $\sigma_*$ & $\Pin \rightarrow \Pin$
\tnlp
Stoch.\ right-shift operator & $T_*$ & $\Pin^- \selfmap$ or $\Pstate^- \selfmap$ or $\Pstate \selfmap$ or $P(\XS^- \times \US) \selfmap$
\tnlp
Stoch.\ projection & $(\pi_{\US^-})_*$ & $\Pstate^- \rightarrow \Pin^-$ or $\Pout^- \rightarrow \Pin^-$
\tnlp
Stoch.\ projection & $(\pi_{\US})_*$ & $\Pstate \rightarrow \Pin$
\tnlp
Stoch. attractor & $\Sc_{\varphi_*}$ & $\Pin \rightarrow 2^{\Pstate}$
\tnlp
(Causal) stoch.\ solutions & $\StochSol$, $\StochCausalSol$ & $\Pin^- \rightarrow 2^{\Pstate^-}$
\tnlp
(Causal) stoch.\ outputs & $\StochOut$, $\StochCausalOut$ & $\Pin^- \rightarrow 2^{\Pout^-}$
\tnlpb
\end{longtable}

\vspace*{0mm}

\section{Deterministic state-space systems}
\label{sec_det_SSS}

Throughout this section, let $\Uc$, $\Xc$, and $\Yc$ be Hausdorff spaces.
Sequences are denoted as underlined letters, e.g.\ $\Seq{u} = (\seq{u}{t})_{t \in \Z_-} \in \Uc^{\Z_-}$.
We denote the right-shift operator $(\seq{u}{t})_t \mapsto (\seq{u}{t-1})_t$ on all left- and bi-infinite sequence spaces by the letter $T$ and the left-shift operator $(\seq{u}{t})_t \mapsto (\seq{u}{t+1})_t$ on bi-infinite sequence spaces by $\sigma$.
Fix backwards shift-invariant subsets $\US^- \subseteq \Uc^{\Z_-}$, $\XS^- \subseteq \Xc^{\Z_-}$, and $\YS^- \subseteq \Yc^{\Z_-}$, that is, $T^{-1}(\US^-) = \US^-$ and likewise for $\XS^-$ and $\YS^-$, and fix a shift-invariant subset $\US \subseteq \Uc^{\Z}$ that is mapped surjectively onto $\US^-$ by the truncation, that is, $\sigma(\US) = \US$ and $\tau(\US) = \US^-$, where $\tau \colon \US \rightarrow \US^-$, $\Seq{u} \mapsto (\seq{u}{t})_{t \in \Z_-}$.
We will use the same letter $\tau$ to also denote the truncation $\XS^- \times \US \rightarrow \XS^- \times \US^-$, $(\Seq{x},\Seq{u}) \mapsto (\Seq{x},\tau(\Seq{u}))$ on the domain $\XS^- \times \US$.
The topologies on these sequence spaces are subject to the following conditions.

\begin{assumption}
\label{assmpt_det_SSS}
	The sets $\US^-$, $\XS^-$, $\YS^-$, and $\US$ are endowed with topologies that are at least as fine as the product topologies and that satisfy the following properties.
\begin{enumerate}[\upshape (i)]\itemsep=0em
\item
The right-shift operators on $\US^-$, $\XS^-$, and $\YS^-$ are continuous, and the left-shift operator on $\US$ is a homeomorphism.

\item
The truncation map $\US \rightarrow \US^-$ is continuous and admits a continuous right-inverse $\US^- \rightarrow \US$.

\item
The concatenation map $\XS^- \times \Xc \rightarrow \XS^-$, $(\Seq{x},x) \mapsto (\dots,\seq{x}{-1},\seq{x}{0},x)$ is continuous, as is the analogous map $\US^- \times \Uc \rightarrow \US^-$.

\item
For any $\Seq{u},\Seq{u}' \in \US^-$, the sequence $\Seq{u}^n := (\dots,\seq{u}{-1}',\seq{u}{0}',\seq{u}{-n},\dots,\seq{u}{0}) \in \US^-$ converges to $\Seq{u}$ as $n \rightarrow \infty$.
\end{enumerate}
\end{assumption}

\begin{example}
	Admissible topologies for $\US^-$ and $\US$ include the product topology and the topologies induced by weighted $\ell^p$-norms, $p \in [1,\infty]$.
The same topologies are admissible for $\XS^-$ as well as the topologies induced by unweighted $\ell^p$-norms, $p \in [1,\infty]$.
The topology induced by a weighted $\ell^{\infty}$-norm is the one used by Boyd and Chua in their seminal work on fading memory \cite{BoydChua1985}.
For further details on these particular choices of topology, we refer to \cite{RC9,RC30}.
\end{example}

We remark that all sequence spaces are guaranteed to be Hausdorff since the base spaces are Hausdorff and the topologies on the sequence spaces are, by hypothesis, at least as fine as the product topologies.

\subsection{Dynamics on sequence spaces}
\label{subsec_dyn_sequence}

Consider a continuous state map $f \colon \Xc \times \Uc \rightarrow \Xc$.
A {\bfi solution of the state equation} is a pair of sequences $(\Seq{x},\Seq{u}) \in \XS^- \times \US^-$ that satisfies $\seq{x}{t} = f(\seq{x}{t-1},\seq{u}{t})$ for all $t \in \Z_-$.
In this case, we also call $\Seq{x}$ a {\bfi solution for the input} $\Seq{u}$.
Denote the set of all solutions by (`det' for deterministic)
\begin{equation*}
	\DetSol
	:= \{ (\Seq{x},\Seq{u}) \in \XS^- \times \US^- \colon \seq{x}{t} = f(\seq{x}{t-1},\seq{u}{t}) \text{ for all } t \in \Z_- \}.
\end{equation*}
Also consider the dynamical system $\varphi \colon \XS^- \times \US \rightarrow \XS^- \times \US$ given by $\varphi(\Seq{x},\Seq{u}) = ((\Seq{x},f(\seq{x}{0},\seq{u}{1})),\sigma(\Seq{u}))$, which is an extension to the input sequence space of the reservoir flow considered in \cite{RC9}.
Note that $\varphi$ is continuous by \cref{assmpt_det_SSS}.
In an abstract dynamical systems sense, a {\bfi solution of} $\varphi$ is a point through which a bi-infinite orbit of $\varphi$ passes \cite{KloedenRasmussen2011}.
Since the right-shift operator $T$ is a left-inverse of $\varphi$, the set $\Sc_{\varphi}$ of solutions of $\varphi$ can be characterized as follows;
\begin{equation*}
	\Sc_{\varphi}
	= \{ (\Seq{x},\Seq{u}) \in \XS^- \times \US \colon \varphi^n(T^n(\Seq{x},\Seq{u})) = (\Seq{x},\Seq{u}) \text{ for all } n \in \N \}.
\end{equation*}
In fact, $\Sc_{\varphi} = \bigcap_{n \in \N_0} \varphi^n(\XS^- \times \US)$ is the global attractor of $\varphi$.
We will review attractors and solutions of abstract dynamical systems in more detail later in \cref{sec_abstract}.
For the state equation, we make the following important observation, which tells us that to understand the solutions of the state equation we may study the attractor of $\varphi$.
This paves the way to leveraging tools from \textit{autonomous} dynamical systems theory.

\begin{proposition}
\label{prop_det_sol_attractor}
	The set of solutions of the state equation satisfies $\DetSol = \tau(\Sc_{\varphi})$ and $\Sc_{\varphi} = \tau^{-1}(\DetSol)$.
\end{proposition}

The dynamical system $\varphi$ processes bi-infinite input sequences in its second argument.
The reason to work with full bi-infinite input sequences is that we do not prescribe a rule by which the inputs are generated.
Allowing $\varphi$ to access the next input $\seq{u}{1}$ in the input sequence $\Seq{u}$ covers the general case of arbitrary inputs.
Let us spell out an important special case.
More precisely, suppose (only for the remainder of this subsection) that the inputs are generated by an invertible dynamical system $\phi \colon \Mc \rightarrow \Mc$, masked by an observation function $\omega \colon \Mc \rightarrow \Uc$, which is a customary setup in the literature \cite{BerryDas2023, RC18, RC21}.
Then, one considers the autonomous dynamical system $\psi \colon \Xc \times \Mc \rightarrow \Xc \times \Mc$ given by $\psi(x,p) = (f(x,\omega(\phi(p))),\phi(p))$, which evolves on the base spaces instead of the sequence spaces.
To recover this special case, we take $\XS^- = \Xc^{\Z_-}$ to be the entire sequence space and $\US = \overline{\omega}(\Mc)$ to be the image of the map $\overline{\omega} \colon \Mc \rightarrow \Uc^{\Z}$, $p \mapsto (\omega(\phi^t(p)))_{t \in \Z}$.
With these choices, we ensure that $\Sc_{\psi} = \{ (\seq{x}{0},p) \colon (\Seq{x},\overline{\omega}(p)) \in \Sc_{\varphi} \}$, where $\Sc_{\psi}$ is the set of solutions of $\psi$, which we recall are points through which a bi-infinite orbit of $\psi$ passes.
Conversely, $\Sc_{\varphi} = \{ (\Seq{x},\Seq{u}) \colon (\seq{x}{t},p_t) \in \Sc_{\psi} \text{ for some } p_t \in \overline{\omega}^{-1}(\sigma^t(\Seq{u})) \text{ for all } t \in \Z_- \}$.
Furthermore, if $\zeta \colon \US \rightarrow \XS^-$ is a generalized synchronization (GS, see \cref{subsec_skewproducts,subsec_bundlesystems} and \cite{AmigoEtal2024, LuHuntOtt2018Chaos, KocarevParlitz1996}) for $\varphi$, meaning that we have the semi-conjugacy $\varphi \circ (\zeta \times \mathrm{id}_{\US}) = (\zeta \times \mathrm{id}_{\US}) \circ \sigma$, then $\zeta_0 \colon \Mc \rightarrow \Xc$, $p \mapsto (\zeta \circ \overline{\omega}(p))_0$ is a GS for $\psi$, meaning that we have the semi-conjugacy $\psi \circ (\zeta_0 \times \mathrm{id}_{\Mc}) = (\zeta_0 \times \mathrm{id}_{\Mc}) \circ \phi$.
Conversely, if $\zeta_0 \colon \Mc \rightarrow \Xc$ is a GS for $\psi$, then the set $\zeta_0(\overline{\omega}^{-1}(\Seq{u}))$ is a singleton for all $\Seq{u} \in \US$ even if $\overline{\omega}$ is not injective, and a GS for $\varphi$ can be recovered by $\zeta \colon \US \rightarrow \XS^-$, $\Seq{u} \mapsto (\zeta_0(\overline{\omega}^{-1}(\sigma^t(\Seq{u}))))_{t \in \Z_-}$.
We point out that $\zeta$ synchronizes the state-space dynamics with the `dynamics' of the observations, whereas $\zeta_0$ synchronizes with the true dynamics of the underlying system $\phi$.
In particular, notice that $\zeta$ is injective if $\zeta_0$ is injective.
In this case, if $\eta_0 \colon \zeta_0(\Mc) \rightarrow \Mc$ is the left-inverse of $\zeta_0$, then $\eta \colon \zeta(\US) \rightarrow \US$ given by $\eta(\Seq{x}) = (\omega \circ \eta_0(\seq{x}{t}))_{t \in \Z}$ is the left-inverse of $\zeta$.
However, injectivity of $\zeta$ does not necessarily transfer to $\zeta_0$.
This injectivity is commonly regarded as `learnability' since the existence of the inverse means that we can learn a readout that returns the next observation from the current state of the state-space system.

If the dynamical system $\phi$ generating the inputs satisfies the hypothesis of Takens' embedding theorem \cite{Takens1981} and $\omega$ is a generic observation function, then the map $\Seq{\omega} \colon \Mc^{\Z} \rightarrow \US$, $\Seq{p} \mapsto (\omega(\seq{p}{t}))_{t \in \Z}$ is injective on the set $\MS := \{ (\phi^t(p))_{t \in \Z} \colon p \in \Mc \}$.
In this case, we can regard $\varphi$ as a bundle map over the left-shift on $\MS$ on the bundle $\pi' := \Seq{\omega}^{-1} \circ \pi_{\US} \colon \XS^- \times \US \rightarrow \MS$ (see \cref{subsec_bundlesystems}), and a GS for $\varphi$ synchronizes the state-space dynamics also with the true dynamics of the underlying system instead of the observations.
The various maps and semi-conjugacies are summarized in the commutative diagrams in \cref{fig_GS_sss}.

\begin{figure}
\begin{tikzcd}[row sep = 4em, column sep = 4em, every label/.append style = {font = \normalsize}]
	\zeta(\US) \arrow[bend right = 50]{dd}{\eta} \arrow[dashed]{r} \pgfmatrixnextcell \arrow[swap, bend left = 50]{dd}{\eta} \zeta(\US)
	\\
	\arrow[swap]{u}{\pi_{\XS^-}} \Sc_{\varphi} \arrow[swap, shift right=1]{d}{\pi_{\US}} \arrow{r}{\varphi} \pgfmatrixnextcell \Sc_{\varphi} \arrow[shift left=1]{d}{\pi_{\US}} \arrow{u}{\pi_{\XS^-}}
	\\
	\US \arrow[swap, shift right=1]{u}{\zeta \!\times\! \mathrm{id}_{\US}} \arrow{r}{\sigma} \pgfmatrixnextcell \US \arrow[shift left=1]{u}{\zeta \!\times\! \mathrm{id}_{\US}}
	\\
	\Mc \arrow{u}{\overline{\omega}} \arrow{r}{\phi} \pgfmatrixnextcell \Mc \arrow{u}{\overline{\omega}}
\end{tikzcd}
\hfill%
\begin{tikzcd}[row sep = 4em, column sep = 4em, every label/.append style = {font = \normalsize}]
	\phantom{\US} \pgfmatrixnextcell \phantom{\US}
	\\
	\Mc \arrow[dashed]{u} \arrow[shift left=1]{d}{\zeta_0 \!\times\! \mathrm{id}_{\Mc}} \arrow{r}{\phi} \pgfmatrixnextcell \Mc \arrow[dashed]{u} \arrow[swap, shift right=1, near start]{d}{\zeta_0 \!\times\! \mathrm{id}_{\Mc}}
	\\
	\arrow{d}{\pi_{\Xc}} \Sc_{\psi} \arrow[shift left=1]{u}{\pi_{\Mc}} \arrow{r}{\psi} \pgfmatrixnextcell \Sc_{\psi} \arrow[swap, shift right=1]{u}{\pi_{\Mc}} \arrow[swap]{d}{\pi_{\Xc}}
	\\
	\zeta_0(\Mc) \arrow[swap, bend left = 50]{uu}{\eta_0} \arrow[dashed]{r} \pgfmatrixnextcell \arrow[bend right = 50]{uu}{\eta_0} \zeta_0(\Mc)
\end{tikzcd}
\hfill%
\begin{tikzcd}[row sep = 4em, column sep = 4em, every label/.append style = {font = \normalsize}]
	\zeta'(\US) \arrow[bend right = 50]{dd}{\eta'} \arrow[dashed]{r} \pgfmatrixnextcell \arrow[swap, bend left = 50]{dd}{\eta'} \zeta'(\US)
	\\
	\arrow[swap]{u}{\pi_{\XS^-}} \Sc_{\varphi} \arrow[swap, shift right=1]{d}{\pi'} \arrow{r}{\varphi} \pgfmatrixnextcell \Sc_{\varphi} \arrow[shift left=1]{d}{\pi'} \arrow{u}{\pi_{\XS^-}}
	\\
	\MS \arrow[swap, shift right=1]{u}{\zeta' \!\times\! \Seq{\omega}} \arrow{r}{\sigma} \pgfmatrixnextcell \MS \arrow[shift left=1]{u}{\zeta' \!\times\! \Seq{\omega}}
	\\
	\Mc \arrow{u}{\mathrm{Orb}_{\phi}} \arrow{r}{\phi} \pgfmatrixnextcell \Mc \arrow{u}{\mathrm{Orb}_{\phi}}
\end{tikzcd}%
\caption{
Commutative diagrams depicting the relations in \cref{subsec_dyn_sequence}.
The dashed horizontal arrows depict maps induced by $\eta$, $\eta_0$, and $\eta'$ by a diagram chase.
The dashed vertical arrows indicate that the middle diagram could be added at the bottom of either the left or right diagram to create larger commutative diagrams.}
\label{fig_GS_sss}
\end{figure}
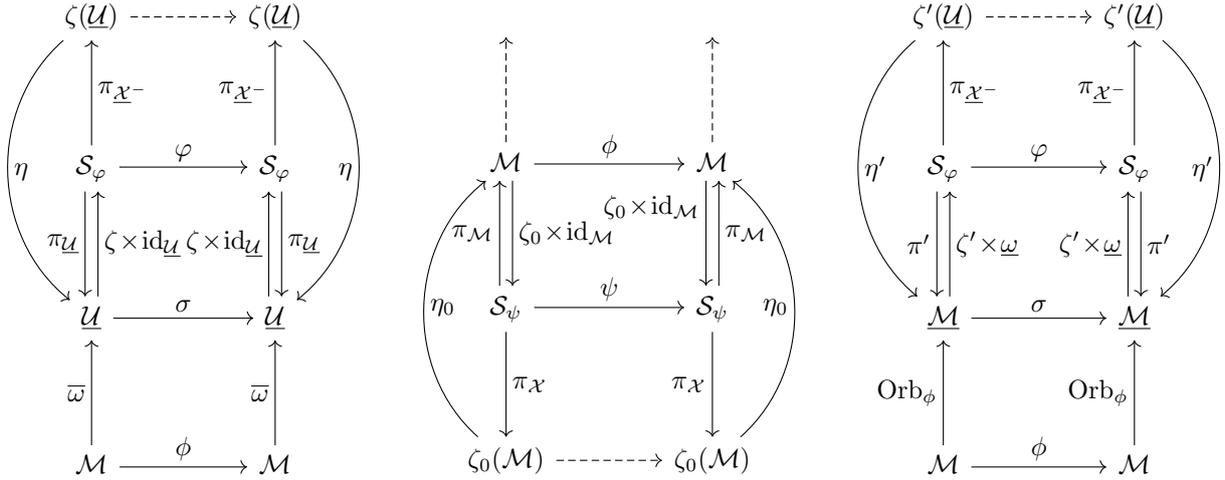

\subsection{Echo states}

In addition to the state map $f$, consider a readout $h \colon \Xc \rightarrow \Yc$ for which $\Seq{x} \mapsto (h(\seq{x}{t}))_{t \in \Z_-}$ poses a well-defined and continuous map $\XS^- \rightarrow \YS^-$.
An {\bfi output} of the state-space system is a pair of sequences $(\Seq{y},\Seq{u}) \in \YS^- \times \US^-$ for which there exists a solution $\Seq{x} \in \XS^-$ for the input $\Seq{u}$ that satisfies $h(\seq{x}{t}) = \seq{y}{t}$ for all $t \in \Z_-$.
Naturally, we also call $\Seq{y}$ an output for the input $\Seq{u}$.

Solutions of state-space systems can be seen as special cases of outputs of state-space systems simply by taking the identity as readout.
However, it is instructive to treat solutions and outputs separately to emphasize the connection of solutions with the attractor of the dynamical system exhibited in the previous section.

\begin{definition}
	We say that the state-space system has the {\bfi echo state property (ESP)} if there exists a unique solution $\Seq{x} \in \XS^-$ for any given input $\Seq{u} \in \US^-$;
and we say that it has the {\bfi ESP in the outputs} if there exists a unique output $\Seq{y} \in \YS^-$ for any given input $\Seq{u} \in \US^-$.
\end{definition}

Clearly, if the state-space system has the ESP, then it also has the ESP in the outputs.
We had introduced the set $\DetSol$ of solutions in the previous section.
Analogously, denote the set of all outputs by
\begin{equation*}
	\DetOut
	:= \{ H(\Seq{x},\Seq{u}) \in \YS^- \times \US^- \colon (\Seq{x},\Seq{u}) \in \DetSol \},
\end{equation*}
where $H \colon \XS^- \times \US^- \rightarrow \YS^- \times \US^-$ denotes the extended readout $(\Seq{x},\Seq{u}) \mapsto ((h(\seq{x}{t}))_{t \in \Z_-},\Seq{u})$.
In the presence of the ESP, we obtain well-defined maps $\US^- \rightarrow \XS^-$ and $\US^- \rightarrow \YS^-$ that associate to any given input its unique solution and output, respectively.
Without the ESP, we instead consider the set-valued map $\Seq{u} \mapsto \{ (\Seq{x},\Seq{u}) \colon (\Seq{x},\Seq{u}) \in \DetSol \}$.
By a slight abuse of notation, we denote this set-valued map by $\Seq{u} \mapsto \DetSol(\Seq{u})$ since it is exactly the fiber map $\Seq{u} \mapsto \DetSol \cap \pi_{\US^-}^{-1}(\Seq{u})$ of the set $\DetSol$.
We similarly regard $\DetOut$ as a set-valued fiber map.

As motivated in the introduction, the ESP is closely linked to the notion of fading memory \cite{BoydChua1985,RC30}, which we introduce next.
The concept of fading memory is, in fact, formalized by \cref{assmpt_det_SSS}.(iv).
Indeed, fading memory is typically realized as a property of the topology on the domain of a filter, with respect to which continuity is required to hold.
Here, we possibly work with set-valued maps, with which we can state the most general definition, following \cite{Manjunath2020ProcA}.
The notion of hemi-continuity generalizes the concept of continuity to set-valued maps.
A set-valued map $S \colon \Vc \rightarrow 2^{\Wc}$, where $\Vc$ and $\Wc$ are topological spaces, is hemi-continuous at a point $v \in \Vc$ if for any two open sets $W,W' \subseteq \Wc$ with $W \cap S(v) \neq \emptyset$ and $S(v) \subseteq W'$ there exists an open neighborhood $V \subseteq \Vc$ of $v$ such that $W \cap S(v') \neq \emptyset$ and $S(v') \subseteq W'$ for all $v' \in V$.

\begin{definition}
	We say that the state-space system has the {\bfi fading memory property (FMP)} at a given input $\Seq{u} \in \US^-$ if $\DetSol$ is hemi-continuous at $\Seq{u}$.
We denote the set of all points at which the state-space system has the FMP by $D(\DetSol)$.
We say that the state-space system has the FMP if $D(\DetSol) = \US^-$.
The {\bfi FMP in the outputs} and the set $D(\DetOut)$ are defined analogously.
\end{definition}

That the ESP implies the FMP has been claimed and proved before \cite{Jaeger2010, Manjunath2020ProcA}.
However, \cite{Jaeger2010} crucially assumed compactness not just of the state space but also of the input space.
The most general result thus far has been \cite[Corollary 3.3]{Manjunath2020ProcA}.
Our next result is a slight generalization thereof.

\begin{theorem}
\label{thrm_det_ESP_FMP}
	Suppose $\XS^-$ is compact.
If the state-space system has the ESP (in the outputs), then it has the FMP (in the outputs).
\end{theorem}

\begin{example}
\label{ex_linear_SSS}
	In general, the assumption that $\XS^-$ be compact cannot be dropped.
Let $\US^- = \ell^{\infty}(\R^d)$ and $\XS^- = \ell^{\infty}(\R^n)$ with the product topologies.
Consider a matrix $A \in \R^{n \times n}$ with a spectral radius strictly smaller than 1 and a matrix $B \in \R^{n \times d}$.
The linear state-space system driven by $f(x,u) = Ax + Bu$ has the ESP;
indeed, $\pi_{\XS^-}(\DetSol(\Seq{u})) = \{ ( \sum_{s \leq 0} A^{-s}B\seq{u}{t+s})_{t \leq 0} \}$.
However, since $\US^-$ and $\XS^-$ are endowed with the product topologies, the map $\pi_{\XS^-} \circ \DetSol \colon \US^- \rightarrow \XS^-$ is known to be continuous if and only if $A$ is nilpotent \cite[Proposition 11]{RC30}.
\end{example}

Noteworthy in \cref{thrm_det_ESP_FMP} is that neither compactness of the state sequence space nor the ESP depend on the choice of topology on the input sequence space;
yet, together they imply the FMP, which explicitly depends on that choice.
In particular, \cref{thrm_det_ESP_FMP} holds for the coarsest topology that we admit for defining the FMP, namely the product topology.
We take note of this consequence in the corollary below.

\begin{corollary}
	Let $A \colon \US^- \rightarrow \YS^-$ be a function, and suppose there exist a continuous state map and a continuous readout with the ESP in the outputs such that $A(\Seq{u})$ is the output for the input $\Seq{u}$, for all $\Seq{u} \in \US^-$.
If the state sequence space on which the state-space system evolves can be taken compact, then $A$ is continuous with respect to the product topology on $\US^-$.
\end{corollary}

It is known that linear time-invariant filters $\ell^{\infty}(\R) \rightarrow \ell^{\infty}(\R)$ that are continuous with respect to a weighted sup-norm can be realized by a linear state-space system with the ESP in the outputs \cite{RC16}.
However, the state space on which this state-space system lives is a potentially complicated object.
Conditions that guarantee compactness in this context are unknown, as are generalizations to the non-linear case.

\subsection{Generic fading memory}

Embodying the notion of continuity of set-valued maps, we can show that the FMP implies stability of the number of solutions as a function of the input.

\begin{proposition}
\label{prop_det_FMP_mESP_sol}
	The map $D(\DetSol) \rightarrow \N_0 \cup \{\infty\}$, $\Seq{u} \mapsto \#\DetSol(\Seq{u})$ is constant.
Furthermore, if we denote this constant by $M$, then $\#\DetSol(\Seq{u}) \geq M$ for any $\Seq{u} \in \US^-$.
In particular, if the state-space system has the FMP, then $\#\DetSol(\Seq{u})$ is constant on $\US^-$, and if it has the FMP and $\#\DetSol(\Seq{u}) = 1$  for some $\Seq{u} \in \US^-$, then it has the ESP.
\end{proposition}

We had pointed out that \cref{thrm_det_ESP_FMP} holds for any choice of topology on $\US^-$ admissible under \cref{assmpt_det_SSS} so that, in particular, we could have chosen the product topology to obtain the strongest notion of fading memory in \cref{thrm_det_ESP_FMP}.
The same applies to \cref{prop_det_FMP_mESP_sol}, except that this time the FMP acts as a hypothesis.
In this case, we could have chosen an arbitrarily fine topology on $\US^-$ as long as it satisfies \cref{assmpt_det_SSS}.

\cref{prop_det_FMP_mESP_sol} generalizes a result of \cite{Manjunath2020ProcA, Manjunath2022Nonlin}.
In those works, it is assumed that $\#\DetSol(\Seq{u}) = 1$ for some $\Seq{u} \in \US^-$, in which case $\#\DetSol(\Seq{u})$ being constant becomes the ESP.
More importantly, \cref{prop_det_FMP_mESP_sol} poses no additional assumptions on the state or input space such as compactness.
To extend this result to the outputs, we adapt a few notions from control theory to suit our context \cite{RC16, Sontag1990}.

\begin{definition}
	Two states $x,x' \in \Xc$ are called {\bfi simultaneously reachable} if there exist two solutions $(\Seq{x},\Seq{u}),(\Seq{x}',\Seq{u}) \in \DetSol$ with $\seq{x}{0} = x$ and $\seq{x}{0}' = x'$.
Two states $x,x' \in \Xc$ are called {\bfi distinguishable} if there exists some $u \in \Uc$ such that $h(f(x,u)) \neq h(f(x',u))$.
We say that the state-space system {\bfi distinguishes reachable states}\footnote{
Our notion of distinguishing reachable states is closely related but not identical to the control-theoretic notion of observability.
The difference is, first, that our notion considers only \textit{simultaneously} reachable states, that is, states that can be reached with the same input/control sequence;
and, second, our notion asks two such states to be instantaneously distinguishable whenever they are distinguishable in time 1, in the language of \cite{Sontag1990}.} if $h(x) \neq h(x')$ for any two distinguishable simultaneously reachable states $x,x' \in \Xc$.
\end{definition}

It is clear that a state-space system with the ESP distinguishes reachable states since then there are no distinct simultaneously reachable states.

\begin{remark}
\label{rem_weakly_disting}
	The state-space system distinguishes reachable states if and only if $h(\seq{x}{0}) = h(\seq{x}{0}')$ for any $(\Seq{x},\Seq{u}),(\Seq{x}',\Seq{u}) \in \DetSol$ with $h(\seq{x}{-1}) = h(\seq{x}{-1}')$.
\end{remark}

\begin{example}
	Consider the linear state-space system $f(x,u) = Ax + Bu$ and $h(x) = Wx$ with matrices $A \in \R^{n \times n}$, $B \in \R^{n \times d}$, and $W \in \R^{m \times n}$.
Then, the system distinguishes reachable states if $A$ maps the kernel of $W$ into the kernel of $W$.
\end{example}

\begin{proposition}
\label{prop_det_FMP_mESP_out}
	If the state-space system distinguishes reachable states, then the analog statement of \cref{prop_det_FMP_mESP_sol} holds with $\DetOut$ in place of $\DetSol$.
\end{proposition}

\begin{example}
	In general, the assumption that the state-space system distinguishes reachable states cannot be dropped.
Consider the linear state-space system from \cref{ex_linear_SSS}.
Extend the system to the one-point compactification $\Xc := \R^n \cup \{ \infty \} \cong \Sb^n$ by setting $f' = f$ on $\R^n \times \R^d$ and $f' = \infty$ on $\{ \infty \} \times \R^d$.
This extension $f'$ is continuous if the matrix $A$ is injective.
%As state sequence space, take $\XS^- = \Xc^{\Z_-} \supset \ell^{\infty}(\R^n)$.
As state sequence space, take $\XS^- = \ell^{\infty}(\R^n) \cup \{ \Seq{x} \in \Xc^{\Z_-} \colon \inf_{t \in \Z_-} \norm{\seq{x}{t}} > 0 \}$.
The extended system has exactly two solutions for any input $\Seq{u} \in \US^-$, namely the original solution and the sequence $\Seq{\infty}$ that is constantly $\infty$.
Suppose $B$ is injective so that $(\Seq{0},\Seq{u}) \notin \DetSol(\Seq{u})$ for any $\Seq{u} \neq \Seq{0}$.
Take a continuous readout $h \colon \Xc \rightarrow \R$ that maps the poles and only the poles to zero, that is, $h^{-1}(0) = \{0,\infty\}$.
Then, $\DetOut_{f'}(\Seq{0}) = \{ (\Seq{0},\Seq{u}) \}$ and $\#\DetOut_{f'}(\Seq{u}) = 2$ for any $\Seq{u} \neq \Seq{0}$.
The new system does not distinguish reachable states because $0$ and $\infty$ are two distinguishable simultaneously reachable states.
The state-space system has the FMP in the outputs if $\US^-$ is endowed with the topology induced by a weighted supremums norm \cite[Theorem 10]{RC30}.
\end{example}

The next result says that if the underlying spaces exhibit more regularity, then the FMP holds generically and bifurcations of solutions (and outputs) can occur only within a meager set.
We recall that a subset of a topological space is residual if it can be written as a countable intersection of subsets each of which has a dense interior;
a subset is meager if its complement is residual;
and a property holds `generically' if it holds on a residual subset.

\begin{theorem}
\label{thrm_det_generic_FMP}
	Suppose that $\US^-$, $\XS^-$, and $\YS^-$ are Polish and that $\XS^-$ is compact.
Then, $D(\DetSol)$ and $D(\DetOut)$ contain a residual set.
In particular, $\#\DetSol(\Seq{u})$ is generically constant, and so is $\#\DetOut(\Seq{u})$ if the state-space system distinguishes reachable states.
\end{theorem}

We have seen in \cref{prop_det_FMP_mESP_sol} that the number of solutions is constant in the input if the state-space system has the FMP.
We conclude this section with the conjecture that the converse also holds.

\begin{conjecture}
	If $\#\DetSol(\Seq{u})$ is constant on $\US^-$ and this constant is finite, then the state-space system has the FMP.
\end{conjecture}

\subsection{On a prominent example}

The dynamical system induced by the state map $f(x,u) = ux/(1+\abs{x})$ had been introduced in \cite{KloedenEtal2012} to illustrate subtleties of pullback and forward attraction in non-autonomous systems.
This example had been revisited in the context of echo states both in \cite{Manjunath2022Nonlin} and, with an additional hyperbolic tangent transformation, in \cite{CeniEtal2020PhysD}.
Suppose we feed the state map with the input sequence $\Seq{u}$ that is constantly $u^- \in (0,1)$ in negative time and then switches to being constantly $u^+ \in (1,2)$ in positive time.
The only solution for the input $\Seq{u}$ is the state sequence $\Seq{x} = \Seq{0}$ that is constantly zero.
Since we have seen that solutions form the global attractor of the dynamical system, it is therefore tempting to think that trajectories of the state-space system are attracted to the constant zero sequence.
However, this is not the case as noted in \cite{CeniEtal2020PhysD,KloedenEtal2012};
even more so, the constant zero sequence is a repeller.
It had been argued in \cite{CeniEtal2020PhysD} that the issue is due to pullback versus forward attraction.
On the contrary, the results in this work reveal that we cannot expect trajectories to converge to the solution for the input $\Seq{u}$.
Indeed, guided by the idea of fading memory (which we proved to hold generically), after a while the dynamical system can hardly distinguish between the input sequence $\Seq{u}$ and the input sequence $\Seq{u}^+$ that is constantly $u^+$.
For the input sequence $\Seq{u}^+$, there are three solutions $\Seq{0},\Seq{x}^+,\Seq{x}^-$, namely the state sequences that are constantly one of the fixed points $0,x^+,x^-$ of the map $x \mapsto u^+x/(1+\abs{x})$.
Thus, trajectories of the state-space system with input $\Seq{u}$ are expected not to be attracted to the single sequence $\Seq{0}$ but instead to the set $\{\Seq{0},\Seq{x}^+,\Seq{x}^-\}$, which is indeed the behavior that had been observed in the numerical illustrations in \cite{CeniEtal2020PhysD,KloedenEtal2012}.
In particular, the (fiber of the) global pullback attractor $\{\Seq{0},\Seq{x}^+,\Seq{x}^-\}$ acts as a forward attractor as well in this case.
The relationship between echo states and forward attraction has been investigated in more detail in \cite{RC32}.
We have proved in this work that the number of solutions is constant on a residual set.
In particular, in the example above, the input sequence is taken from a meager set and represents an unlikely scenario for actual deployment of the model.
At the same time, this means that fading memory does not hold at the input $\Seq{u}$.
However, what breaks down at $\Seq{u}$ is lower hemi-continuity whereas upper hemi-continuity still provably holds --- and it is upper hemi-continuity that warrants the intuition for converging to the solution set of $\Seq{u}^+$ as opposed to the solution of $\Seq{u}$.

\section{Stochastic state-space systems}
\label{sec_stoch_SSS}

We adopt the notation from the previous section, continue to pose \cref{assmpt_det_SSS}, and assume additionally that $\US^-$, $\XS^-$, $\YS^-$, and $\US$ are completely regular.
Let $\Pstate^- \subseteq P(\XS^- \times \US^-)$ and $\Pin^- \subseteq P(\US^-)$ be backwards shift-invariant, that is, $(T_*)^{-1}(\Pstate^-) = \Pstate^-$ and $(T_*)^{-1}(\Pin^-) = \Pin^-$, respectively, and satisfy $(\pi_{\US^-})_*(\Pstate^-) = \Pin^-$.
Also, let $\Pout^- := (\pi_{\US^-})_*^{-1}(\Pin^-) \subseteq P(\YS^- \times \US^-)$.
In the deterministic case, we studied the global attractor of the dynamical system $\varphi$.
In the stochastic case, we replace $\varphi$ by its push-forward, whose domain we pick as follows.
Let $\Pstate := (\tau_*)^{-1}(\Pstate^-) \subseteq P(\XS^- \times \US)$ and $\Pin := (\tau_*)^{-1}(\Pin^-) \subseteq P(\US)$.
We point out that $\Pstate$ and $\Pin$ are backwards shift-invariant and satisfy $(\pi_{\US})_*(\Pstate) \subseteq \Pin$.
The reverse inclusion $\Pin \subseteq (\pi_{\US})_*(\Pstate)$ holds if $\US$ is Polish (this follows from \cref{lem_causal_at_zero} further below).

\begin{assumption}
	The sets $\Pin^-$, $\Pin$, $\Pstate^-$, $\Pstate$, and $\Pout^-$ are endowed with topologies that are at least as fine as the weak topologies and such that the following push-forward maps are all continuous:
$(\pi_{\US^-})_* \colon \Pstate^- \rightarrow \Pin^-$,
$(\pi_{\US})_* \colon \Pstate \rightarrow \Pin$,
%$T_* \colon \Pstate^- \rightarrow \Pstate^-$,
%$T_* \colon \Pin^- \rightarrow \Pin^-$,
$T_* \colon \Pstate \rightarrow \Pstate$,
%$T_* \colon \Pin \rightarrow \Pin$,
$\sigma_* \colon \Pin \rightarrow \Pin$,
$\varphi_* \colon \Pstate \rightarrow \Pstate$,
$H_* \colon \Pstate^- \rightarrow \Pout^-$.
\end{assumption}

\begin{example}
\label{ex_weak_top}
	The weak topologies on $\Pin^-$, $\Pin$, $\Pstate^-$, $\Pstate$, and $\Pout^-$ are admissible.
Indeed, push-forward maps induced by continuous maps are continuous on the spaces of Radon probability measures with respect to the weak topologies if the underlying spaces are completely regular \cite[Proposition 7.2.2 and Theorem 8.4.1]{Bogachev2007}.
Furthermore, we point out that complete regularity of the underlying spaces makes the weak topologies Hausdorff \cite[Theorem 7.10.6 and Corollary 8.2.5]{Bogachev2007}.
\end{example}

\subsection{Stochastic echo states}
\label{sec_stoch_echo_states}

Given an element $\mu \in P(\XS^- \times \US^-)$, a realization of $\mu$ is a pair $(\Seq{X},\Seq{U})$ of an $\XS^-$-valued and a $\US^-$-valued random variable\footnote{
An $\XS^-$-valued random variable defines a stochastic process on $\Xc$ but not every stochastic process on $\Xc$ defines a random variable on $\XS^-$ if the topology on $\XS^-$ is finer than the product topology.} that are defined on some common probability space and whose joint law is $\mu$.
We call $\mu \in \Pstate^-$ a {\bfi stochastic solution} of the state-space system if some realization $(\Seq{X},\Seq{U})$ of $\mu$ is a solution of the state equation almost surely, that is, $\seq{X}{t} = f(\seq{X}{t-1},\seq{U}{t})$ holds almost surely for all $t \in \Z_-$.
As such, a stochastic solution not only captures the law of the states but also encodes the input law in its $\US^-$-marginal and the dependence structure between the states and the inputs.
Given an input law $\Xi \in \Pin^-$, an element $\mu \in \StochSol$ is called a {\bfi stochastic solution for the input law} $\Xi$ if $(\pi_{\US^-})_*\mu = \Xi$.
The set of all stochastic solutions will be denoted $\StochSol \subseteq \Pstate^-$.

In the previous section, we observed that deterministic solutions are (truncations of) elements of the global attractor of the dynamical system $\varphi$.
In the stochastic case, the push-forward $\varphi_*$ acts as an autonomous dynamical system on $\Pstate$.
The global attractor of $\varphi_*$ is
\begin{equation*}
	\Sc_{\varphi_*}
	= \bigcap_{n \in \N} \varphi_*^n(\Pstate)
	= \{ \mu \in \Pstate \colon \varphi_*^nT_*^n\mu = \mu \text{ for all } n \in \N \}.
\end{equation*}
Although stochastic solutions are defined through almost sure equality of the state equation and $\varphi_*$ only sees laws, we recover the analog of the deterministic result from \cref{prop_det_sol_attractor}.

\begin{proposition}
\label{prop_stoch_sol}
	The set of stochastic solutions satisfies $\StochSol = \tau_*(\Sc_{\varphi_*})$.
Furthermore, for any $\mu \in \Pstate^-$, the following are equivalent.
\begin{enumerate}[\upshape (i)]\itemsep=0em
\item
Some realization of $\mu$ is a solution of the state equation almost surely.
\item
Every realization of $\mu$ is a solution of the state equation almost surely.
\item
The support of $\mu$ is contained in $\DetSol$.
\end{enumerate}
\end{proposition}

In the introduction, we discussed that the stochastic context requires making a choice when defining solutions of the state-space system:
almost sure equality of the state equation versus equality in law.
More formally, consider the continuous map $\Fc \colon \XS^- \times \US^- \rightarrow \Xc^{\Z_-} \times \Uc^{\Z_-}$, $(\Seq{x},\Seq{u}) \mapsto (f(\seq{x}{t-1},\seq{u}{t}),\seq{u}{t})_{t \in \Z_-}$ and the inclusion $\iota \colon \XS^- \times \US^- \rightarrow \Xc^{\Z_-} \times \Uc^{\Z_-}$.
In the deterministic context, it is clear that solutions to the state equation are precisely the fixed points of $\Fc$.
We call $\mu \in \Pstate^-$ a {\bfi stochastic fixed-point solution} of the state equation if $\Fc_*\mu = \iota_*\mu$.
Since almost sure equality of the state equation implies equality in law, it is clear that any stochastic solution of the state-space system is also a stochastic fixed-point solution.
The converse is --- unsurprisingly --- not true.
The link between stochastic solutions and the attractor of $\varphi_*$ exhibited in \cref{prop_stoch_sol} poses a strong argument in favor of the definition of a stochastic solution through almost sure equality of the state equation.

\begin{example}
	Let us demonstrate that stochastic fixed-point solutions that are not stochastic solutions exist even in trivial cases.
Suppose the state map $f(x,u) = x$ does not depend on the input.
Let $\nu \in P(\XS^-)$ be shift-invariant, that is, $T_*\nu = \nu$, and let $\mu \in P(\XS^- \times \US^-)$ be the product measure of $\nu$ and any given input law $\Xi \in P(\US^-)$.
Then, $\mu$ is a stochastic fixed-point solution but we know from \cref{prop_stoch_sol} that $\mu \in \StochSol$ if only if $\nu$ is supported on the diagonal in $\XS^-$.
\end{example}

For the remainder of the paper, we will use being an element of $\tau_*(\Sc_{\varphi_*})$ as a working definition for a stochastic solution.
{\bfi Stochastic outputs} are elements of $\StochOut := H_*(\StochSol)$.
%It follows from \cref{prop_stoch_sol} that the support of an element in $\StochOut$ is contained in the closure of $\DetOut$.

\begin{definition}
	We say that the state-space system has the {\bfi stochastic ESP} if there exists a unique stochastic solution for any given input law;
and we say that it has the {\bfi stochastic ESP in the outputs} if there exists a unique stochastic output for any given input law.
\end{definition}

As in the deterministic framework, we regard the sets $\StochSol$ and $\StochOut$ as their set-valued fiber maps $\Xi \mapsto \StochSol \cap (\pi_{\US^-})_*^{-1}(\Xi)$ and $\Xi \mapsto \StochOut \cap (\pi_{\US^-})_*^{-1}(\Xi)$, respectively.

\begin{definition}
	We say that the state-space system has the {\bfi stochastic FMP} at a given input law $\Xi \in \Pin^-$ if $\StochSol$ is hemi-continuous at $\Xi$.
We denote the set of all points at which the state-space system has the stochastic FMP by $D(\StochSol)$.
We say that the state-space system has the stochastic FMP if $D(\StochSol) = \Pin^-$.
The {\bfi stochastic FMP in the outputs} and the set $D(\StochOut)$ are defined analogously.
\end{definition}

In the presence of the ESP and the FMP, $\DetSol$ becomes a continuous map $\US^- \rightarrow \XS^- \times \US^-$ and its push-forward is well-defined.
It was shown in \cite[Remark 3.11]{RC28} that if the ESP, the FMP, and the stochastic ESP hold, then the unique stochastic solution is given by the push-forward of the input law under $\DetSol$.
The proof therein presupposed the stochastic ESP, but we would expect that the stochastic ESP is, in fact, a consequence of the ESP and the FMP.
Here, we supplement this generalization.

\begin{proposition}
\label{prop_ESP_stoch_ESP}
	Suppose the state-space system has the ESP and the FMP.
Then, $\DetSol_*(\Pin^-) \cap \Pstate^- = \StochSol$.
In particular, if $\DetSol_*(\Pin^-) \subseteq \Pstate^-$, then the state-space system also has the stochastic ESP.
\end{proposition}

In the next result, we generalize \cref{thrm_det_ESP_FMP} and the first part of \cref{thrm_det_generic_FMP} to the stochastic case.

\begin{theorem}
\label{thrm_stoch_ESP_FMP}
	Suppose $(\pi_{\US^-})_* \colon \Pstate^- \rightarrow \Pin^-$ is proper.\footnote{
A map is proper if it is closed and has compact fibers.
Equivalent characterizations are derived in \cref{lem_proper}.}
If the state-space system has the stochastic ESP (in the outputs), then it has the stochastic FMP (in the outputs).
Furthermore, if $\Pin^-$ and $\Pstate^-$ are Polish, then the stochastic FMP holds generically even without the stochastic ESP.
\end{theorem}

\begin{example}
	Suppose $\US^-$ and $\XS^-$ are Polish and metrized by $d_{\US^-}$ and $d_{\XS^-}$, respectively.
The product topology on $\XS^- \times \US^-$ is metrized by $d_{\XS^- \times \US^-} = d_{\XS^-} \circ \pi_{\XS^-} + d_{\US^-} \circ \pi_{\US^-}$.
Suppose $\Pstate^-$ and $\Pin^-$ are closed subsets of the Wasserstein-$p$ spaces defined on $(\XS^- \times \US^-,d_{\XS^- \times \US^-})$ and $(\US^-,d_{\US^-})$, respectively, for some $p \in [1,\infty)$.
If $\XS^-$ is compact, then $(\pi_{\US^-})_* \colon \Pstate^- \rightarrow \Pin^-$ is proper.
We emphasize that compactness of $\XS^-$ is sufficient but not necessary.
\end{example}

\subsection{Causal solutions}
\label{subsec_causal_sol}

\begin{assumption}
\label{assmpt_prod_top}
	For the remainder of this section, suppose that the topologies on $\US$ and $\US^-$ are the product topologies and that $\US$, $\US^-$, and $\XS^-$ are Polish.\footnote{
The product topology on $\US$ ($\US^-$) is Polish if and only if $\US$ ($\US^-$) can be written as a countable intersection of open subsets of $\Uc^{\Z}$ ($\Uc^{\Z_-}$) \cite{Engelking1989}.}
\end{assumption}

The stochastic analog of \cref{prop_det_FMP_mESP_sol} (generic fading memory) is not as straightforward as the deterministic result.
It requires us to consider only solutions with a special probabilistic dependence structure.
In a deterministic setting, it is clear that past states are not influenced by future inputs.
This is no longer the case in the stochastic setting.
Past inputs, which determine past states, are in general correlated to future inputs, thus creating correlation between past states and future inputs.
However, it is natural to assume that there is no `direct' correlation between past states and future inputs but only such correlation that arises through past inputs.
We will call this causality.
The notion of conditional independence of sigma-algebras captures this idea mathematically.
The reader unfamiliar with this notion may momentarily jump to \cref{sec_cond_ind} further below, where the definition is stated, and consult \cite[Chapter 8]{Kallenberg2021} and \cite[Chapter 7.3]{ChowTeicher1997} for detailed introductions.

\begin{definition}
\label{def_causal}
	For any $t \in \Z_-$, let $\Sigma_{\Xc,t}$ and $\Sigma_{\Uc,t}$ be the sigma-algebras on $\XS^- \times \US^-$ generated by the maps $T^{-t} \circ \pi_{\XS^-}$ and $T^{-t} \circ \pi_{\US^-}$, respectively.
We call a measure $\mu \in P(\XS^- \times \US^-)$ {\bfi causal} if $\Sigma_{\Xc,t}$ and $\Sigma_{\Uc,0}$ are conditionally independent given $\Sigma_{\Uc,t}$ with respect to $\mu$ for all $t \in \Z_-$.
The set of all causal measures is denoted $\Pcausal$.
\end{definition}

\begin{example}
(i)
Suppose $V \colon \US^- \rightarrow \XS^-$ is measurable and time-invariant, that is, $V \circ T = T \circ V$.
Then, $(V \times \mathrm{id}_{\US^-})_*\Xi$ is causal for any $\Xi \in P(\US^-)$.
(ii)
Suppose $\Zc$ is another completely regular Hausdorff space, $V \colon \Zc^{Z_-} \rightarrow \US^-$ is measurable and time-invariant, $\Seq{Z}$ is a stochastic process on $\Zc$ with mutually independent marginals, and $\Seq{X}$ is an $\XS^-$-valued random variable such that $(\seq{X}{s},\seq{Z}{s})_{s \leq t}$ is independent of $(\seq{Z}{t+1},\dots,\seq{Z}{0})$ for any $t \leq -1$.
Then, the joint law of $(\Seq{X},V(\Seq{Z}))$ is causal.
Causal measures of this type appeared in the related work \cite{RC28}.
\end{example}

Let us denote the sets of causal stochastic solutions and outputs by $\StochCausalSol = \StochSol \cap \Pcausal$ and $\StochCausalOut = H_*(\StochCausalSol)$.
The {\bfi causal stochastic ESP} and the {\bfi causal stochastic FMP} are defined in the obvious analogous ways.
The following result is known in the deterministic case \cite[Proposition 3]{ManjunathJaeger2013}.
We prove that it extends to the stochastic case.

\begin{proposition}
\label{prop_stoch_per_sol}
	Suppose $\US = \tau^{-1}(\US^-)$ and that the state-space system has the causal stochastic ESP.
Then, the causal solution to a shift-periodic input is also shift-periodic with the same minimal period.
\end{proposition}

Under a compactness assumption, causal solutions always exist for any given input law.
In particular, unique solutions are automatically causal.

\begin{proposition}
\label{prop_stoch_causal_ESP}
	Suppose $(\pi_{\US^-})_* \colon \Pstate^- \rightarrow \Pin^-$ has compact fibers.
Then, $\StochCausalSol$ has non-empty fibers.
In particular, if the state-space system has the stochastic ESP (in the outputs), then it also has the causal stochastic ESP (in the outputs).
\end{proposition}

\begin{example}
	Let us demonstrate that not all stochastic solutions are causal.
Suppose the state map $f(x,u) = x$ does not depend on the input.
Let $(\Seq{X},\Seq{U})$ be an $\XS^- \times \US^-$-valued random variable such that $\seq{U}{0}$ is not $T(\Seq{U})$-measurable and $\seq{X}{t} = \seq{U}{0}$ almost surely for all $t \in \Z_-$.
Then, their joint law belongs to $\StochSol$ and is not causal.
\end{example}

It is an open problem whether the causal analog of \cref{thrm_stoch_ESP_FMP} holds, that is, whether the causal stochastic ESP implies the causal stochastic FMP.
However, using causality, we are able to recover the generic fading memory in the stochastic case.
For the remaining two results in this subsection, we add the following assumption.

\begin{assumption}
\label{assmpt_concat_conv}
	For any $\Xi,\Xi' \in \Pin^-$, suppose that $\gamma^n_*(\Xi' \otimes \Xi)$ converges to $\Xi$ as $n \rightarrow \infty$ in the topology of $\Pin^-$, where $\gamma^n \colon \US^- \times \US^- \rightarrow \US^-$ denotes the map $\gamma^n(\Seq{u}',\Seq{u}) = (\dots,\seq{u}{-1}',\seq{u}{0}',\seq{u}{-n},\dots,\seq{u}{0})$.
\end{assumption}

\begin{remark}
	The map $\gamma^n$ is continuous by \cref{assmpt_det_SSS}.(iii).
Furthermore, we know from \cref{assmpt_det_SSS}.(iv) that $\gamma^n(\Seq{u}',\Seq{u})$ converges point-wise to $\Seq{u}$ as $n \rightarrow \infty$.
Thus, if the topology on $\Pin^-$ is the weak topology, then \cref{assmpt_concat_conv} is satisfied.
\end{remark}

\begin{proposition}
\label{prop_stoch_FMP_mESP_sol}
	The map $D(\StochCausalSol) \rightarrow \N_0 \cup \{\infty\}$, $\Xi \mapsto \#\StochCausalSol(\Xi)$ is constant.
Furthermore, if we denote this constant by $M$, then $\#\StochCausalSol(\Xi) \geq M$ for any $\Xi \in \Pin^-$.
In particular, if the state-space system has the causal stochastic FMP, then $\#\StochCausalSol(\Xi)$ is constant on $\Pin^-$, and if it has the causal stochastic FMP and $\#\StochCausalSol(\Xi) = 1$ for some $\Xi \in \Pin^-$, then it has the causal stochastic ESP.
\end{proposition}

As in the deterministic case, to extend this result to the outputs, we pose a control-theoretic condition.
Let $C := \{(h(\seq{x}{-1}),\Seq{u}) \colon (\Seq{x},\Seq{u}) \in \DetSol \}$.
Note that the state-space system distinguishes reachable states if and only if there exists a function $g \colon C \rightarrow \Yc$ that satisfies $g(h(\seq{x}{-1}),\Seq{u}) = h(\seq{x}{0})$ for all $(\Seq{x},\Seq{u}) \in \DetSol$.
This motivates the next definition.

\begin{definition}
	We say that the state-space system {\bfi measurably distinguishes reachable states} if the set $C$ is Borel measurable and there exists a measurable function $g \colon C \rightarrow \Yc$ that satisfies $g(h(\seq{x}{-1}),\Seq{u}) = h(\seq{x}{0})$ for all $(\Seq{x},\Seq{u}) \in \DetSol$.
\end{definition}

\begin{proposition}
\label{prop_stoch_FMP_mESP_out}
	If the state-space system measurably distinguishes reachable states, then the analog statement of \cref{prop_stoch_FMP_mESP_sol} holds with $\StochCausalOut$ in place of $\StochCausalSol$.
\end{proposition}

In \cref{thrm_det_generic_FMP}, we uncovered that bifurcations of deterministic solutions can only occur within a meager set.
This was obtained by combining the generic FMP with the result that the number of solutions is constant on the domain of the FMP.
In the stochastic context, we saw in \cref{thrm_stoch_ESP_FMP} that the stochastic FMP holds generically.
However, the analog claim about bifurcations of solutions is proved only for causal stochastic solutions.
With causality, we do not know if the causal stochastic FMP holds generically.
Thus, we cannot infer the same result as \cref{thrm_det_generic_FMP} about bifurcations of stochastic solutions on a meager set.
We pose this open problem as a conjecture.

\begin{conjecture}
	Suppose that $\Pin^-$ and $\Pstate^-$ are Polish and that $(\pi_{\US^-})_* \colon \Pstate^- \rightarrow \Pin^-$ is proper.
Then, $\#\StochCausalSol(\Xi)$ is generically constant, and so is $\#\StochCausalOut(\Xi)$ if the state-space system measurably distinguishes reachable states.
\end{conjecture}

\section{On abstract dynamical systems}
\label{sec_abstract}

The results on state-space systems that we presented in the previous sections can (in part) be proved in a more general framework for abstract dynamical systems.
On the one hand, this gives us insights on exactly which properties of state-space systems are the ones that enable us to prove said results.
On the other hand, we will see that both the deterministic and the stochastic results can be accommodated in one unifying theory.

In \cref{subsec_dyn_sequence}, we teased a connection between solutions and attractors on the level of abstract dynamical systems.
This is the content of \cref{subsec_abstr_sol}.
Dynamical systems are broadly categorized in autonomous and non-autonomous ones.
We review how the notions of a solution for these systems are related to each other and propose a unifying framework.
Being intimately related to these solutions, we will also encounter generalized synchronizations \cite{ErogluLambPereira2017, KocarevParlitz1996}.
The global attractor of a dynamical system will make its appearance naturally along the way.

In \cref{subsec_dep_sol_base}, we prove the majority of results about how sets of solutions depend on a base point, which corresponds to the input in the context of state-space systems, ranging from continuous dependence to bifurcations.
Thereafter, \cref{sec_abstract_stochastic} discusses the distributional behavior of a random state of the dynamical system.

\subsection{Solutions of dynamical systems}
\label{subsec_abstr_sol}

The notions of a solution of a dynamical system in the following three subsections are standard material \cite{KloedenRasmussen2011, KloedenEtal2013} but need to be reviewed to set the stage for our proposed unifying framework in \cref{subsec_bundlesystems}.

\subsubsection{Autonomous systems}
\label{subsec_abstract_aut}

Let $\Phi \colon \X \rightarrow \X$ be an autonomous dynamical system.
A subset $A \subseteq \X$ is called forward-invariant if $\Phi(A) \subseteq A$ and strictly invariant if $\Phi(A) = A$.
Recall that a {\bfi solution} of $\Phi$ through $x \in \X$ is a bi-infinite trajectory through $x$, that is, a map $\chi \colon \Z \rightarrow \X$ satisfying $\chi(0) = x$ and $\chi(t+n) = \Phi^n(\chi(t))$ for all $t \in \Z$ and $n \in \N$ (but it suffices to check this for $n=1$).
We denote the set of solution points as
\begin{equation*}
	\Sc_{\Phi}
	:= \{ x \in \X \colon \text{there exists a solution of } \Phi \text{ through } x \}.
\end{equation*}
It is not difficult to show that $\Sc_{\Phi}$ is the maximal strictly invariant set, that is, $\Sc_{\Phi}$ is strictly invariant and any other strictly invariant subset $A \subseteq \X$ is contained in $\Sc_{\Phi}$.
In particular, $\Sc_{\Phi}$ contains all equilibria and periodic orbits.
If $\Phi$ is injective and $\Theta \colon \X \rightarrow \X$ is a left-inverse of $\Phi$, then $\chi \colon \Z \rightarrow \X$ is a solution of $\Phi$ through $x$ if and only if $\chi(n) = \Phi^n(x)$, $\chi(-n) = \Theta^n(x)$, and $\Phi^n(\Theta^n(x)) = x$ for all $n \in \N_0$.
In this case, $\Sc_{\Phi}$ becomes the maximal set on which $\Theta^n$ is also a right-inverse of $\Phi^n$ for any $n \in \N$;
\begin{equation*}
	\Sc_{\Phi}
	= \{ x \in \X \colon \Phi^n(\Theta^n(x)) = x \text{ for all } n \in \N \}.
\end{equation*}
In particular, if $\Phi$ is injective, then it follows easily that $\Sc_{\Phi} = \bigcap_{n \in \N} \Phi^n(\X)$.
There are competing definitions of an attractor of a dynamical system \cite{Milnor1985, OljaAshwinRasmussen2024}, discussing which goes beyond the scope of this paper.
In this paper, we call $\bigcap_{n \in \N} \Phi^n(\X)$ the {\bfi global attractor}.
Thus, if $\Phi$ is injective, then $\Sc_{\Phi}$ coincides with the global attractor.
We emphasize that injectivity is sufficient but not necessary for this.

\subsubsection{Non-autonomous systems: Processes}

Denote $\Z^2_{\geq} := \{(t_1,t_0) \in \Z \times \Z \colon t_1 \geq t_0\}$.
Let $\Psi \colon \Z^2_{\geq} \times \Xc \rightarrow \Xc$ be a non-autonomous dynamical system in the language of processes, that is, a map that satisfies the initial value condition $\Psi(t_0,t_0,x) = x$ and the evolution property $\Psi(t_2,t_0,x) = \Psi(t_2,t_1,\Psi(t_1,t_0,x))$ for all $t_2 \geq t_1 \geq t_0$.
Since the evolution of a process depends on the initial time, it matters at which time a solution passes through a point.
A solution of $\Psi$ through $x \in \Xc$ at time $t \in \Z$ is a map $\chi \colon \Z \rightarrow \Xc$ that satisfies $\chi(t) = x$ and $\chi(t_1) = \Psi(t_1,t_0,\chi(t_0))$ for all $(t_1,t_0) \in \Z^2_{\geq}$.
The set of solution points becomes a subset of $\Xc \times \Z$, namely
\begin{equation*}
	\Sc_{\Psi}^{\mathrm{process}}
	:= \{ (x,t) \in \Xc \times \Z \colon \text{there exists a solution of } \Psi \text{ through } x \text{ at time } t \}.
\end{equation*}
It is well known that a process can be turned into an autonomous dynamical system by augmenting the state space with a time coordinate.
Specifically, the map $\Phi \colon \Xc \times \Z \rightarrow \Xc \times \Z$, $(x,t) \mapsto (\Psi(t+1,t,x),t+1)$ satisfies $\Phi^{t_1-t_0}(x,t_0) = (\Psi(t_1,t_0,x),t_1)$ for all $(t_1,t_0) \in \Z^2_{\geq}$.
If $\chi$ is a solution of the process $\Psi$ through $x$ at time $t$, then the augmented map $\chi' \colon \Z \rightarrow \Xc \times \Z$, $s \mapsto (\chi(t+s),t+s)$ is a solution of the autonomous dynamical system $\Phi$ through $(x,t)$.
Conversely, if $\chi'$ is a solution of $\Phi$ through $(x,t)$, then the projection of $\chi'$ to the $\Xc$-coordinate is a solution of $\Psi$ through $x$ at time $t$.
Thus, $\Sc_{\Psi}^{\mathrm{process}} = \Sc_{\Phi}$.

\subsubsection{Non-autonomous systems: Skew-products}
\label{subsec_skewproducts}

Let $\phi \colon \Zc \rightarrow \Zc$ be an invertible autonomous dynamical system, and let $\Psi \colon \N_0 \times \Zc \times \Xc \rightarrow \Xc$ be a non-autonomous dynamical system driven by $\phi$ in the language of skew-products, that is, a map that satisfies the initial value condition $\Psi(0,z,x) = x$ and the co-cycle property $\Psi(n+m,z,x) = \Psi(n,\phi^m(z),\Psi(m,z,x))$ for all $m,n \in \N_0$.
Let $\mathrm{orb}_{\phi}(z) := \bigcup_{t \in \Z} \{\phi^t(z)\}$ denote the bi-infinite orbit of a point $z \in \Zc$ under $\phi$.
In the skew-product formalism, a solution of $\Psi$ through $x \in \Xc$ with input $z \in \Zc$ is a map $\chi \colon \mathrm{orb}_{\phi}(z) \rightarrow \Xc$ that satisfies $\chi(z) = x$ and $\chi(\phi^{t_1}(z)) = \Psi(t_1-t_0,\phi^{t_0}(z),\chi(\phi^{t_0}(z)))$ for all $(t_1,t_0) \in \Z^2_{\geq}$.
The set of solution points is analogously
\begin{equation*}
	\Sc_{\Psi}^{\mathrm{skew\text{-}product}}
	:= \{ (x,z) \in \Xc \times \Zc \colon \text{there exists a solution of } \Psi \text{ through } x \text{ with input } z \}.
\end{equation*}
As for processes, we may augment the state space of the skew-product and consider the autonomous dynamical system $\Phi \colon \Xc \times \Zc \rightarrow \Xc \times \Zc$, $(x,z) \mapsto (\Psi(1,z,x),\phi(z))$, which satisfies $\Phi^n(x,z) = (\Psi(n,z,x),\phi^n(z))$ for all $n \in \N_0$.
If $\chi$ is a solution of the skew-product $\Psi$ through $x$ with input $z$, then the augmented map $\chi' \colon \Z \rightarrow \Xc \times \Zc$, $t \mapsto (\chi(\phi^t(z)),\phi^t(z))$ is a solution of $\Phi$ through $(x,z)$.
Unlike for processes, the converse does not hold in general.
Thus, $\Sc_{\Psi}^{\mathrm{skew\text{-}product}} \subseteq \Sc_{\Phi}$, and the inclusion may be strict.
In the next section, we characterize the points that belong to $\Sc_{\Phi}$ but not to $\Sc_{\Psi}^{\mathrm{skew\text{-}product}}$.

More common in the literature is the notion of an {\bfi entire solution} $\chi \colon \Zc \rightarrow \Xc$, which is defined by the same characteristic property $\chi(\phi^{t_1}(z)) = \Psi(t_1-t_0,\phi^{t_0}(z),\chi(\phi^{t_0}(z)))$ but has domain $\Zc$.
Clearly, if $\chi$ is an entire solution, then its restriction to $\mathrm{orb}_{\phi}(z)$ is a solution of $\Psi$ through $\chi(z)$ with input $z$.
We could denote
\begin{equation*}
	\Sc_{\Psi}^{\mathrm{entire}}
	:= \{ (x,z) \in \Xc \times \Zc \colon \text{there exists an entire solution } \chi \text{ of } \Psi \text{ with } \chi(z) = x \}
	\subseteq \Sc_{\Psi}^{\mathrm{skew\text{-}product}}.
\end{equation*}
The subtlety of entire solutions lies in the dependence of existence of solutions for various inputs:
if a solution with input $z$ does not exist for some $z \in \Zc$, then neither does an entire solution exist, and then $\Sc_{\Psi}^{\mathrm{entire}}$ is empty even if $\Sc_{\Psi}^{\mathrm{skew\text{-}product}}$ is non-empty.
This makes the concept of an entire solution useful only if we have some additional knowledge of the dynamical system that guarantees the existence of at least one solution with each input.
If this guarantee is met, then $\Sc_{\Psi}^{\mathrm{entire}} = \Sc_{\Psi}^{\mathrm{skew\text{-}product}}$.

The non-autonomous dynamical system $\Psi$ has the {\bfi unique solution property (USP)} if for each $z \in \Zc$ there exists a unique solution of $\Psi$ with input $z$.
The USP is equivalent to the existence of a \textit{unique} entire solution.
In this case, that unique entire solution is called the {\bfi generalized synchronization} \cite{KocarevParlitz1996}.

\subsubsection{Autonomous systems on bundles}
\label{subsec_bundlesystems}

We introduce an overarching framework.
Consider a map $\pi \colon \X \rightarrow \Zc$ and an invertible autonomous dynamical system $\phi \colon \Zc \rightarrow \Zc$.
A {\bfi bundle map} over $\phi$ is a map $\Phi \colon \X \rightarrow \X$ that satisfies $\pi \circ \Phi = \phi \circ \pi$.
The autonomous dynamical system obtained from a skew-product in the previous section fits the special case $\X = \Xc \times \Zc$.
The autonomous dynamical system obtained from a process can be regarded as a bundle map over $\Z \rightarrow \Z$, $t \mapsto t+1$.
In fact, any autonomous dynamical system can be regarded as a bundle map by taking $\Zc$ to be a one-point space if no additional structure is available.

\begin{example}
\label{ex_bundle_dynamics}
	The push-forward $\varphi_* \colon \Pstate \rightarrow \Pstate$ from \cref{sec_stoch_SSS} is a bundle map over $\sigma_* \colon \Pin \rightarrow \Pin$ on the bundle $(\pi_{\Uc})_* \colon \Pstate \rightarrow \Pin$.
This dynamical system cannot be written as a skew-product since $P(\XS^- \times \US) \neq P(\XS^-) \times P(\US)$.
This exemplifies the need for autonomous systems on bundles as a generalization of skew-products.
\end{example}

\begin{remark}
	Suppose $\Phi$ is injective.
Then, $\Sc_{\Phi} = \bigcap_{n \in \N} \Phi^n(\X)$ as noted in \cref{subsec_abstract_aut}.
Since $\Phi$ is a bundle map over $\phi$, we have $\Sc_{\Phi} \cap \pi^{-1}(z) = \bigcap_{n \in \N} \Phi^n(\pi^{-1}(\phi^{-n}(z)))$.
Thus, if $\Phi$ encodes a non-autonomous dynamical system, say by augmentation, then $\Sc_{\Phi}$ corresponds to its global pullback attractor.
\end{remark}

Since $\Phi$ is an autonomous dynamical system, the initial notion of solution from \cref{subsec_abstract_aut} persists.
A {\bfi synchronized solution} of $\Phi$ through $x \in \X$ is a map $\chi \colon \mathrm{orb}_{\phi}(\pi(x)) \rightarrow \X$ that satisfies $\chi(\pi(x)) = x$ and $\Phi \circ \chi = \chi \circ \phi$.
The denomination `synchronized' is motivated by the fact that $\chi$ is a semi-conjugacy between $\Phi$ and $\phi$ on the orbit of $z$.
Let
\begin{equation*}
	\Sc_{\Phi}^{\mathrm{sync}}
	= \{ x \in \X \colon \text{there exists a synchronized solution of } \Phi \text{ through } x \}.
\end{equation*}
If $\chi$ is a synchronized solution of $\Phi$ through $x$, then $\chi' \colon \Z \rightarrow \X$, $t \mapsto \chi \circ \phi^t \circ \pi(x)$ is a solution of $\Phi$ through $x$.
Thus, $\Sc_{\Phi}^{\mathrm{sync}} \subseteq \Sc_{\Phi}$.
If $\Phi$ is the autonomous dynamical system obtained from augmenting a skew-product $\Psi$, then $\Sc_{\Phi}^{\mathrm{sync}} = \Sc_{\Psi}^{\mathrm{skew\text{-}product}}$.
To characterize $\Sc_{\Phi} \backslash \Sc_{\Phi}^{\mathrm{sync}}$, we need to consider periodic points.
Let $\Per{\Phi}$ and $\Per{\phi}$ denote the set of periodic points of $\Phi$ and $\phi$.
Since $\Phi$ is a bundle map over $\phi$, we have $\Per{\Phi} \subseteq \pi^{-1}(\Per{\phi})$, and the inclusion may be strict.
Similarly, let $\APer{\Phi} = \X \backslash \Per{\Phi}$ and $\APer{\phi} = \Zc \backslash \Per{\phi}$ be the set of aperiodic points of $\Phi$ and $\phi$.
Introduce the set of {\bfi synchronized periodic points} of $\Phi$ and $\phi$ as the set of all periodic points $x$ of $\Phi$ whose minimal period is not greater than the minimal period of $\pi(x)$ under $\phi$, that is,
\begin{equation*}
	\SPer{\Phi}{\phi}
	:= \{ x \in \X \colon \exists\, n \in \N \colon \Phi^n(x) = x \text{ and } \forall\, 1 \leq m < n \colon \phi^m(\pi(x)) \neq \pi(x) \}.
\end{equation*}
Heuristically, if we drive a system with a periodic input, we would hope that the system creates a stable response to the input with the same periodicity.
The following proposition shows that this heuristic is captured by the synchronized solutions.

\begin{proposition}
\label{prop_driver_comp_sol}
	The set of synchronized solution points is characterized by
\begin{equation*}
	\Sc_{\Phi}^{\mathrm{sync}}
	= \Big( \Sc_{\Phi} \cap \pi^{-1}( \APer{\phi} ) \Big) \cup \SPer{\Phi}{\phi}
\end{equation*}
and, hence,
\begin{equation*}
	\Sc_{\Phi} \backslash \Sc_{\Phi}^{\mathrm{sync}}
	= \Big( \Sc_{\Phi} \cap \APer{\Phi} \cap \pi^{-1}( \Per{\phi} ) \Big) \cup \Big( \Per{\Phi} \backslash \SPer{\Phi}{\phi} \Big).
\end{equation*}
\end{proposition}

\begin{proof}
	Observe that $\Sc_{\Phi}$ is the disjoint union of the four sets $S_0 := \Sc_{\Phi} \cap \pi^{-1}( \APer{\phi} )$, $S_1 := \SPer{\Phi}{\phi}$, $S_2 := \Sc_{\Phi} \cap \APer{\Phi} \cap \pi^{-1}( \Per{\phi} )$, and $S_3 := \Per{\Phi} \backslash \SPer{\Phi}{\phi}$.
First, if $\chi$ is a synchronized solution of $\Phi$ through $x$ and if $z := \pi(x)$ is periodic under $\phi$ with period $n \in \N$, then $\Phi^n(x) = \Phi^n \circ \chi(z) = \chi \circ \phi^n(z) = \chi(z) = x$.
This shows that $S_2 \cup S_3 \subseteq \Sc_{\Phi} \backslash \Sc_{\Phi}^{\mathrm{sync}}$.
Second, if $x \in \SPer{\Phi}{\phi}$ and if $n \in \N$ is the minimal period of $z := \pi(x)$ under $\phi$, then the map $\{z,\phi(z),\dots,\phi^{n-1}(z)\} \rightarrow \X$, $\phi^t(z) \mapsto \Phi^t(x)$ defines a synchronized solution.
Thus, $S_1 \subseteq \Sc_{\Phi}^{\mathrm{sync}}$.
Lastly, if $\chi \colon \Z \rightarrow \X$ is a solution of $\Phi$ through $x$ and if $\pi(x) \in \APer{\phi}$, then the map $\mathrm{orb}_{\phi}(\pi(x)) \rightarrow \X$, $\phi^t(\pi(x)) \mapsto \chi(t)$ is well-defined and a synchronized solution.
Thus, $S_0 \subseteq \Sc_{\Phi}^{\mathrm{sync}}$.
\end{proof}

On the one hand, if $\Phi$ is a bundle map over $\phi \colon \Z \rightarrow \Z$, $t \mapsto t+1$, e.g.\ obtained from augmenting a process, then $\Per{\phi}$ is empty and $\Sc_{\Phi} = \Sc_{\Phi}^{\mathrm{sync}}$.
On the other hand, it is easy to construct examples of skew-products such that the bundle map obtained from augmentation admits points in either $\Sc_{\Phi} \cap \APer{\Phi} \cap \pi^{-1}( \Per{\phi} )$ or $\Per{\Phi} \backslash \SPer{\Phi}{\phi}$.
In particular, the maximal invariant set of a skew-product is, in general, not equal to the collection of solutions of the system.
This corrects a remark made in \cite{KloedenEtal2012}.

Extending the notion for skew-products, we can introduce an {\bfi entire solution} as a map $\chi \colon \Zc \rightarrow \X$ that satisfies $\Phi \circ \chi = \chi \circ \phi$ and $\pi \circ \chi = \mathrm{id}_{\Zc}$;
and we can say that $\Phi$ has the {\bfi unique solution property (USP)} if for every $z \in \Zc$ there exists a unique $x \in \pi^{-1}(z)$ through which there exists a solution of $\Phi$.
In other words, $\Phi$ has the USP if the fiber $\Sc_{\Phi} \cap \pi^{-1}(z)$ is a singleton for all $z \in \Zc$.
The same comment applies that we made for entire solutions of skew-products:
if there exists some $z \in \Zc$ such that through any $x \in \pi^{-1}(z)$ no solution exists, then neither does an entire solution exist.
However, the USP is equivalent to the existence of a unique entire solution, which is then the unique left-inverse of $\pi$ with range $\Sc_{\Phi}$.
As before, in this case, that unique entire solution is called the {\bfi generalized synchronization (GS)}.
In the language of autonomous systems on bundles, by definition, the generalized synchronization is a semi-conjugacy between the base dynamical system $\phi$ and the restriction of the dynamical system $\Phi$ to $\Sc_{\Phi}$.
In particular, that the dynamics of $\Phi$ on $\Sc_{\Phi}$ become synchronized with $\phi$ implies that $\Sc_{\Phi} \backslash \Sc_{\Phi}^{\mathrm{sync}}$ is empty, which we make rigorous in the corollary below.
This corollary generalizes \cite[Proposition 3]{ManjunathJaeger2013}.

\begin{corollary}
\label{cor_USP_driver_comp}
	Suppose $\Phi$ has the USP.
Then, $\Sc_{\Phi} = \Sc_{\Phi}^{\mathrm{sync}}$.
\end{corollary}

\begin{proof}
	By \cref{prop_driver_comp_sol}, we have to show that $\Sc_{\Phi} \cap \pi^{-1}(\Per{\phi}) \subseteq \SPer{\Phi}{\phi}$.
Let $\chi \colon \Zc \rightarrow \X$ be the generalized synchronization, and let $x \in \Sc_{\Phi}$.
If $z:= \pi(x)$ is periodic under $\phi$ with period $n \in \N$, then $\Phi^n(x) = \Phi^n \circ \chi(z) = \chi \circ \phi^n(z) = \chi(z) = x$, which shows that $x \in \SPer{\Phi}{\phi}$.
\end{proof}

Being a right-inverse of $\pi$, the generalized synchronization --- if it exists --- is always injective.
If $\X = \Xc \times \Zc$ is a product space, then the GS is a map of the form $\chi = \zeta \times \mathrm{id}_{\Zc}$ with $\zeta \colon \Zc \rightarrow \Xc$.
The GS for the skew-product is the standalone map $\zeta \colon \Zc \rightarrow \Xc$.
The map $\zeta$ is not necessarily injective, but if it is, then its left-inverse is a conjugacy between the true underlying dynamics of $\phi$ and the dynamics of the states of the skew-product on its pullback attractor.
The GS of an autonomous system on a bundle also defines a conjugacy but not the desired one, because the `extended states' in $\Xc \times \Zc$ contain information about the inputs.

\vspace*{1mm}\noindent
\begin{minipage}{0.6\textwidth}
\hspace*{15pt}%
To unify the concept of the GS, we have to compose the GS of an autonomous system on a bundle with the projection map $\Xc \times \Zc \rightarrow \Xc$.
Generalizing this idea to the case of non-product spaces $\X$, we have to compose the GS $\chi \colon \Zc \rightarrow \X$ with a map $\Pi \colon \X \rightarrow \Xc$, which we think of as a state-extraction map that forgets the encoding of the input.
Now, if $\Pi \circ \chi$ happens to be injective with left-inverse $\eta \colon \Pi(\Sc_{\Phi}) \rightarrow \Zc$, then we get the commutative diagram on the right, where $\hat{\Phi} = \Pi \circ \Phi \circ \chi \circ \eta \colon \Pi(\Sc_{\Phi}) \rightarrow \Pi(\Sc_{\Phi})$ captures the `states-only' dynamics of $\Phi$.
In particular, $\eta$ is a conjugacy between $\phi$ and $\hat{\Phi}$.
\end{minipage}%
\hfill%
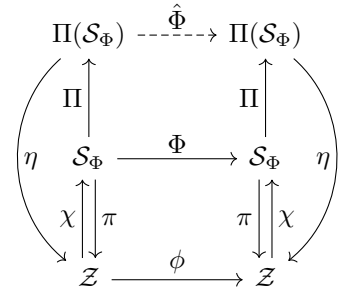
\begin{minipage}{0.4\textwidth}
\begin{center}
\begin{tikzcd}[row sep = 3em, column sep = 3em, every label/.append style = {font = \normalsize}]
	\Pi(\Sc_{\Phi}) \arrow[bend right = 50]{dd}{\eta} \arrow[dashed]{r}{\hat{\Phi}} \pgfmatrixnextcell \arrow[swap, bend left = 50]{dd}{\eta} \Pi(\Sc_{\Phi})
	\\
	\arrow{u}{\Pi} \Sc_{\Phi} \arrow[shift left=1]{d}{\pi} \arrow{r}{\Phi} \pgfmatrixnextcell \Sc_{\Phi} \arrow[swap, shift right=1]{d}{\pi} \arrow{u}{\Pi}
	\\
	\Zc \arrow[shift left=1]{u}{\chi} \arrow{r}{\phi} \pgfmatrixnextcell \Zc \arrow[swap, shift right=1]{u}{\chi}
\end{tikzcd}%
\end{center}
\end{minipage}

\subsubsection{Parametrized autonomous systems}
\label{subsec_abstract_param}

Consider a parametrized family $(\Phi_z)_{z \in \Zc}$ of autonomous dynamical systems $\Phi_z \colon \Xc \rightarrow \Xc$.
Let $\X = \Xc \times \Zc$ and $\pi = \pi_{\Zc} \colon \X \rightarrow \Zc$ be the projection.
Then, $\Phi(x,z) := (\Phi_z(x),z)$ is a bundle map $\X \rightarrow \X$ over the identity $\Zc \rightarrow \Zc$.
It is easy to see that $\Sc_{\Phi_z} \times \{z\} = \Sc_{\Phi} \cap \pi_{\Zc}^{-1}(z)$.
Thus, a parametrized family of solution sets can be translated to the solution set of a single bundle map parametrized by its fibers.

We show that the converse can (almost) be done even if $\Phi \colon \Xc \times \Zc \rightarrow \Xc \times \Zc$ is a bundle map over a bijection $\phi \colon \Zc \rightarrow \Zc$ that is not necessarily the identity.
For all $z \in \Zc$, consider the autonomous dynamical system $\Phi_z \colon \Xc \times \Z \rightarrow \Xc \times \Z$ given by $\Phi_z(x,t) = (\pi_{\Xc} \circ \Phi(x,\phi^t(x)),t+1)$ (which is the autonomous system obtained from augmenting the process $\Psi_z$ on $\Xc$ determined by $\Psi_z(t+1,t,x) = \pi_{\Xc} \circ \Phi(x,\phi^t(x))$).
If $\chi \colon \Z \rightarrow \Xc \times \Z$ is a solution of $\Phi_z$ through $(x,0)$, then $\chi' \colon \Z \rightarrow \X$, $t \mapsto (\pi_{\Xc} \circ \chi(t),\phi^t(z))$ is a solution of $\Phi$ through $(x,z)$.
Conversely, if $\chi \colon \Z \rightarrow \Xc \times \Zc$ is a solution of $\Phi$ through $(x,z)$, then $\chi' \colon \Z \rightarrow \Xc \times \Z$, $t \mapsto (\pi_{\Xc} \circ \chi(t),t)$ is a solution of $\Phi_z$ through $(x,0)$.
Thus, the fiber $\Sc_{\Phi} \cap \pi_{\Zc}^{-1}(z)$ is in bijection with the fiber $\Sc_{\Phi_z} \cap \pi_{\Z}^{-1}(0)$ via $(x,z) \mapsto (x,0)$.
If $\phi$ is not the identity, we cannot get around making the domain of $\Phi_z$ partially discrete since $\Phi$ encodes non-autonomous dynamics.

The preceding paragraphs show that parametrized families of autonomous dynamical systems and autonomous systems on trivial bundles are equivalent if we are interested in their solution sets.
However, systems on bundles become more general when $\X$ is not a product space $\Xc \times \Zc$ as in \cref{ex_bundle_dynamics}.

\subsection{Dependence of solutions on base points}
\label{subsec_dep_sol_base}

For the remainder of \cref{sec_abstract}, let $\X$ and $\Zc$ be Hausdorff spaces and $\pi \colon \X \rightarrow \Zc$ be a continuous map.
Let $\phi \colon \Zc \rightarrow \Zc$ be a continuous bijection and $\Phi$ be a continuous bundle map over $\phi$, that is, $\pi \circ \Phi = \phi \circ \pi$.
In this section, we study how the fibers of $\Sc_{\Phi}$ vary with the base point.

\subsubsection{Continuity of set of solutions}

An important question in dynamical systems theory is whether the fibers of the global attractor depend continuously on the base point.
To address this question rigorously, we recall the notion of hemi-continuity, which generalizes the concept of continuity to set-valued maps in general topological spaces \cite{AubinFrankowska2009}.
Given a point $z_0 \in \Zc$, a set-valued map $S \colon \Zc \rightarrow 2^{\X}$ is {\bfi upper hemi-continuous} at $z_0$ if for any open set $W \subseteq \X$ containing $S(z_0)$ there exists an open set $V \subseteq \Zc$ containing $z_0$ such that $S(z) \subseteq W$ for all $z \in V$;
{\bfi lower hemi-continuous} at $z_0$ if for any open set $W \subseteq \X$ with $W \cap S(z_0) \neq \emptyset$ there exists an open set $V \subseteq \Zc$ containing $z_0$ such that $W \cap S(z) \neq \emptyset$ for all $z \in V$;
and {\bfi hemi-continuous} at $z_0$ if it is both upper and lower hemi-continuous at $z_0$.
A set-valued map is (upper/lower) hemi-continuous if it is (upper/lower) hemi-continuous at all points.\footnote{
Here, we consider general set-valued maps on Hausdorff spaces.
When $\X$ is a metric space and $S$ is compact-set-valued, then continuity of $S$ is typically defined as continuity with respect to the Hausdorff distance, which defines a metric on the set of non-empty compact subsets of $\X$.
In this case, $S$ is hemi-continuous if and only if $S$ is continuous with respect to the Hausdorff distance.
%In general, it is \textit{not} true that $S$ is upper (lower) hemi-continuous if and only if $S$ is upper (lower) semi-continuous with respect to the Hausdorff distance.
}
Note that if $S(z)$ is a singleton for any $z \in \Zc$, then upper hemi-continuity, lower hemi-continuity, hemi-continuity of $S$, and usual continuity of $S$ as a map $\Zc \rightarrow \X$ are all equivalent.
Using the bundle structure, by a slight abuse of notation, we say that a subset $S \subseteq \X$ is (upper/lower) hemi-continuous if the set-valued map $z \mapsto S \cap \pi^{-1}(z)$ is (upper/lower) hemi-continuous.
Under some assumptions on $\X$, $\Zc$, and $\pi$, closed subsets turn out to be automatically upper hemi-continuous and generically hemi-continuous; see \cref{prop_closed_upper_hemi}.
Recall that a map is proper if it is closed and has compact fibers;
and that a subset of a topological space is called residual if it can be written as the intersection of countably many subsets each of which has a dense interior.

\begin{proposition}
\label{prop_closed_upper_hemi}
	Suppose $\pi$ is proper, and let $S \subseteq \X$ be closed.
Then, $S$ is upper hemi-continuous.
Furthermore, if $\X$ and $\Zc$ are Polish, then the set of points in $\Zc$ at which $S$ is hemi-continuous is residual.
\end{proposition}

\begin{proof}
	Suppose for contradiction that $S$ is not upper hemi-continuous.
This entails that there exist an element $z \in \Zc$, an open set $W \subseteq \X$ containing $S \cap \pi^{-1}(z)$, a net $(z_i)_{i \in I}$ indexed by the open neighborhoods of $z$ and converging to $z$, and for each $i \in I$ an element $x_i \in S \cap \pi^{-1}(z_i) \cap W^c$.
Since $\pi$ is proper, $(x_i)_{i \in I}$ admits some cluster point $x \in \pi^{-1}(z)$; see \cref{lem_proper}.
Being a closed subset, $S \cap W^c$ must contain $x$, which contradicts $x \in S \cap \pi^{-1}(z) \subseteq W$.
This shows that $S$ is upper hemi-continuous.
It is known that an upper hemi-continuous map from a completely metrizable space to subsets of a Polish space is hemi-continuous on a residual set \cite[Theorem 1.4.13]{AubinFrankowska2009}.
\end{proof}

Given a parametrized family of autonomous dynamical systems, it has been studied when the individual attractors depend hemi-continuously on the parameter \cite{BabinPilyugin1997, DeshengKloeden2004, HoangOlsonRobinson2015, KloedenEtal2013}.
With our discussion in \cref{subsec_abstract_param}, we are able to recover those results that prove residual continuity of attractors of continuously parametrized dynamical systems on a compact Polish space.

\begin{proposition}
\label{prop_hemicont_param_attr}
	Suppose $\Zc$ is Polish and $\Phi_z \colon \Xc \rightarrow \Xc$, $z \in \Zc$, are autonomous dynamical systems on a compact Polish space $\Xc$ such that $\Xc \times \Zc \rightarrow \Xc$, $(x,z) \mapsto \Phi_z(x)$ is continuous.
Then, $\Sc_{\Phi_z} = \bigcap_{n \in \N} \Phi_z^n(\Xc)$ for all $z \in \Zc$, and the set $\Sc_{\Phi_z}$ depends hemi-continuously on $z$ in a residual set.
\end{proposition}

\begin{proof}
	As in \cref{subsec_abstract_param}, consider $\X = \Xc \times \Zc$ and $\Phi \colon \Xc \times \Zc \rightarrow \Xc \times \Zc$, $(x,z) \mapsto (\Phi_z(x),z)$.
Note that the projection $\Xc \times \Zc \rightarrow \Zc$ is proper and the sets $\Phi^n(\Xc \times \Zc)$, $n \in \N$, are closed by compactness of $\Xc$.
We show that $\Sc_{\Phi} = \bigcap_{n \in \N} \Phi^n(\Xc \times \Zc) =: \Ac$.
It is easy to see that $\Sc_{\Phi}$ is contained in $\Ac$.
Since $\Sc_{\Phi}$ is the maximal strictly invariant set, it therefore suffices to show that $\Ac$ is strictly invariant.
That $\Ac$ is forward-invariant is clear.
That $\Ac \subseteq \Phi(\Ac)$ can be shown by a standard argument, which we briefly repeat here.
Suppose $(x,z) \in \Ac$.
In particular, there exists a sequence $(x_n)_{n \in \N} \subseteq \Xc$ with $(x_n,z) \in \Phi^n(\Xc \times \Zc)$ and $(x,z) = \Phi(x_n,z)$.
By compactness, $(x_n)_{n \in \N}$ has some convergent subnet with limit $(x',z)$.
Since $\Phi^n(\Xc \times \Zc)$ is a nested sequence of closed sets, we necessarily have $(x',z) \in \Ac$.
By continuity, $(x,z) = \Phi(x',z) \in \Phi(\Ac)$.
We have shown that $\Sc_{\Phi} = \Ac$ is the intersection of closed sets, hence itself closed.
This enables us to apply \cref{prop_closed_upper_hemi} to conclude that $\Sc_{\Phi_z}$ depends hemi-continuously on $z$ in a residual set.
The same argument as above shows that $\Sc_{\Phi_z} = \bigcap_{n \in \N} \Phi_z^n(\Xc)$.
\end{proof}

We conclude this section with an application of \cref{prop_closed_upper_hemi} to dynamical systems with the USP.

\begin{proposition}
	Suppose $\pi$ is proper and $\Phi$ has the USP.
Furthermore, suppose $\X$ is compact or $\Phi$ admits a continuous left-inverse.
Then, $\Sc_{\Phi} = \bigcap_{n \in \N} \Phi^n(\X)$ is (hemi-)continuous.
\end{proposition}

\begin{proof}
	We pointed out further above that upper hemi-continuity, hemi-continuity, and usual continuity agree for singleton-valued maps.
Hence, by \cref{prop_closed_upper_hemi}, we only have to verify that $\Sc_{\Phi}$ equals $\bigcap_{n \in \N} \Phi^n(\X)$ and is a closed set.
For compact $\X$, this is done exactly as in the proof of \cref{prop_hemicont_param_attr}.
We noted in \cref{subsec_abstract_aut} that if $\Phi$ admits a continuous left-inverse, then $\Sc_{\Phi}$ coincides with the global attractor.
Furthermore, the existence of a continuous left-inverse entails that $\Phi$ is closed; see \cref{lem_proper}.
In particular, the global attractor is a closed set.
\end{proof}

\subsubsection{Bifurcations}

Beyond continuity of the fibers of $\Sc_{\Phi}$, another interesting question is whether something can be said about the number of solutions in a fiber as a function of the base point.
To this end, we review the effect of lower hemi-continuity.
Let $\Y$ and $\Zc_-$ be Hausdorff spaces.

\begin{lemma}
\label{lem_lower_hemi}
	Suppose $\Oc \colon \Zc_- \rightarrow 2^{\Y}$ is lower hemi-continuous at $z_- \in \Zc_-$.
If $(z_-^i)_{i \in I} \subseteq \Zc_-$ is a convergent net with limit $z_-$, then $\# \Oc(z_-) \leq \limsup_i \# \Oc(z_-^i)$.
\end{lemma}

\begin{proof}
	Suppose for contradiction that $\# \Oc(z_-) > \limsup_i \# \Oc(z_-^i) =: M$.
In particular, $M < \infty$.
Take $M+1$ distinct points $y_0,\dots,y_M \in \Oc(z_-)$.
Since $\Y$ is Hausdorff, we can take disjoint open neighborhoods $W_m \subseteq \Y$ of $y_m$.
By lower hemi-continuity, there exists an $i_0 \in I$ such that $W_m \cap \Oc(z_-^i) \neq \emptyset$ for all $i \geq i_0$ and all $0 \leq m \leq M$.
In other words, $\Oc(z_-^i)$ contains at least one point in each of the $M+1$ disjoint sets $W_0,\dots,W_M$, which contradicts $\limsup_i \# \Oc(z_-^i) = M$.
\end{proof}

The hypotheses in the next result may seem artificial but they fit precisely the framework of state-space systems as we will see later.

\begin{proposition}
\label{prop_const_card}
	Suppose that $\Phi$ admits a continuous left-inverse $\Theta \colon \X \rightarrow \X$ and that there exist maps $\tau \colon \Zc \rightarrow \Zc_-$ and $\eta \colon \X \rightarrow \Y$ and a subset $\X_{\infty} \subseteq \X$ that is forward-invariant under $\Theta$ such that the following hold.
\begin{enumerate}[\upshape (i)]\itemsep=0em
\item\label{prop_const_card_assmpt_indisting}
For any $x,x' \in \Sc_{\Phi} \cap \X_{\infty}$, if $\eta(\Theta(x)) = \eta(\Theta(x'))$ and $\pi(x) = \pi(x')$, then $\eta(x) = \eta(x')$.\footnote{
Note that this always holds if $\eta$ is the identity since $x = \Phi(\Theta(x)) = \Phi(\Theta(x')) = x'$ for all $x,x' \in \Sc_{\Phi}$ with $\Theta(x) = \Theta(x')$.}

\item\label{prop_const_card_assmpt_stable_sol}
For any $z,z' \in \Zc$ with $\tau(z) = \tau(z')$, it holds that $\eta(\Sc_{\Phi} \cap \X_{\infty} \cap \pi^{-1}(z)) = \eta(\Sc_{\Phi} \cap \X_{\infty} \cap \pi^{-1}(z'))$.

\item\label{prop_const_card_assmpt_reachable_input}
For all $z_-,z_-' \in \Zc_-$ there exist nets $(z_i)_{i \in I} \subseteq \Zc$ and $(k_i)_{i \in I} \subseteq \N_0$ such that $\tau(z_i) \rightarrow z_-$ and $\tau(\phi^{-k_i}(z_i)) = z_-'$ for all $i \in I$.
\end{enumerate}
Let $\Zc_-^* \subseteq \Zc_-$ be the set of points at which $z_- \mapsto \Oc(z_-) := \eta(\Sc_{\Phi} \cap \X_{\infty} \cap (\tau \circ \pi)^{-1}(z_-))$ is lower hemi-continuous.
Then, $\#\Oc(z_-) \leq \#\Oc(z_-')$ for any $z_- \in \Zc_-^*$ and any $z_-' \in \Zc_-$.
In particular, $\#\Oc(z_-)$ is constant on $\Zc_-^*$.
\end{proposition}

\begin{proof}
	Fix $z_- \in \Zc_-^*$ and $z_-' \in \Zc_-$, and take nets $(z_i)_{i \in I} \subseteq \Zc$ and $(k_i)_{i \in I} \subseteq \N_0$ as in \eqref{prop_const_card_assmpt_reachable_input} such that $z_-^i := \tau(z_i)$ converges to $z_-$ and $\tau(\phi^{-k_i}(z_i)) = z_-'$ for all $i \in I$.
By \cref{lem_lower_hemi}, we have $\#\Oc(z_-) \leq \limsup_i \#\Oc(z_-^i)$.
Fix $i \in I$.
We claim that $\#\Oc(z_-^i) \leq \#\Oc(z_-')$.
By \eqref{prop_const_card_assmpt_stable_sol}, there exists a right-inverse $\iota \colon \Oc(z_-^i) \rightarrow \Sc_{\Phi} \cap \X_{\infty} \cap \pi^{-1}(z_i)$ of $\eta$.
The sets $\Sc_{\Phi}$ and $\X_{\infty}$ are forward-invariant under $\Theta$.
The fiber $\Sc_{\Phi} \cap \pi^{-1}(z_i)$ gets mapped into the fiber $\Sc_{\Phi} \cap \pi^{-1}(\phi^{-k_i}(z_i))$ under $\Theta^{k_i}$.
We show that the composition $\eta \circ \Theta^{k_i} \circ \iota$ maps the set $\Oc(z_-^i)$ injectively into $\eta(\Sc_{\Phi} \cap \X_{\infty} \cap \pi^{-1}(\phi^{-k_i}(z_i)))$.
This is trivial if $k_i = 0$, so we consider the case $k_i \geq 1$.
To this end, suppose $y,y' \in \Oc(z_-^i)$ satisfy $\eta \circ \Theta^{k_i} \circ \iota(y) = \eta \circ \Theta^{k_i} \circ \iota(y')$.
Abbreviate $x_m = \Theta^m \circ \iota(y)$ and $x'_m = \Theta^m \circ \iota(y')$ for $0 \leq m \leq k_i$, and note that $\pi(x_m) = \phi^{-m}(z_i) = \pi(x_m')$.
From \eqref{prop_const_card_assmpt_indisting}, we obtain $\eta(x_{k_i-1}) = \eta(x_{k_i-1}')$ and inductively $\eta(x_{k_i-2}) = \eta(x_{k_i-2}')$ and so forth until $\eta(x_0) = \eta(x_0')$.
But $\eta(x_0) = y$ and $\eta(x_0') = y'$, which shows that $\eta \circ \Theta^{k_i} \circ \iota$ is injective on $\Oc(z_-^i)$.
Lastly, by \eqref{prop_const_card_assmpt_stable_sol}, we have $\eta(\Sc_{\Phi} \cap \X_{\infty} \cap \pi^{-1}(\phi^{-k_i}(z_i))) = \Oc(z_-')$.
This concludes the claim $\#\Oc(z_-^i) \leq \#\Oc(z_-')$, which yields $\#\Oc(z_-) \leq \#\Oc(z_-')$.
This is an equality if $z_-'$ also belongs to $\Zc_-^*$ by symmetry of the argument in $z_-$ and $z_-'$.
\end{proof}

\begin{remark}
	It follows easily from the definition of lower hemi-continuity that if $\tau$ is open, $\eta$ is continuous, and $\Zc^* \subseteq \Zc$ denotes the set of points at which $\Sc_{\Phi} \cap \X_{\infty}$ is lower hemi-continuous, then $\tau(\Zc^*) \subseteq \Zc_-^*$.
\end{remark}

\subsubsection{Structured solutions}
\label{sec_struct_sol}

In \cref{prop_const_card}, we introduced a subset $\X_{\infty} \subseteq \X$ that ensured additional properties of the solutions in $\Sc_{\Phi} \cap \X_{\infty}$.
Next, we study the effect of prescribing a specific structure for the set $\X_{\infty}$.
Throughout this subsection, suppose that $\Phi$ admits a continuous left-inverse $\Theta \colon \X \rightarrow \X$ that is a bundle map over $\phi^{-1}$, that is, $\pi \circ \Theta = \phi^{-1} \circ \pi$, and suppose $\X_0 \subseteq \X$ is forward-invariant under $\Phi$ with closed and non-empty fibers.
Let $\X_{\infty} = \bigcap_{n \in \N_0} (\Theta^n)^{-1}(\X_0)$.
The motivation to introduce such a structure is that the set of causal measures that we encountered in \cref{subsec_causal_sol} in the context of stochastic state-space systems can be written in such a way.
We point out that $\X_{\infty}$ may not be closed since we do not assume $\X_0$ to be closed.

\begin{lemma}
\label{lem_X_inf_closed_fibres}
	The fibers of $\X_{\infty}$ are closed.
Furthermore, $\Phi^{-1}(\X_{\infty}) = \X_{\infty} = \Theta(\X_{\infty}) = \X_0 \cap \Theta^{-1}(\X_{\infty})$.
\end{lemma}

\begin{proof}
	The assumption $\pi \circ \Theta = \phi^{-1} \circ \pi$ implies that $(\Theta^n)^{-1}(\X_0) \cap \pi^{-1}(z) = (\Theta^n)^{-1}(\X_0 \cap \pi^{-1}(\phi^{-n}(z)))$.
Since $\X_0$ has closed fibers, it follows readily that the same is true of $\X_{\infty}$.
Since $\Phi$ is a right-inverse of $\Theta$, we have $\Phi^{-1}(\X_{\infty}) = \X_{\infty} \cap \Phi^{-1}(\X_0)$.
Forward-invariance of $\X_0$ yields $\X_{\infty} \subseteq \X_0 \subseteq \Phi^{-1}(\X_0)$.
Hence, $\Phi^{-1}(\X_{\infty}) = \X_{\infty}$.
This also yields $\X_{\infty} = \Theta(\Phi(\X_{\infty})) \subseteq \Theta(\X_{\infty})$.
It is clear that $\Theta(\X_{\infty}) \subseteq \X_{\infty} = \X_0 \cap \Theta^{-1}(\X_{\infty})$.
\end{proof}

Since $\X_{\infty}$ may not be closed, we cannot, in general, apply \cref{prop_closed_upper_hemi} to $\Sc_{\Phi} \cap \X_{\infty}$.\footnote{
It is tempting to consider the dynamical system $\Phi' = \Phi|_{\X_{\infty}}$ on the state space $\X' = \X_{\infty}$.
Then, $\Sc_{\Phi'} = \Sc_{\Phi} \cap \X_{\infty}$ is closed in the topology of $\X'$.
If $\pi$ is proper, then $\pi' = \pi|_{\X'} \colon \X' \rightarrow \Zc$ has compact fibers.
But, in general, $\pi'$ is not closed.
}
However, if $\Phi$ has the USP, then we can show that all solutions belong to $\X_{\infty}$ and, hence, $\Sc_{\Phi} \cap \X_{\infty} = \Sc_{\Phi}$ is continuous.

\begin{proposition}
\label{prop_USP_struct_sol}
	Suppose the restriction of $\pi$ to $\X_0$ has compact fibers.
Then, $\Sc_{\Phi} \cap \X_{\infty}$ has non-empty fibers.
In particular, if $\Phi$ has the USP, then $\Sc_{\Phi} \subseteq \X_{\infty}$.
\end{proposition}

\begin{proof}
	Since the fibers of $\X_0$ are non-empty, there exists a left-inverse $\rho \colon \X \rightarrow \X_0$ of the inclusion $\X_0 \hookrightarrow \X$ that satisfies $\pi \circ \rho = \pi$.
Let $z \in \Zc$ and $x \in \pi^{-1}(z)$.
Consider $x_n = \Phi^n \circ \rho \circ \Theta^n(x)$, $n \in \N_0$, which satisfy $\pi(x_n) = z$ since $\Phi$ and $\Theta$ are bundle maps over $\phi$ and $\phi^{-1}$, respectively.
Since $\X_0$ is forward-invariant, $(x_n)_{n \in \N_0} \subseteq \X_0$.
By compactness of $\X_0 \cap \pi^{-1}(z)$, the sequence $(x_n)_{n \in \N_0}$ has a convergent subnet with some limit $x' \in \pi^{-1}(z)$.
Since $\X_0$ has closed fibers, $(\Theta^k)^{-1}(\X_0) \cap \pi^{-1}(z) = (\Theta^k)^{-1}(\X_0 \cap \pi^{-1}(\phi^{-k}(z)))$ is closed.
Since $\X_0$ is forward-invariant and $\Theta$ is a left-inverse of $\Phi$, we have $x_n \in (\Theta^k)^{-1}(\X_0)$ for all $k,n \in \N_0$ with $k \leq n$.
Hence, we must have $x' \in (\Theta^k)^{-1}(\X_0) \cap \pi^{-1}(z)$ for all $k \in \N_0$, that is, $x' \in \X_{\infty} \cap \pi^{-1}(z)$.
\end{proof}

Recall that the {\bfi omega limit set} of a point $x \in \X$ under $\Phi$ is $\omega_{\Phi}(x) = \bigcap_{n \in \N} \overline{ \bigcup_{m \geq n} \{\Phi^m(x)\} }$.
It is easy to see that $\omega_{\Phi}(x) \subseteq \Sc_{\Phi} \cap \pi^{-1}( \omega_{\phi}( \pi(x) ) )$.
We recall that a topological space is called first-countable if every point admits a countable neighborhood basis.
If $\X$ is first-countable, then
\begin{equation*}
	\omega_{\Phi}(x)
	= \{ x' \in \X \colon \exists\, (n_k)_{k \in \N} \subseteq \N \text{ such that } n_k \rightarrow \infty \text{ and } \Phi^{n_k}(x) \rightarrow x' \}.
\end{equation*}

\begin{proposition}
\label{prop_omega_limit_struc_sol}
	Suppose $\X$ is first-countable.
Let $x \in \X_{\infty}$ such that $\pi(x)$ is eventually periodic under $\phi$.
Then, $\omega_{\Phi}(x) \subseteq \Sc_{\Phi} \cap \X_{\infty}$.
\end{proposition}

\begin{proof}
	Let $x' \in \omega_{\Phi}(x)$, and take an unbounded sequence $(n_k)_{k \in \N} \subseteq \N$ with $x_k := \Phi^{n_k}(x) \rightarrow x'$.
Since $\pi(x)$ is eventually periodic under $\phi$, there exists a subsequence $(x_{k_l})_l$ of $(x_k)_k$ such that $\pi(x_{k_l}) = \phi^{n_{k_l}}(\pi(x))$ is constant in $l$.
We know from \cref{lem_X_inf_closed_fibres} that $\X_{\infty}$ has closed fibers.
Thus, $x' \in \X_{\infty}$.
\end{proof}

\subsection{Stochastic dynamical systems}
\label{sec_abstract_stochastic}

Now, suppose $\X$ and $\Zc$ are completely regular and $\Phi$ admits a continuous left-inverse $\Theta \colon \X \rightarrow \X$.
Let $\Pc_{\X} \subseteq P(\X)$ and $\Pc_{\Zc} \subseteq P(\Zc)$ satisfy $\pi_*(\Pc_{\X}) \subseteq \Pc_{\Zc}$ and $(\Theta_*)^{-1}(\Pc_{\X}) = \Pc_{\X}$.
Then, consider the push-forward $\Phi_* \colon \Pc_{\X} \rightarrow \Pc_{\X}$, which has the left-inverse $\Theta_* \colon \Pc_{\X} \rightarrow \Pc_{\X}$.
We remark that these maps are well-defined since the push-forward of a Radon measure under a continuous map is itself a Radon measure.
The dynamical system $\Phi_*$ is `stochastic' in the sense that it encodes the evolution of the law of a random variable under iterations of $\Phi$.
However, as an abstract dynamical system, $\Phi_*$ is a deterministic function on the space of Radon probability measures.
In particular, we can study its solution set, which we know equals its global attractor due to the existence of a left-inverse;
\begin{equation*}
	\Sc_{\Phi_*}
	= \{ \mu \in \Pc_{\X} \colon \Phi_*^n\Theta_*^n\mu = \mu \text{ for all } n \in \N \}
	= \bigcap_{n \in \N} \Phi_*^n(\Pc_{\X}).
\end{equation*}
We refer to elements in $\Sc_{\Phi_*}$ as {\bfi stochastic solutions} of the dynamical system $\Phi$.
The following proposition justifies this denomination.
Indeed, the statement of \cref{prop_abstr_support_sol} is equivalent to saying that an element $\mu \in \Pc_{\X}$ is a solution of the dynamical system $\Phi_*$ if and only if some (in fact, every) $\X$-valued random variable $X$ with law $\mu$ satisfies $\Phi^n(\Theta^n(X)) = X$ almost surely for all $n \in \N$.
Recall that the support of a measure is the set of all points of which every open neighborhood has strictly positive measure.

\begin{proposition}
\label{prop_abstr_support_sol}
	An element $\mu \in \Pc_{\X}$ belongs to $\Sc_{\Phi_*}$ if and only if the support of $\mu$ is contained in $\Sc_{\Phi}$.
\end{proposition}

\begin{proof}
	Suppose the support of $\mu$ is contained in $\Sc_{\Phi}$.
Note that $(\Phi^n \circ \Theta^n)^{-1}(A) \cap \Sc_{\Phi} = A \cap \Sc_{\Phi}$ for any $n \in \N$ and any measurable $A \subseteq \X$.
Thus,
\begin{equation*}
	\Phi^n_*\Theta^n_*\mu(A)
	= \mu( (\Phi^n \circ \Theta^n)^{-1}(A) )
	= \mu ((\Phi^n \circ \Theta^n)^{-1}(A) \cap \Sc_{\Phi} )
	= \mu( A \cap \Sc_{\Phi})
	= \mu(A),
\end{equation*}
which shows that $\mu \in \Sc_{\Phi_*}$.
Conversely, suppose $\mu \in \Sc_{\Phi_*}$, and let $x \in \X \backslash \Sc_{\Phi}$.
Then, there exists some $n \in \N$ with $x \notin \Phi^n(\X)$.
Since $\Theta$ is a continuous left-inverse of $\Phi$, the latter is closed; see \cref{lem_proper}.
Thus, $\Phi^n(\X)$ is a closed set.
Since $\X$ is completely regular, there exists an open neighborhood $V \subseteq \X$ of $x$ with $V \cap \Phi^n(\X) = \emptyset$.
Equivalently, $(\Phi^n)^{-1}(V) = \emptyset$.
Now, for any $\mu \in \Sc_{\Phi_*}$, we have $\mu(V) = \Phi^n_*\Theta^n_*\mu(V) = \Theta^n_*\mu(\emptyset) = 0$.
Thus, $x$ does not belong to the support of $\mu$.
\end{proof}

\begin{remark}
	Suppose $\chi \colon \Z \rightarrow \X$ is a solution of $\Phi$.
Let $(w_t)_{t \in \Z} \subseteq (0,1)$ satisfy $\sum_{t \in \Z} w_t = 1$.
Then, the measure $\mu = \sum_{t \in \Z} w_t \delta_{\chi(t)}$ belongs to $\Sc_{\Phi_*}$.
In particular, there are elements in $\Sc_{\Phi_*}$ whose support belongs to the complement of the chain-recurrent set of $\Phi$ (if that complement is non-empty).
\end{remark}

It is straight-forward to see that stochastic solutions can be obtained from pushing forward measures in $\Pc_{\Zc}$ under a continuous map into $\Sc_{\Phi}$.
More precisely, if $S \colon \Zc \rightarrow \Sc_{\Phi}$ is continuous, then $S_*(\Pc_{\Zc}) \cap \Pc_{\X} \subseteq \Sc_{\Phi_*}$.
In \cref{prop_abstr_stoch_USP}, recall that the set-valued fiber-map $\Sc_{\Phi}$ becomes a map $\Zc \rightarrow \X$ under the USP, and that continuity of $\Sc_{\Phi}$ is guaranteed by \cref{prop_closed_upper_hemi} if $\pi$ is proper.

\begin{proposition}
\label{prop_abstr_stoch_USP}
	Suppose $\Phi$ has the USP and $\Sc_{\Phi}$ is continuous.
Then, $(\Sc_{\Phi})_*(\Pc_{\Zc}) \cap \Pc_{\X} = \Sc_{\Phi_*}$.
In particular, if $(\Sc_{\Phi})_*(\Pc_{\Zc}) \subseteq \Pc_{\X}$, then $\Phi_*$ has the USP.
\end{proposition}

\begin{proof}
	We observed above that $(\Sc_{\Phi})_*(\Pc_{\Zc}) \cap \Pc_{\X} \subseteq \Sc_{\Phi_*}$.
For the reverse inclusion, let $A \subseteq \X$ be measurable.
By the USP, $A \cap \Sc_{\Phi} = (\Sc_{\Phi} \circ \pi)^{-1}(A) \cap \Sc_{\Phi}$ (here, $\Sc_{\Phi}$ denotes both the subset of $\X$ and its fiber-map).
This and \cref{prop_abstr_support_sol} yield for any $\mu \in \Sc_{\Phi_*}$
\begin{equation*}
	\mu(A)
	= \mu(A \cap \Sc_{\Phi})
	= \mu((\Sc_{\Phi} \circ \pi)^{-1}(A) \cap \Sc_{\Phi})
	= \mu((\Sc_{\Phi} \circ \pi)^{-1}(A))
	= (\Sc_{\Phi})_*\pi_*\mu(A),
\end{equation*}
which shows that $\Sc_{\Phi_*} \subseteq (\Sc_{\Phi})_*(\Pc_{\Zc})$.
\end{proof}

As noted in \cref{ex_weak_top}, complete regularity of the underlying spaces ensures that push-forward maps induced by continuous maps become continuous with respect to the weak topologies and that the weak topology on the space of Radon probability measures is Hausdorff (as is any finer topology).
Thus, if $P(\X)$ and $P(\Zc)$ are endowed with the weak topologies, then $\Phi_*$ is a continuous bundle map over $\phi_*$ with continuous left-inverse $\Theta_*$.
As such, the results that have been developed in \cref{subsec_dep_sol_base} can also be applied to the dynamical system $\Phi_*$.
One may also work with topologies finer than the weak topologies such as Wasserstein topologies if continuity of the various push-forward maps continues to hold.

\begin{remark}
	Stochastic solutions as defined in this paper are the laws of random variables that are deterministic solutions almost surely.
An extensive branch of literature has studied random dynamical systems without passing to the level of the laws.
For such systems, the notion of an attractor involves a random set, that is, a family of sets parametrized by the underlying probability space of the source of randomness \cite{Arnold1998, CrauelFlandoli1994}.
The approach with such a random system can reveal more detailed behavior of the dynamics since it can study typical random trajectories.
However, it generally requires stronger assumptions (such as compactness of the state space) and quickly runs into technical obstacles, which we avoid with the push-forward dynamical system.
For our purposes, the information that we can infer on the level of the laws is sufficient.
\end{remark}

\section{Proofs of results on state-space systems}
\label{sec_proofs}

We return to the notation from \cref{sec_det_SSS,sec_stoch_SSS} and prove all results stated in those sections, using the tools developed in \cref{sec_abstract}.

\subsection{Proofs for deterministic solutions}
\label{sec_det_SSS_proofs}

\begin{proof}[\Pf{prop_det_sol_attractor}]
	First, we show by induction that $\varphi^n \circ T^n$ is the identity on $\tau^{-1}(\DetSol)$ for all $n \in \N$.
For any $(\Seq{x},\Seq{u}) \in \XS^- \times \US$, we have $\varphi \circ T(\Seq{x},\Seq{u}) = (\Seq{x},\Seq{u})$ if and only if $f(\seq{x}{-1},\seq{u}{0}) = \seq{x}{0}$, which is indeed the case for any $(\Seq{x},\Seq{u}) \in \tau^{-1}(\DetSol)$.
If $n \geq 2$ and $(\Seq{x},\Seq{u}) \in \tau^{-1}(\DetSol)$, then
\begin{equation*}
	\varphi^n \circ T^n(\Seq{x},\Seq{u})
	= \varphi \circ \varphi^{n-1} \circ T^{n-1}(T(\Seq{x},\Seq{u}))
	= \varphi(T(\Seq{x},\Seq{u}))
	= (\Seq{x},\Seq{u}),
\end{equation*}
where we applied the induction hypothesis to $T(\Seq{x},\Seq{u})$, which is possible since $T(\DetSol) \subseteq \DetSol$, and then the induction start to $(\Seq{x},\Seq{u})$.
Thus, $\tau^{-1}(\DetSol) \subseteq \Sc_{\varphi}$.
Since $\tau(\US) = \US^-$, we also have $\DetSol = \tau(\tau^{-1}(\DetSol)) \subseteq \tau(\Sc_{\varphi})$.
It remains to be shown that $\tau(\Sc_{\varphi}) \subseteq \DetSol$, which would also imply $\Sc_{\varphi} \subseteq \tau^{-1}(\tau(\Sc_{\varphi})) \subseteq \tau^{-1}(\DetSol)$.
Denote $p_0 \colon \XS^- \times \US \rightarrow \Xc$, $(\Seq{x},\Seq{u}) \mapsto \seq{x}{0}$, and let $(\Seq{x},\Seq{u}) \in \Sc_{\varphi}$.
We need to show that $f(\seq{x}{t-1},\seq{u}{t}) = \seq{x}{t}$ for all $t \in \Z_-$.
For $t=0$, this is immediate from applying $p_0$ on both sides of the equality $\varphi \circ T(\Seq{x},\Seq{u}) = (\Seq{x},\Seq{u})$.
For $t \leq -1$, note that $\Sc_{\varphi}$ is forward-invariant under $T$ and, hence, $T^{-t-1}(\Seq{x},\Seq{u}) \in \Sc_{\varphi}$.
Then,
\begin{equation*}
	f(\seq{x}{t-1},\seq{u}{t})
	= p_0 \circ \varphi \circ T(T^{-t-1}(\Seq{x},\Seq{u}))
	= p_0(T^{-t-1}(\Seq{x},\Seq{u}))
	= \seq{x}{t},
\end{equation*}
which concludes the proof.
\end{proof}

\begin{proof}[\Pf{thrm_det_ESP_FMP}]
	Since the topology on $\XS^- \times \US^-$ is at least as fine as the product topology, it is clear that $\DetSol$ is closed.
Since $\XS^-$ is compact, the projection $\pi_{\US^-} \colon \XS^- \times \US^- \rightarrow \US^-$ is proper.
By \cref{prop_closed_upper_hemi}, $\DetSol$ is upper hemi-continuous.
As the composition of an upper hemi-continuous map and a continuous map, $\DetOut = H(\DetSol)$ is also upper hemi-continuous.
If the ESP holds (in the outputs), then $\DetSol$ ($\DetOut$) is singleton-valued and upper hemi-continuity agrees with usual continuity.
\end{proof}

\begin{proof}[\Pf{prop_det_FMP_mESP_out}]
	Let us verify that \cref{prop_const_card} is applicable with $\X = \X_{\infty} = \XS^- \times \US$, $\Zc = \US$, $\Zc_- = \US^-$, $\Y = \YS^- \times \US^-$, $\pi = \pi_{\US}$, $\phi = \varphi$, $\tau = \tau$, and $\eta = H \circ \tau$.
With these choices,
\begin{equation*}
	\eta(\Sc_{\phi} \cap \X_{\infty} \cap (\tau \circ \pi)^{-1}(\Seq{u}))
	= H(\tau(\Sc_{\varphi})) \cap \pi_{\US^-}^{-1}(\Seq{u})
	= \DetOut(\Seq{u})
\end{equation*}
for any $\Seq{u} \in \US^-$.
Hence, if we find that the three hypotheses of \cref{prop_const_card} are satisfied, then the conclusion of \cref{prop_const_card} is exactly the desired statement.
Hypothesis (i):
Suppose $(\Seq{x},\Seq{u})$ and $(\Seq{x}',\Seq{u}')$ are elements of $\Sc_{\varphi}$ such that $\eta(T(\Seq{x},\Seq{u})) = \eta(T(\Seq{x}',\Seq{u}'))$ and $\pi(\Seq{x},\Seq{u}) = \pi(\Seq{x}',\Seq{u}')$.
Then, $\Seq{u} = \Seq{u}'$ and $h(\seq{x}{-1}) = h(\seq{x}{-1}')$.
By \cref{rem_weakly_disting}, we have $h(\seq{x}{0}) = h(\seq{x}{0}')$ and, hence, $\eta(\Seq{x},\Seq{u}) = \eta(\Seq{x}',\Seq{u}')$, which verifies hypothesis (i).
Hypothesis (ii):
It is evident from \cref{prop_det_sol_attractor} that $\tau(\Sc_{\varphi} \cap \pi^{-1}(\Seq{u}))$ depends on $\Seq{u}$ only through $\tau(\Seq{u})$.
Hypothesis (iii):
This is precisely stipulated in \cref{assmpt_det_SSS}.(iv) (and recall that $\tau$ maps $\US$ surjectively onto $\US^-$).
\end{proof}

\cref{prop_det_FMP_mESP_sol} corresponds to the special case of \cref{prop_det_FMP_mESP_out} in which the readout is the identity.

\begin{proof}[\Pf{thrm_det_generic_FMP}]
	As in the proof of \cref{thrm_det_ESP_FMP}, we can apply \cref{prop_closed_upper_hemi} to the set-valued maps $\DetSol$ and $\DetOut =  H(\DetSol)$.
Thus, we obtain that $D(\DetSol)$ and $D(\DetOut)$ contain a residual set.
The final statement follows from \cref{prop_det_FMP_mESP_sol,prop_det_FMP_mESP_out}.
\end{proof}

\subsection{Proofs for general stochastic solutions}
\label{sec_stoch_SSS_proofs}

From now on, let $j^+ \colon \XS^- \times \US^- \rightarrow \XS^- \times \US$ denote a continuous right-inverse of the truncation, whose existence we stipulated in \cref{assmpt_det_SSS}.(ii).

\begin{proof}[\Pf{prop_stoch_sol}]
	It is clear that (iii) $\Rightarrow$ (ii) $\Rightarrow$ (i).
If (i) holds, then $\mu(\DetSol) = 1$.
This is sufficient to deduce (iii) because $\DetSol$ is closed and $\XS^- \times \US^-$ is completely regular.

Next, we know from \cref{prop_abstr_support_sol} that an element $\mu \in \Pstate$ belongs to $\Sc_{\varphi_*}$ if and only if the support of $\mu$ is contained in $\Sc_{\varphi}$.
The support of $\tau_*\mu$ is contained in the closure of $\tau(\mathrm{support}(\mu))$.
In particular, the support of an element in $\tau_*(\Sc_{\varphi_*})$ is contained in $\tau(\Sc_{\varphi}) = \DetSol$.
This and the equivalence between (i) and (iii) show $\tau_*(\Sc_{\varphi_*}) \subseteq \StochSol$.
Conversely, suppose $\mu \in \StochSol$.
Then, the support of $\mu$ is contained in $\DetSol$, and the support of $j^+\mu$ is contained in the closure of $j^+(\DetSol)$.
Since $j^+(\DetSol)$ is a subset of $\tau^{-1}(\DetSol) = \Sc_{\varphi}$, we conclude that $\mu = \tau_*j^+_*\mu \in \tau_*(\Sc_{\varphi_*})$.
\end{proof}

\begin{proof}[\Pf{prop_ESP_stoch_ESP}]
	We know from \cref{prop_abstr_stoch_USP} that $(\Sc_{\varphi})_*(\Pin) \cap \Pstate = \Sc_{\varphi_*}$.
Observe that $\tau_*((\Sc_{\varphi})_*(\Pin) \cap \Pstate) = (\tau \circ \Sc_{\varphi})_*(\Pin) \cap \Pstate^-$.
Furthermore, the equality $\tau \circ \Sc_{\varphi} = \DetSol \circ \tau$ implies $(\tau \circ \Sc_{\varphi})_*(\Pin) = \DetSol_*(\Pin^-)$.
Thus,
\begin{equation*}
	\StochSol
	= \tau_*(\Sc_{\varphi_*})
	= \tau_*((\Sc_{\varphi})_*(\Pin) \cap \Pstate)
	= (\tau \circ \Sc_{\varphi})_*(\Pin) \cap \Pstate^-
	= \DetSol_*(\Pin^-) \cap \Pstate^-.
\end{equation*}
\end{proof}

\begin{lemma}
\label{lem_stoch_sol_closed}
	The set $\StochSol$ is closed (in the subspace topology of $\Pstate^-$).
\end{lemma}

\begin{proof}
	Note that $\varphi^n \circ T^n \circ j^+ \circ \tau = j^+ \circ \tau \circ \varphi^n \circ T^n$ for all $n \in \N$.
This observation implies $j^+_*(\StochSol) \subseteq \Sc_{\varphi_*}$.
Now, suppose $\mu \in \Pstate^-$ is a cluster point of a net $(\mu_i)_{i \in I} \subseteq \StochSol$.
Then, $j^+_*\mu \in \Pstate$ is a cluster point of $(j^+_*\mu_i)_{i \in I} \subseteq \Sc_{\varphi_*}$ in the weak topology.
Since $\Sc_{\varphi_*}$ is weakly closed (as a subset of $\Pstate$), it contains $j^+_*\mu$.
In particular, $\mu = \tau_*j^+_*\mu \in \tau_*(\Sc_{\varphi_*}) = \StochSol$.
\end{proof}

\begin{proof}[\Pf{thrm_stoch_ESP_FMP}]
	Since $\StochSol$ is a closed subset of $\Pstate^-$ by \cref{lem_stoch_sol_closed}, its fiber map is upper hemi-continuous by \cref{prop_closed_upper_hemi}.
The same goes for $\StochOut = H_*(\StochSol)$, being the composition of an upper hemi-continuous map and a continuous map.
If the stochastic ESP holds, then these are singleton-valued maps and upper hemi-continuity agrees with usual continuity.
Without the stochastic ESP but if $\Pin^-$ and $\Pstate^-$ are Polish, the second part of \cref{prop_closed_upper_hemi} applies.
\end{proof}

To prove the remaining results, we need to have a closer look at causal solutions.
Before we continue with stochastic state-space systems, we briefly recap the definition and most important properties of conditional independence of sigma-algebras, which we used to define causal measures.

\subsection{Interlude: Conditional independence}
\label{sec_cond_ind}

Consider a probability space $(\Omega,\Ac,\Pb)$ and sub-sigma-algebras $\Sigma_1,\Sigma_2,\Sigma_3 \subseteq \Ac$.
We remind the reader that $\Sigma_1$ and $\Sigma_3$ are said to be conditionally independent given $\Sigma_2$ (written $\Sigma_1 \CondInd{\Sigma_2} \Sigma_3$) with respect to $\Pb$ if
\begin{equation}
\label{eq_CondInd_def}
	\ExpC{\Pb}{\beta_1 \beta_3}{\Sigma_2}
	= \ExpC{\Pb}{\beta_1}{\Sigma_2} \ExpC{\Pb}{\beta_3}{\Sigma_2} \quad \Pb\text{-a.s.}
\end{equation}
for any two bounded functions $\beta_1,\beta_3 \colon \Omega \rightarrow \R$ that are $\Sigma_1$-measurable and $\Sigma_3$-measurable, respectively.
An equivalent characterization is
\begin{equation}
\label{eq_CondInd_alt_def}
	\ExpC{\Pb}{\beta_3}{\Sigma_1 \vee \Sigma_2}
	= \ExpC{\Pb}{\beta_3}{\Sigma_2} \quad \Pb\text{-a.s.}
\end{equation}
for any bounded $\Sigma_3$-measurable function $\beta_3 \colon \Omega \rightarrow \R$.
Here, $\Sigma_1 \vee \Sigma_2$ denotes the sigma-algebra generated by $\Sigma_1 \cup \Sigma_2$.
Various properties of conditional independence can be found, for example, in \cite[Chapter 8]{Kallenberg2021} and \cite[Chapter 7.3]{ChowTeicher1997}.
The following facts will be useful.

\begin{remark}
\label{rem_CondInd_transform}
	Consider sub-sigma-algebras $\Sigma_1,\Sigma_1',\Sigma_2,\Sigma_2',\Sigma_3 \subseteq \Ac$.
\begin{enumerate}[\upshape (i)]\itemsep=0em
\item
If $\Sigma_1' \subseteq \Sigma_1 \vee \Sigma_2$ and if $\Sigma_1 \CondInd{\Sigma_2} \Sigma_3$ wrt $\Pb$, then $\Sigma_1' \CondInd{\Sigma_2} \Sigma_3$ wrt $\Pb$.
This is immediate from the equivalent characterization \eqref{eq_CondInd_alt_def} and the tower property of conditional expectations.

\item
If $\Sigma_3' \subseteq \Sigma_3$ and if $\Sigma_1 \CondInd{\Sigma_2} \Sigma_3$ wrt $\Pb$, then $\Sigma_1 \CondInd{\Sigma_2 \vee \Sigma_3'} \Sigma_3$ wrt $\Pb$.
Indeed, it is immediate from the definition \eqref{eq_CondInd_def} that $\Sigma_3 \CondInd{\Sigma_2} \Sigma_1$ and $\Sigma_3' \CondInd{\Sigma_2} \Sigma_1$.
Then, applying the equivalent characterization \eqref{eq_CondInd_alt_def} to each immediately yields $\Sigma_1 \CondInd{\Sigma_2 \vee \Sigma_3'} \Sigma_3$.
\end{enumerate}
Furthermore, let $(\Omega',\Ac',\Pb')$ be another probability space and $\Gc \colon (\Omega',\Ac') \rightarrow (\Omega,\Ac)$ be measurable.
\begin{enumerate}[\upshape (i)]\itemsep=0em
\setcounter{enumi}{2}
\item
It holds $\Sigma_1 \CondInd{\Sigma_2} \Sigma_3$ wrt $\Gc_*\Pb'$ if and only if $\Gc^{-1}(\Sigma_1) \CondInd{\Gc^{-1}(\Sigma_2)} \Gc^{-1}(\Sigma_3)$ wrt $\Pb'$.
This follows from how conditional expectations behave under push-forward transformations, namely
\begin{equation}
\label{eq_CondInd_transform}
	\ExpC{\Gc_*\Pb'}{\beta}{\Sigma} \circ \Gc
	= \ExpC{\Pb'}{\beta \circ \Gc}{\Gc^{-1}(\Sigma)} \quad \Pb'\text{-a.s.}
\end{equation}
for any bounded $\Ac$-measurable function $\beta \colon \Omega \rightarrow \R$ and any sub-sigma-algebra $\Sigma \subseteq \Ac$.
\end{enumerate}
\end{remark}

The following lemma is a version of \cite[Proposition 8.20]{Kallenberg2021}.
For the sake of completeness, we provide a proof in \cref{app_sec_technical}.

\begin{lemma}[Adaptation of Proposition 8.20 in \cite{Kallenberg2021}]
\label{lem_CondInd_repr}
	Let $\Lambda$ denote the Lebesgue measure on $[0,1]$.
For each $i = 1,2,3$, suppose $\Sigma_i$ is the sigma-algebra generated by a measurable map $V_i \colon \Omega \rightarrow \Vc_i$ into a measurable space $\Vc_i$.
Suppose $g \colon \Vc_2 \times [0,1] \rightarrow \Vc_3$ is a measurable map that satisfies $(\pi_{\Vc_2} \times g)_*( (V_2)_*\Pb \otimes \Lambda ) = (V_2 \times V_3)_*\Pb$.
Let $G \colon \Vc_1 \times \Vc_2 \times [0,1] \rightarrow \Vc_1 \times \Vc_2 \times \Vc_3$ denote the map $G(v_1,v_2,\lambda) = (v_1,v_2,g(v_2,\lambda))$.
Then, the following are equivalent.
\begin{enumerate}[\upshape (i)]\itemsep=0em
\item
$\Sigma_1$ and $\Sigma_3$ are conditionally independent given $\Sigma_2$ wrt $\Pb$.

\item
The map $G$ satisfies $G_*((V_1 \times V_2)_*\Pb \otimes \Lambda) = (V_1 \times V_2 \times V_3)_*\Pb$.
\end{enumerate}
\end{lemma}

\begin{remark}
\label{rem_CondInd_repr_add}
	In \cref{lem_CondInd_repr}, such a measurable map $g$ exists as soon as $\Vc_3$ is a standard Borel space; see, for example, \cite[Theorem 8.17]{Kallenberg2021} and its proof.
\end{remark}

The following is an immediate consequence of \cref{lem_CondInd_repr,rem_CondInd_repr_add}.

\begin{lemma}
\label{lem_CondInd_injective}
	For each $i = 1,2,3$, suppose $\Sigma_i$ is the sigma-algebra generated by a measurable map $V_i \colon \Omega \rightarrow \Vc_i$ into a measurable space $\Vc_i$, with $\Vc_3$ a standard Borel space.
Furthermore, suppose the map $V_1 \times V_2 \times V_3 \colon \Omega \rightarrow \Vc_1 \times \Vc_2 \times \Vc_3$ admits a measurable left-inverse.
Let $\Pc$ denote the set of all probability measures $\Pb'$ on $(\Omega,\Ac)$ with respect to which $\Sigma_1 \CondInd{\Sigma_2} \Sigma_3$ and which satisfy $(V_2 \times V_3)_*\Pb' = (V_2 \times V_3)_*\Pb$.
Then, the map $\Pb' \mapsto (V_1 \times V_2)_*\Pb'$ is injective on $\Pc$.
\end{lemma}

In general, for fixed sigma-algebras, the set of measures with respect to which conditional independence holds is closed neither in the weak topology nor in a Wasserstein topology \cite{Lauritzen1996}.
Only in the topology of total variation does this set become closed \cite{CheEckstein2025, Lauritzen2024}.
However, this topology is too strong for our purposes.
The issue in the weak and Wasserstein topologies is that information in $\Sigma_2$ could degenerate in a limit, and $\Sigma_1$ and $\Sigma_3$ become dependent once the information from $\Sigma_2$ is lost.
Conversely, if we stipulate that only information in $\Sigma_1$ varies, then conditional independence remains preserved under weak limits.
This is the message of the next lemma.

\begin{lemma}
\label{lem_push_forw_cont}
	For each $i = 1,2,3$, suppose $\Sigma_i$ is the sigma-algebra generated by a continuous map $V_i \colon \Omega \rightarrow \Vc_i$ into a Polish space $\Vc_i$.
Suppose there exists a sequence $(\Pb_n)_{n \in \N}$ of probability measures on $\Omega$ with the following properties.
\begin{enumerate}[\upshape (i)]\itemsep=0em
\item
$\Sigma_1 \CondInd{\Sigma_2} \Sigma_3$ wrt $\Pb_n$ for all $n \in \N$.

\item
$(V_2 \times V_3)_*\Pb_n = (V_2 \times V_3)_*\Pb$ for all $n \in \N$.

\item
$(V_1 \times V_2 \times V_3)_*\Pb$ is an accumulation point of $((V_1 \times V_2 \times V_3)_*\Pb_n)_{n \in \N}$ in the weak topology.
\end{enumerate}
Then, $\Sigma_1 \CondInd{\Sigma_2} \Sigma_3$ wrt $\Pb$.
\end{lemma}

\begin{proof}
	By hypotheses (i) and (ii), we know from \cref{lem_CondInd_repr,rem_CondInd_repr_add} that $G_*((V_1 \times V_2)_*\Pb_n \otimes \Lambda) = (V_1 \times V_2 \times V_3)_*\Pb_n$ for all $n \in \N$.
We need to argue that this equality continues to hold with $\Pb$ in place of $\Pb_n$ despite the possibility that $g$ may not be continuous.
But this is indeed not a problem since $\Vc_1$ and $\Vc_2$ are Polish, $G$ is continuous in its $\Vc_1$-coordinate, and $(\pi_{\Vc_2} \times \pi_{[0,1]})_*((V_1 \times V_2)_*\Pb_n \otimes \Lambda) = (V_2)_*\Pb \otimes \Lambda$ does not depend on $n$; see \cite[Corollary 2.9]{JacodMemin1981} and also \cite[Example 2]{PratelliRigo2023}.
\end{proof}

\subsection{Proofs for causal stochastic solutions}

To prove results about causal solutions, we first translate the notion of causality to measures on $P(\XS^- \times \US)$.
Let $\Sigma_{\Xc}^-$ and $\Sigma_{\Uc}$ be the sigma-algebras on $\XS^- \times \US$ generated by the projections onto $\XS^-$ and $\US$, respectively, and let $\Sigma_{\Uc}^-$ be the sigma-algebra on $\XS^- \times \US$ generated by the composition $\XS^- \times \US \rightarrow \US^-$ of the projection and the truncation.

\begin{definition}
	We call a measure $\mu \in P(\XS^- \times \US)$ {\bfi causal at time} $t \in \Z_-$ if $\Sigma_{\Xc}^-$ and $\Sigma_{\Uc}$ are conditionally independent given $\Sigma_{\Uc}^-$ with respect to $T^{-t}_*\mu$.
We call $\mu$ {\bfi causal} if it is causal at all times $t \in \Z_-$.
\end{definition}

It is immediate from the definition that if $\mu \in P(\XS^- \times \US)$ is causal, then so is $T_*\mu$.
In particular, if $\varphi_*\mu$ is causal, then so is $T_*\varphi_*\mu = \mu$.
We will see in \cref{lem_causal_at_zero}.(i) below that the converse also holds.

\begin{remark}
\label{rem_causal_comp}
	Note that $(T^{-t})^{-1}(\Sigma_{\Xc}^-) = \tau^{-1}(\Sigma_{\Xc,t})$ and $(T^{-t})^{-1}(\Sigma_{\Uc}^-) = \tau^{-1}(\Sigma_{\Uc,t})$.
It follows from \cref{rem_CondInd_transform}.(iii) that a measure $\mu \in P(\XS^- \times \US)$ is causal at time $t \in \Z_-$ if and only if  $\tau^{-1}(\Sigma_{\Xc,t}) \CondInd{\tau^{-1}(\Sigma_{\Uc,t})} \Sigma_{\Uc}$ wrt $\mu$.
Using that $\tau^{-1}(\Sigma_{\Uc,0}) \subseteq \Sigma_{\Uc}$ and applying \cref{rem_CondInd_transform}.(iii) once more, we deduce that if $\mu$ is causal, then so is $\tau_*\mu$ in the sense of \cref{def_causal}.
\end{remark}

Let $\X_0 \subseteq \Pstate$ denote the set of all measures that are causal at time zero.
The next lemma asserts that this set satisfies the assumptions that we posed on the set $\X_0$ at the beginning of \cref{sec_struct_sol}.

\begin{lemma}
\label{lem_causal_at_zero}
	The set $\X_0$ has the following properties.
\begin{enumerate}[\upshape (i)]\itemsep=0em
\item
It is forward-invariant under $\varphi_*$.

\item
Its fiber $\X_0 \cap (\pi_{\US})_*^{-1}(\Xi)$ is weakly closed (in the subspace topology of $\Pstate$) for any $\Xi \in P(\US)$.

\item
For any $\Xi \in \Pin$ and $\mu^- \in \Pstate^- \cap (\pi_{\US^-})_*^{-1}(\tau_*\Xi)$, there exists a unique $\mu \in \X_0 \cap (\pi_{\US})_*^{-1}(\Xi)$ with $\tau_*\mu = \mu^-$.
\end{enumerate}
\end{lemma}

\begin{proof}
	(i)
Let $\mu \in \X_0$.
We need to show that $\Sigma_{\Xc}^- \CondInd{\Sigma_{\Uc}^-} \Sigma_{\Uc}$ wrt $\varphi_*\mu$.
By \cref{rem_CondInd_transform}.(iii) and since $\varphi^{-1}(\Sigma_{\Uc}) = \Sigma_{\Uc}$, it suffices to convince ourselves that $\varphi^{-1}(\Sigma_{\Xc}^-) \CondInd{\varphi^{-1}(\Sigma_{\Uc}^-)} \Sigma_{\Uc}$ wrt $\mu$.
Since we work with product topologies (by \cref{assmpt_prod_top}), the sigma-algebra $\varphi^{-1}(\Sigma_{\Uc}^-)$ can be written as $\Sigma_{\Uc}^- \vee \Sigma_{\Uc}^1$, where $\Sigma_{\Uc}^1$ is the sigma-algebra generated by the projection $\XS^- \times \US \rightarrow \Uc$, $(\Seq{x},\Seq{u}) \mapsto \seq{u}{1}$.
Thus, we can apply \cref{rem_CondInd_transform}.(ii) and the premise that $\mu \in \X_0$ to find that $\Sigma_{\Xc}^- \CondInd{\varphi^{-1}(\Sigma_{\Uc}^-)} \Sigma_{\Uc}$ wrt $\mu$.
Then, we apply \cref{rem_CondInd_transform}.(i) with $\varphi^{-1}(\Sigma_{\Xc}^-) \subseteq \Sigma_{\Xc}^- \vee \Sigma_{\Uc}^- \vee \Sigma_{\Uc}^1$ to conclude that $\varphi^{-1}(\Sigma_{\Xc}^-) \CondInd{\varphi^{-1}(\Sigma_{\Uc}^-)} \Sigma_{\Uc}$ wrt $\mu$ as desired.

(ii)
This follows from \cref{lem_push_forw_cont} (since $\XS^-$ and $\US$ are Polish by \cref{assmpt_prod_top}, the weak topology on $\Pstate$ is metrizable and, hence, it suffices to consider sequences instead of nets).

(iii)
Take a measurable map $g \colon \US^- \times [0,1] \rightarrow \US$ that satisfies $(\pi_{\US^-} \times g)_*( \tau_*\Xi \otimes \Lambda ) = (\tau \times \mathrm{id}_{\US})_*\Xi$, whose existence we recalled in \cref{rem_CondInd_repr_add}.
In particular, $\tau \circ g = \pi_{\US^-}$ holds $\tau_*\Xi \otimes \Lambda$-a.s., since $\US^-$ is Polish.
Let $\mu := (\pi_{\XS^-} \times (g \circ \pi_{\US^- \times [0,1]}))_*(\mu^- \otimes \Lambda)$.
Then, $(\pi_{\XS^-} \times (\pi_{\US^-} \circ \tau))_*\mu = \mu^-$ and $((\pi_{\US^-} \circ \tau) \times \pi_{\US})_*\mu = (\tau \times \mathrm{id}_{\US})_*\Xi$.
An application of \cref{lem_CondInd_repr} yields $\Sigma_{\Xc}^- \CondInd{\Sigma_{\Uc}^-} \Sigma_{\Uc}$ wrt $\mu$.
Thus, $\mu \in \X_0$.
Uniqueness follows from \cref{lem_CondInd_injective}.
\end{proof}

As in \cref{sec_struct_sol}, let $\X_{\infty} := \bigcap_{n \in \N_0} (T_*^n)^{-1}(\X_0)$.
It is straight-forward to see that $\X_{\infty}$ is exactly the set of causal measures in $\Pstate$.
Recall that $\tau_*(\Sc_{\varphi_*}) \cap \Pcausal = \StochCausalSol$, by definition.

\begin{lemma}
\label{lem_causal_trunc}
	We have $\tau_*(\Sc_{\varphi_*} \cap \X_{\infty}) = \StochCausalSol$.
\end{lemma}

\begin{proof}
	The inclusion $\tau_*(\Sc_{\varphi_*} \cap \X_{\infty}) \subseteq \StochCausalSol$ is immediate from \cref{rem_causal_comp}.
Conversely, suppose we are given an element $\mu \in \StochCausalSol$.
We show that $j^+_*\mu$ belongs to $\Sc_{\varphi_*} \cap \X_{\infty}$.
We had already seen in the proof of \cref{lem_stoch_sol_closed} that $j^+_*(\StochSol) \subseteq \Sc_{\varphi_*}$.
It remains to be shown that $j^+_*\mu$ is causal.
Since $\mu$ is causal and $(j^+)^{-1}(\Sigma_{\Uc}) \subseteq \Sigma_{\Uc,0}$, we know that $\Sigma_{\Xc,t} \CondInd{\Sigma_{\Uc,t}} (j^+)^{-1}(\Sigma_{\Uc})$ wrt $\mu$ for all $t \in \Z_-$.
By \cref{rem_CondInd_transform}.(iii) and the fact that $(j^+)^{-1}(\tau^{-1}(\Sigma_{\Uc,t})) = \Sigma_{\Uc,t}$ and $(j^+)^{-1}(\tau^{-1}(\Sigma_{\Xc,t})) = \Sigma_{\Xc,t}$, we find that $\tau^{-1}(\Sigma_{\Xc,t}) \CondInd{\tau^{-1}(\Sigma_{\Uc,t})} \Sigma_{\Uc}$ wrt $j^+_*\mu$.
This is equivalent to causality of $j^+_*\mu$ as noted in \cref{rem_causal_comp}.
\end{proof}

As our last technical ingredient, we show that measures that are causal at time zero and whose truncation agrees with a causal measure are themselves causal.

\begin{lemma}
\label{lem_CondInd_shift}
	Let $\mu \in \X_{\infty}$ and $\mu' \in \X_0$ satisfy $\tau_*\mu = \tau_*\mu'$.
Then, $\mu' \in \X_{\infty}$.
Furthermore, if $\mu \in \Sc_{\varphi_*}$, then $\mu' \in \Sc_{\varphi_*}$.
\end{lemma}

\begin{proof}
	Fix $t \in \Z_-$.
We need to show that $(T^{-t})_*\mu' \in \X_0$.
Abbreviate $V_1 := T^{-t} \circ \pi_{\XS^-} \circ \tau$, $V_2 := T^{-t} \circ \pi_{\US^-} \circ \tau$, $V := \pi_{\US^-} \circ \tau$, and $V_3 := \pi_{\US}$.
Let $g^- \colon \US^- \times [0,1] \rightarrow \US^-$ and $g^+ \colon \US^- \times [0,1] \rightarrow \US$ be measurable maps such that $(\pi_{\US^-} \times g^-)_*( (V_2)_*\mu' \otimes \Lambda ) = (V_2 \times V)_*\mu'$ and $(\pi_{\US^-} \times g^+)_*( V_*\mu' \otimes \Lambda ) = (V \times V_3)_*\mu'$, respectively, whose existence we recalled in \cref{rem_CondInd_repr_add} ($\US^-$ and $\US$ are Polish by \cref{assmpt_prod_top}).
In particular, $T^{-t} \circ g^- = \pi_{\US^-}$ holds $(V_2)_*\mu' \otimes \Lambda$-a.s., since $\US^-$ is Polish.
Let $\gamma = (\gamma_1,\gamma_2) \colon [0,1] \rightarrow [0,1] \times [0,1]$ be a measurable map such that $\gamma_*\Lambda = \Lambda \otimes \Lambda$.
Then, let $g \colon \US^- \times [0,1] \rightarrow \US$ be given by $g(\Seq{u},\lambda) = g^+(g^-(\Seq{u},\gamma_1(\lambda)),\gamma_2(\lambda))$.
This map satisfies $(\pi_{\US^-} \times g)_*( (V_2)_*\mu' \otimes \Lambda ) = (V_2 \times V_3)_*\mu'$.
By \cref{rem_causal_comp}, we need to show that $\tau^{-1}(\Sigma_{\Xc,t}) \CondInd{\tau^{-1}(\Sigma_{\Uc,t})} \Sigma_{\Uc}$ wrt $\mu'$.
By the same remark, the hypotheses, and the fact that $\tau^{-1}(\Sigma_{\Xc,t}) \subseteq \tau^{-1}(\Sigma_{\Xc,0})$, we know that $\tau^{-1}(\Sigma_{\Xc,t}) \CondInd{\tau^{-1}(\Sigma_{\Uc,t})} \tau^{-1}(\Sigma_{\Uc,0})$ and $\tau^{-1}(\Sigma_{\Xc,t}) \CondInd{\tau^{-1}(\Sigma_{\Uc,0})} \Sigma_{\Uc}$ wrt $\mu'$.
Let $G^- \colon \XS^- \times \US^- \times [0,1] \rightarrow \XS^- \times \US^- \times [0,1]$ and $G^+ \colon \XS^- \times \US^- \times [0,1] \rightarrow \XS^- \times \US^- \times \US$ be the maps $G^-(\Seq{x},\Seq{u},\lambda) = (\Seq{x},g^-(\Seq{u},\gamma_1(\lambda)),\gamma_2(\lambda))$ and $G^+(\Seq{x},\Seq{u},\lambda) = (\Seq{x},\Seq{u},g^+(\Seq{u},\lambda))$, respectively.
By \cref{lem_CondInd_repr}, we have $G^-_*( (V_1 \times V_2)_*\mu' \otimes \Lambda ) = (V_1 \times V)_*\mu' \otimes \Lambda$ and $G^+_*( (V_1 \times V)_*\mu' \otimes \Lambda ) = (V_1 \times V \times V_3)_*\mu'$.
In particular, the map $G' := (\pi_{\XS^-} \times (T^{-t} \circ \pi_{\US^-}) \times \pi_{\US}) \circ G^+ \circ G^-$ satisfies $G'_*( (V_1 \times V_2)_*\mu' \otimes \Lambda ) = (V_1 \times V_2 \times V_3)_*\mu'$.
Let $G \colon \XS^- \times \US^- \times [0,1] \rightarrow \XS^- \times \US^- \times \US$ be the map $G(\Seq{x},\Seq{u},\lambda) = (\Seq{x},\Seq{u},g(\Seq{u},\lambda))$.
Then, $G = G'$ holds $(V_1 \times V_2)_*\mu' \otimes \Lambda$-a.s., since $T^{-t} \circ g^- = \pi_{\US^-}$ holds $(V_2)_*\mu' \otimes \Lambda$-a.s.
Applying \cref{lem_CondInd_repr} once more, we have found that indeed $\tau^{-1}(\Sigma_{\Xc,t}) \CondInd{\tau^{-1}(\Sigma_{\Uc,t})} \Sigma_{\Uc}$ wrt $\mu'$.

Now, suppose $\mu \in \Sc_{\varphi_*}$.
To see that $\mu' \in \Sc_{\varphi_*}$, fix $n \in \N$ and consider $\mu^n := \varphi^n_*T^n_*\mu'$.
By \cref{lem_X_inf_closed_fibres}, $\mu^n \in \X_0$.
Note that $\tau \circ \varphi^n \circ T^n = \tau \circ \varphi^n \circ T^n \circ j^+ \circ \tau$.
Thus,
\begin{equation*}
	\tau_*\mu^n
	= \tau_*\varphi^n_*T^n_*\mu'
	= \tau_*\varphi^n_*T^n_*j^+_*\tau_*\mu'
	= \tau_*\varphi^n_*T^n_*j^+_*\tau_*\mu
	= \tau_*\varphi^n_*T^n_*\mu
	= \tau_*\mu
	= \tau_*\mu'.
\end{equation*}
At the same time, $(\pi_{\US})_*\mu^n = (\pi_{\US})_*\mu'$ since $\pi_{\US} \circ \varphi^n \circ T^n = \pi_{\US}$.
By the uniqueness in \cref{lem_causal_at_zero}.(iii), we must have $\mu^n = \mu'$.
\end{proof}

\begin{proof}[\Pf{prop_stoch_per_sol}]
	By \cref{lem_X_inf_closed_fibres}, the restriction $\phi := \varphi_*|_{\X_{\infty}}$ is an autonomous dynamical system on $\X_{\infty}$ and a bundle map over $\sigma_*$.
Note that $\Sc_{\phi} = \Sc_{\varphi_*} \cap \X_{\infty}$.
We show that the causal stochastic ESP of the state-space system transfers to the USP of $\phi$.
Let $\Xi \in \Pin$.
Then, there is a unique $\mu^- \in \StochCausalSol$ with $(\pi_{\US^-})_*\mu^- = \tau_*\Xi$.
Take an element $\mu \in \Sc_{\varphi_*} \cap \X_{\infty}$ with $\tau_*\mu = \mu^-$.
By \cref{lem_causal_at_zero}.(iii), there is a unique $\mu' \in \X_0$ with $\tau_*\mu' = \mu^-$ and $(\pi_{\US})_*\mu' = \Xi$.
By \cref{lem_CondInd_shift}, $\mu' \in \Sc_{\phi}$.
This concludes the USP of $\phi$.
By \cref{cor_USP_driver_comp}, $\Sc_{\phi} = \Sc_{\phi}^{\mathrm{sync}}$.
Now, suppose $\Xi^- \in \Pin^-$ is periodic under $T_*$ with minimal period $n \in \N$.
Let $\mu^- \in \StochCausalSol$ be the unique causal stochastic solution for the input $\Xi^-$.
By an application of the Kolmogorov extension theorem, there exists some $\Xi \in P(\US)$ that is periodic under $\sigma_*$ with minimal period $n$ and that satisfies $\tau_*\Xi = \Xi^-$; this is where we use the hypothesis $\US = \tau^{-1}(\US^-)$ --- see \cref{lem_Kolmogorov} in the appendix for details.
Let $\mu \in \Sc_{\phi}$ be the unique element with $(\pi_{\US})_*\mu = \Xi$, which we know from above satisfies $\tau_*\mu = \mu^-$.
By \cref{prop_driver_comp_sol}, $\mu \in \SPer{\phi}{\sigma_*}$.
Thus, $\phi^n(\mu) = \mu$, and $n$ is minimal.
In particular, $T^n_*\mu^- = \tau_*T^n_*\mu = \tau_*T^n_*\phi^n(\mu) = \tau_*\mu = \mu^-$.
\end{proof}

\begin{proof}[\Pf{prop_stoch_causal_ESP}]
	We claim that the restriction of $(\pi_{\US})_*$ to $\X_0$ has weakly compact fibers.
Once this is shown, \cref{prop_USP_struct_sol} implies that $\Sc_{\varphi_*} \cap \X_{\infty}$ has non-empty fibers (note that the fibers in \cref{lem_causal_at_zero}.(ii) are weakly closed), and then \cref{lem_causal_trunc} concludes the proof.
Let $\Xi \in \Pin$.
By \cref{lem_causal_at_zero}.(iii), there exists a bijection $\alpha \colon \Pstate^- \cap (\pi_{\US^-})_*^{-1}(\tau_*\Xi) \rightarrow \X_0 \cap (\pi_{\US})_*^{-1}(\Xi)$.
The domain of $\alpha$ is compact by the hypothesis.
As in the proof of \cref{lem_push_forw_cont}, the map $\alpha$ is continuous with respect to the weak topologies.
Thus, the image of $\alpha$ is weakly compact.
\end{proof}

Next, we apply \cref{prop_const_card} to the dynamical system $\varphi_*$ to deduce \cref{prop_stoch_FMP_mESP_out}.
As in the deterministic case, \cref{prop_stoch_FMP_mESP_sol} is recovered by taking the readout to be the identity.

\begin{proof}[\Pf{prop_stoch_FMP_mESP_out}]
	Let us verify that \cref{prop_const_card} is applicable with $\X = \Pstate$, $\X_{\infty} = \X_{\infty}$, $\Zc = \Pin$, $\Zc_- = \Pin^-$, $\Y = \Pout^-$, $\pi = (\pi_{\US})_*$, $\phi = \varphi_*$, $\tau = \tau_*$, and $\eta = (H \circ \tau)_*$.
With these choices, by \cref{lem_causal_trunc},
\begin{equation*}
	\eta(\Sc_{\phi} \cap \X_{\infty} \cap (\tau \circ \pi)^{-1}(\Xi))
	= H_*(\tau_*(\Sc_{\varphi_*}) \cap \Pcausal) \cap \pi_{\US^-}^{-1}(\Xi)
	= \StochCausalOut(\Xi)
\end{equation*}
for any $\Xi \in \Pin^-$.
Hence, if we find that the three hypotheses of \cref{prop_const_card} are satisfied, then the conclusion of \cref{prop_const_card} is exactly the desired statement.
Hypothesis (i):
Let $\mu,\mu' \in \Sc_{\varphi_*} \cap \X_{\infty}$ satisfy $H_*\tau_*T_*\mu = H_*\tau_*T_*\mu'$ and $(\pi_{\US})_*\mu = (\pi_{\US})_*\mu'$.
We need to show that $H_*\tau_*\mu = H_*\tau_*\mu'$.
Let $H' \colon \XS^- \times \US \rightarrow \YS^- \times \US$ be the extension of the readout determined by $\tau \circ H' = H \circ \tau$ and $\pi_{\US} \circ H' = \pi_{\US}$.
It is evident from this and \cref{rem_CondInd_transform} that $H'_*(\X_0)$ is a subset of the analogously defined $\Y_0 \subseteq P(\YS^- \times \US)$.
By the relations determining $H'$ and the hypotheses on $\mu$ and $\mu'$, we have $\tau_*H'_*T_*\mu = \tau_*H'_*T_*\mu'$ and $(\pi_{\US})_*H'_*T_*\mu = (\pi_{\US})_*H'_*T_*\mu'$.
\cref{lem_CondInd_injective} yields $H'_*T_*\mu = H'_*T_*\mu'$.
Let $g \colon \Yc \times \US^- \rightarrow \Yc$ be a measurable extension of the measurable map promised by the assumption that the system measurably distinguishes reachable states.
Consider the map $\Gc \colon \YS^- \times \US \rightarrow \YS^- \times \US^-$ given by $\Gc(\Seq{y},\Seq{u}) = ((\Seq{y},g(\seq{y}{0},\sigma(\Seq{u}))),(\tau \circ \sigma)(\Seq{u}))$.
This map satisfies $\Gc \circ H' \circ T(\Seq{x},\Seq{u}) = H \circ \tau(\Seq{x},\Seq{u})$ for all $(\Seq{x},\Seq{u}) \in \Sc_{\varphi}$.
This and \cref{prop_abstr_support_sol} yield
\begin{equation*}
	H_*\tau_*\mu
	= \Gc_*H'_*T_*\mu
	= \Gc_*H'_*T_*\mu'
	= H_*\tau_*\mu'.
\end{equation*}
Hypothesis (ii):
Let $\Xi,\Xi' \in \Pin$ satisfy $\tau_*\Xi = \tau_*\Xi'$, and suppose
$\mu \in \Sc_{\varphi_*} \cap \X_{\infty} \cap (\pi_{\US})_*^{-1}(\Xi')$.
We need to show that $\tau_*\mu \in \tau_*(\Sc_{\varphi_*} \cap \X_{\infty} \cap (\pi_{\US})_*^{-1}(\Xi))$.
By \cref{lem_causal_at_zero}.(iii), there exists a unique $\mu' \in \X_0$ with $\tau_*\mu' = \tau_*\mu$ and $(\pi_{\US})_*\mu' = \Xi$.
By \cref{lem_CondInd_shift}, we have $\mu' \in \Sc_{\varphi_*} \cap \X_{\infty}$, as desired.
Hypothesis (iii):
This is stipulated in \cref{assmpt_concat_conv} (recall that $\tau_*(\Pin) = \Pin^-$).
\end{proof}

\section{Conclusion}
\label{sec_conclusion}

We have developed a comprehensive theoretical framework for understanding fading memory and solution stability in state-space systems with both deterministic and stochastic inputs.
Our results show that, generically, these systems exhibit fading memory even in the absence of classical ESP conditions, thus explaining the success of RC models that do not satisfy standard contractivity requirements.

In the stochastic setting, we proposed a new distributional notion of a solution based on attractor dynamics on spaces of probability measures.
This approach naturally recovers the notion of a solution defined through almost sure equality of the state equation, and it leads to a coherent theory that captures essential features one expects from stochastic solutions.
We extend several results on the ESP and fading memory from the deterministic to the stochastic setting, and rigorously discuss differences that arise such as the intricacies of probabilistic causality.

Our work builds on and extends the abstract dynamical systems perspective initiated in earlier foundational studies, enabling a unified treatment that is both rigorous and broadly applicable.
Key concepts such as fading memory and causal structures appear naturally and can be embedded into other context where they appear naturally such as stochastic filtering (see \cref{sec_app_inference}).

We considered stability of solutions as functions of the input.
Future work may discuss robustness of the model --- both deterministically and stochastically.
When an unknown system is learned with a reservoir computing model trained on data, the learned model is a perturbation of the unknown model in function space.
Understanding the robustness of such a perturbation is a crucial aspect of mathematical learning theory.

\acks{%
The authors acknowledge partial financial support from the School of Physical and Mathematical Sciences of the Nanyang Technological University.
The second author is funded by an Eric and Wendy Schmidt AI in Science Postdoctoral Fellowship at the Nanyang Technological University.
The authors thank Lyudmila Grigoryeva, G.\ Manjunath, and Michael Greinecker for helpful comments and discussions.}

\appendix

\section{Lemmas from topology and measure theory}
\label{app_sec_technical}

A map is proper if it is closed and has compact fibers, and it is quasi-proper if preimages of compact sets are compact.
The following result is mostly standard in topology \cite[p.\ 97ff]{Bourbaki1998}.
We review a proof for the part that is not standard.

\begin{lemma}
\label{lem_proper}
	Let $\Xc$ and $\Yc$ be Hausdorff spaces and $f \colon \Xc \rightarrow \Yc$ be continuous.
Then, the following are equivalent.
\begin{enumerate}[\upshape (i)]\itemsep=0em
\item
The map $f$ is proper.
\item
The map $f$ is closed and quasi-proper.
\item
For any net $(x_i)_{i \in I} \subseteq \Xc$ and any cluster point $y \in \Yc$ of the net $(f(x_i))_{i \in I}$ there exists a cluster point $x \in \Xc$ of the net $(x_i)_{i \in I}$ with $f(x) = y$.
\end{enumerate}
Using the equivalence to {\upshape (iii)}, it becomes clear that if $f$ admits a continuous left-inverse, then $f$ is proper.
\end{lemma}

\begin{proof}
	It is straight-forward to see that (iii) $\Rightarrow$ (ii) $\Rightarrow$ (i).
We show that (i) $\Rightarrow$ (iii).
For any $i \in I$, let $A_i$ be the closure of the set $\{ x_j \colon j \geq i \}$ so that $y \in \overline{f(A_i)}$.
Since $f$ is closed, $\overline{f(A_i)} = f(A_i)$.
Thus, $A_i \cap f^{-1}(y)$ is non-empty.
These sets have the finite intersection property.
Indeed, given finitely many $i_1,\dots,i_N \in I$, there exists some $i \in I$ with $i_n \leq i$ for all $1 \leq n \leq N$, and then $A_{i_1} \cap \dots \cap A_{i_N} \cap f^{-1}(y)$ contains the non-empty set $A_i \cap f^{-1}(y)$.
Compactness of $f^{-1}(y)$ implies that $\bigcap_{i \in I} A_i \cap f^{-1}(y)$ is non-empty.
Take an element $x$ from that intersection.
If $x$ was not a cluster point of $(x_i)_{i \in I}$, then there existed some $i \in I$ and an open set $U \subseteq \Xc$ with $x \in U$ such that $x_j \notin U$ for all $j \geq i$.
But then $A_i \cap U$ were empty, contradicting $x \in A_i \cap U$.
\end{proof}

In the following lemma, we flesh out an application of the Kolmogorov extension theorem that we needed in the proof of \cref{prop_stoch_per_sol}.
Below, $\US^-$ and $\tau^{-1}(\US^-)$ are endowed with the product topologies.

\begin{lemma}
\label{lem_Kolmogorov}
	Let $\Uc$ be a Suslin space and $\US^- \subseteq \Uc^{\Z_-}$ be Borel measurable.
Denote the truncation $\Uc^{\Z} \rightarrow \Uc^{\Z_-}$ by $\tau$.
Suppose $\Xi^- \in P(\US^-)$ is periodic under shifts.
Then, there exists some $\Xi \in P(\tau^{-1}(\US^-))$ that is periodic under shifts with the same minimal period as $\Xi^-$ and satisfies $\tau_*\Xi = \Xi^-$.
\end{lemma}

\begin{proof}
	Let $T$ denote the right-shift operator on $\Uc^{\Z_-}$ and on $\Uc^{\Z}$.
Let $n  \in \N$ be the minimal period of $\Xi^-$ under $T_*$.
Given any finite subset $\Lambda \subseteq \Z$, let $p_{\Lambda} \colon \Uc^{\Z} \rightarrow \Uc^{|\Lambda|}$, $\Seq{u} \mapsto (\seq{u}{t})_{t \in \Lambda}$ and $p_{\Lambda}^- \colon \US^- \rightarrow \Uc^{|\Lambda|}$, $\Seq{u} \mapsto (\seq{u}{t-t_{\Lambda}})_{t \in \Lambda}$, where $t_{\Lambda} := \sup_{t \in \Lambda} \abs{nt}$.
Given any two nested finite subsets $\Lambda' \subseteq \Lambda \subseteq \Z$, let $k_{\Lambda',\Lambda} := (t_{\Lambda} - t_{\Lambda'})/n \in \N_0$ and $p_{\Lambda',\Lambda} \colon \Uc^{|\Lambda|} \rightarrow \Uc^{|\Lambda'|}$, $(\seq{u}{t})_{t \in \Lambda} \mapsto (\seq{u}{t})_{t \in \Lambda'}$.
Then, $p_{\Lambda',\Lambda} \circ p_{\Lambda}^- = p_{\Lambda'}^- \circ T^{n k_{\Lambda',\Lambda}}$.
This shows that the measures $\Xi_{\Lambda} := (p_{\Lambda}^-)_*\Xi^-$ satisfy the consistency condition of the Kolmogorov extension theorem \cite[Corollary 7.7.2]{Bogachev2007}.
By that theorem, there exists a measure $\Xi' \in P(\Uc^{\Z})$ that satisfies $(p_{\Lambda})_*\Xi' = \Xi_{\Lambda}$ for all finite subsets $\Lambda \subseteq \Z$.
Let us show that $T^n_*\Xi' = \Xi'$.
Since we work with the product topology of a Suslin space, it suffices to show that $(p_{\Lambda_k})_*T^n_*\Xi' = (p_{\Lambda_k})_*\Xi'$ for all $k \in \N$, where $\Lambda_k := \{-k,\dots,k\}$.
Fix $k \in \N$.
Let $\Lambda' := \{-k-n,\dots,k-n\}$ and $\Lambda := \Lambda_k \cup \Lambda'$.
Note that $p_{\Lambda_k} \circ T^n = p_{\Lambda'}$ and $p_{\Lambda'}^- = p_{\Lambda_k}^- \circ T^n$.
Therefore, $(p_{\Lambda_k})_*T^n_*\Xi' = \Xi_{\Lambda'} = \Xi_{\Lambda_k} = (p_{\Lambda_k})_*\Xi'$, as desired.
A similar argument shows that $\tau_*\Xi' = \Xi^-$.
In particular, the period of $\Xi'$ under $T_*$ cannot be smaller than $n$.
Finally, let $\Xi$ be the restriction of $\Xi'$ to $\tau^{-1}(\US^-)$, which is a probability measure since $\Xi'(\tau^{-1}(\US^-)) = \tau_*\Xi'(\US^-) = \Xi^-(\US^-) = 1$.
\end{proof}

For the sake of completeness, we provide a proof of \cref{lem_CondInd_repr}.

\begin{proof}[\Pf{lem_CondInd_repr}]
	This lemma is an adaptation of \cite[Proposition 8.20]{Kallenberg2021}.
Let us spell out the details.
To ensure the existence of an independent uniform random variable, consider the product probability space $(\Omega',\Ac',\Pb') := (\Omega \times [0,1],\Ac \vee \Bc([0,1]),\Pb \otimes \Lambda)$ and the projection $\vartheta \colon \Omega' \rightarrow [0,1]$, where $\Bc([0,1])$ is the Borel sigma-algebra on $[0,1]$.
Let $\Sigma_i' \subseteq \Ac'$ be the sigma-algebra generated by the extended maps $V_i' \colon \Omega' \rightarrow \Vc_i$, $(\omega,\lambda) \mapsto V_i(\omega)$, $i = 1,2,3$.
Note that $\Sigma_1 \CondInd{\Sigma_2} \Sigma_3$ wrt $\Pb$ if and only if $\Sigma_1' \CondInd{\Sigma_2'} \Sigma_3'$ wrt $\Pb'$.
Since $\vartheta$ is independent of $(V_1',V_2')$, we trivially have $\Sigma_1' \CondInd{\Sigma_2'} \Bc([0,1])$ wrt $\Pb'$.
By \cref{rem_CondInd_transform}.(i), it also holds that $\Sigma_1' \CondInd{\Sigma_2'} \Sigma_2' \vee \Bc([0,1])$ wrt $\Pb'$.
Thus, by \eqref{eq_CondInd_alt_def},
\begin{equation}
\label{lem_CondInd_repr_pf}
	\ExpC{\Pb'}{\beta(g(V_2',\vartheta))}{\Sigma_1' \vee \Sigma_2'}
	= \ExpC{\Pb'}{\beta(g(V_2',\vartheta))}{\Sigma_2'} \quad \Pb'\text{-a.s.}
\end{equation}
for any bounded measurable function $\beta \colon \Vc_3 \rightarrow \R$.
Now, suppose (ii) holds.
Given any bounded $\Sigma_3'$-measurable function $\beta_3' \colon \Omega' \rightarrow \R$, take a measurable function $\beta \colon \Vc_3 \rightarrow \R$ such that $\beta_3' = \beta \circ V_3'$.
Two applications of \eqref{eq_CondInd_transform} with the map $V_1' \times V_2' \times V_3'$, the hypothesis (ii), and \eqref{lem_CondInd_repr_pf} yield
\begin{equation*}
	\ExpC{\Pb'}{\beta_3'}{\Sigma_1' \vee \Sigma_2'}
	= \ExpC{\Pb'}{\beta(g(V_2',\vartheta))}{\Sigma_1' \vee \Sigma_2'}
	= \ExpC{\Pb'}{\beta(g(V_2',\vartheta))}{\Sigma_2'}
	= \ExpC{\Pb'}{\beta_3'}{\Sigma_2'} \quad \Pb'\text{-a.s.}
\end{equation*}
This shows that $\Sigma_1' \CondInd{\Sigma_2'} \Sigma_3'$ wrt $\Pb'$.
Conversely, suppose (i) holds.
Let $\hat{\Sigma}_i$ be the sigma-algebra on $\Vc_1 \times \Vc_2 \times \Vc_3$ generated by the projection onto $\Vc_i$.
Let $\beta \colon \Vc_3 \rightarrow \R$ be bounded and measurable.
By \eqref{eq_CondInd_transform}, we have
\begin{equation*}
	\ExpC{(V_1 \times V_2 \times V_3)_*\Pb}{\beta \circ \pi_{\Vc_3}}{\hat{\Sigma}_1 \vee \hat{\Sigma}_2} \circ (V_1 \times V_2 \times V_3)
	= \ExpC{\Pb'}{\beta \circ V_3'}{\Sigma_1' \vee \Sigma_2'} \quad \Pb'\text{-a.s.}
\end{equation*}
This, in turn, equals $\ExpC{\Pb'}{\beta \circ V_3'}{\Sigma_2'}$ $\Pb'$-a.s.\ by (i).
The hypothesis on the function $g$ yields
\begin{equation*}
	\ExpC{\Pb'}{\beta \circ V_3'}{\Sigma_2'}
	= \ExpC{\Pb'}{\beta(g(V_2',\vartheta))}{\Sigma_2'} \quad \Pb'\text{-a.s.}
\end{equation*}
By another application of \eqref{eq_CondInd_transform} and \eqref{lem_CondInd_repr_pf}, we find that
\begin{equation*}
	\ExpC{\Pb'}{\beta(g(V_2',\vartheta))}{\Sigma_2'}
	= \ExpC{G_*((V_1 \times V_2)_*\Pb \otimes \Lambda)}{\beta \circ \pi_{\Vc_3}}{\hat{\Sigma}_1 \vee \hat{\Sigma}_2} \circ G \quad \Pb'\text{-a.s.}
\end{equation*}
We have shown that
\begin{equation*}
	\ExpC{(V_1 \times V_2 \times V_3)_*\Pb}{\beta \circ \pi_{\Vc_3}}{\hat{\Sigma}_1 \vee \hat{\Sigma}_2} \circ (V_1 \times V_2 \times V_3)
	= \ExpC{G_*((V_1 \times V_2)_*\Pb \otimes \Lambda)}{\beta \circ \pi_{\Vc_3}}{\hat{\Sigma}_1 \vee \hat{\Sigma}_2} \circ G \quad \Pb'\text{-a.s.}
\end{equation*}
for any bounded measurable $\beta \colon \Vc_3 \rightarrow \R$.
The tower property of conditional expectation yields
\begin{equation*}
	\ExpS{(V_1 \times V_2 \times V_3)_*\Pb}{\beta_1 \beta_2 \beta_3}
	= \ExpS{G_*((V_1 \times V_2)_*\Pb \otimes \Lambda)}{\beta_1 \beta_2 \beta_3}
\end{equation*}
for any bounded $\hat{\Sigma}_i$-measurable $\beta_i \colon \Vc_1 \times \Vc_2 \times \Vc_3 \rightarrow \R$, $i=1,2,3$.
We conclude the lemma.
\end{proof}

\section{Statistical inference}
\label{sec_app_inference}

Consider a triple of stochastic processes $(\Seq{Y},\Seq{X},\Seq{U}) = (\seq{Y}{t},\seq{X}{t},\seq{U}{t})_{t \in \Z_-}$ with values in $\YS^- \times \XS^- \times \US^-$ such that the joint law of $(\Seq{X},\Seq{U})$ is a stochastic solution of the state-space system and the joint law of $(\Seq{Y},\Seq{U})$ a stochastic output.
Throughout the paper, we worked with outputs of the form $\seq{Y}{t} = h(\seq{X}{t})$ because this was sufficient to capture the echo state and fading memory properties we were interested in.
In practice, one often models the output by $\seq{Y}{t} = h(\seq{X}{t},\seq{V}{t})$ with an additional measurement noise $\seq{V}{t}$.
In (Bayesian) filtering, one is concerned with sampling from the conditional law of the hidden state $\seq{X}{0}$ given the observations $\Seq{Y}$.
The linear Gaussian case is well understood, that is, when the inputs $\seq{U}{t}$ and the measurement noise $\seq{V}{t}$ are i.i.d.\ Gaussian random variables and $f$ and $h$ are linear transformations.
This classical filtering problem is solved by the Kalman filter, pioneered by the eponymous Kalman \cite{Kalman1960}.
More general cases are typically tackled with the likes of the extended or unscented Kalman filters, which first linearize the transformations and approximate the inputs and measurement noise by i.i.d.\ Gaussians and then apply the classical Kalman filter \cite{Saerkkae2013};
or with particle filters and their various extensions, which pass carefully sampled `particles' through the non-linear transformations and weigh them appropriately \cite{KantasEtal2015}.
Crucially, the theoretical analyses of all these filtering methods assume the hidden states to be Markovian.
This is guaranteed if the inputs are mutually independent but fails in general if the inputs exhibit statistical interdependence.

Consider the more general case in which the inputs are not mutually independent but $\Seq{U}$ itself is Markovian.
Since past states typically depend on past inputs and the latter may exhibit interdependence with the current input, we cannot hope that the states $\Seq{X}$ are Markovian.
However, the sequence of augmented states $\seq{X}{t}' := (\seq{X}{t},\seq{U}{t})$ is Markovian if there is no direct interdependence between $\seq{U}{t}$ and past states, that is, $\Seq{X}'$ is Markovian if the joint law of $(\Seq{X},\Seq{U})$ is causal.
This highlights the need to understand the appearance of causal structures in stochastic solutions of state-space systems.
Although we expect real-world systems to be causal, theoretical justifications are lacking.
The states of a (random or noisy) real-world system, whose evolution is governed by a state map and which has long been evolving in the past, represent a stochastic solution in the forward attractor of the system, more precisely, in one of its omega limit sets.
If the underlying inputs of the system (random factors, noise, controls) are periodic, then we know from \cref{prop_omega_limit_struc_sol} that the system is indeed causal.
It is an interesting open problem to extend this result to non-periodic inputs.

Augmenting the states by the inputs and, going further, augmenting them by (finitely many) past states and past inputs is a known trick in filtering.
In particular, this has been done in \cite{LindstenEtal2012} to obtain an approximately Markovian model out of a non-Markovian one.
That the augmented states are indeed approximately Markovian relies on a fading memory assumption posed in that work.
Fading memory also appears as ground for some convergence results for filtering algorithms \cite{KantasEtal2015}.
In light of such results being enabled by fading memory assumptions, it is a critical theoretical contribution that we established generic fading memory of general state-space systems in \cref{thrm_det_ESP_FMP,thrm_stoch_ESP_FMP}.

\bib{acm}{bibfile_FR}


\begin{thebibliography}{10}

\bibitem{AllamEtal2021}
{\sc Allam, A., Feuerriegel, S., Rebhan, M., and Krauthammer, M.}
\newblock {Analyzing Patient Trajectories With Artificial Intelligence}.
\newblock {\em J Med Internet Res 23}, 12 (Dec 2021), e29812.

\bibitem{AmigoEtal2024}
{\sc Amig{\'o}, J.~M., Dale, R., King, J.~C., and Lehnertz, K.}
\newblock {Generalized synchronization in the presence of dynamical noise and
  its detection via recurrent neural networks}.
\newblock {\em Chaos: An Interdisciplinary Journal of Nonlinear Science 34}, 12
  (Dec 2024), 123156.

\bibitem{AppeltantEtAl2011}
{\sc Appeltant, L., Soriano, M.~C., Van~der Sande, G., Danckaert, J., Massar,
  S., Dambre, J., Schrauwen, B., Mirasso, C.~R., and Fischer, I.}
\newblock {Information processing using a single dynamical node as complex
  system}.
\newblock {\em Nature Communications 2}, 1 (2011), 468.

\bibitem{ArcomanoEtal2022}
{\sc Arcomano, T., Szunyogh, I., Wikner, A., Pathak, J., Hunt, B.~R., and Ott,
  E.}
\newblock {A hybrid approach to atmospheric modeling that combines machine
  learning with a physics-based numerical model}.
\newblock {\em Journal of Advances in Modeling Earth Systems 14}, 3 (2022),
  e2021MS002712.

\bibitem{Arnold1998}
{\sc Arnold, L.}
\newblock {\em {Random Dynamical Systems}}, 1~ed.
\newblock Springer Monographs in Mathematics. Springer Berlin, Heidelberg,
  1998.

\bibitem{AubinFrankowska2009}
{\sc Aubin, J.-P., and Frankowska, H.}
\newblock {\em {Set-Valued Analysis}}, 1~ed., vol.~234 of {\em Modern
  Birkh{\"a}user Classics}.
\newblock Birkh{\"a}user Boston, 2009.

\bibitem{BabinPilyugin1997}
{\sc Babin, A.~V., and Pilyugin, S.~Y.}
\newblock {Continuous dependence of attractors on the shape of domain}.
\newblock {\em Journal of Mathematical Sciences 87}, 2 (Nov 1997), 3304--3310.
\newblock transl. from Zap. Nauchn. Sem. POMI, 221 (1995), p. 58--66.

\bibitem{BallarinEtal2024}
{\sc Ballarin, G., Dellaportas, P., Grigoryeva, L., Hirt, M., {van Huellen},
  S., and Ortega, J.-P.}
\newblock {Reservoir computing for macroeconomic forecasting with
  mixed-frequency data}.
\newblock {\em International Journal of Forecasting 40}, 3 (2024), 1206--1237.

\bibitem{BerryDas2023}
{\sc Berry, T., and Das, S.}
\newblock {Learning Theory for Dynamical Systems}.
\newblock {\em SIAM Journal on Applied Dynamical Systems 22}, 3 (2023),
  2082--2122.

\bibitem{Bogachev2007}
{\sc Bogachev, V.~I.}
\newblock {\em {Measure Theory}}, 1~ed.
\newblock Springer-Verlag, Berlin Heidelberg, 2007.

\bibitem{BogmansEtal2025}
{\sc Bogmans, C., Gomez-Gonzalez, P., Ganpurev, G., Melina, G., Pescatori, A.,
  and Thube, S.}
\newblock {Power Hungry: How AI Will Drive Energy Demand}.
\newblock {\em IMF Working Papers 2025}, 81 (4 2025), 32.

\bibitem{Bourbaki1998}
{\sc Bourbaki, N.}
\newblock {\em {General Topology}}, 1~ed.
\newblock Springer Berlin, Heidelberg, 1998.

\bibitem{BoydChua1985}
{\sc Boyd, S., and Chua, L.}
\newblock {Fading memory and the problem of approximating nonlinear operators
  with Volterra series}.
\newblock {\em IEEE Transactions on Circuits and Systems 32}, 11 (1985),
  1150--1161.

\bibitem{BuehlerEtal2020}
{\sc Buehler, H., Horvath, B., Lyons, T., Perez~Arribas, I., and Wood, B.}
\newblock {Generating Financial Markets With Signatures}.
\newblock {\em SSRN:3657366\/} (2020).

\bibitem{BuehnerYoung2006}
{\sc Buehner, M., and Young, P.}
\newblock {A tighter bound for the echo state property}.
\newblock {\em IEEE Transactions on Neural Networks 17}, 3 (2006), 820--824.

\bibitem{CeniEtal2020PhysD}
{\sc Ceni, A., Ashwin, P., Livi, L., and Postlethwaite, C.}
\newblock {The echo index and multistability in input-driven recurrent neural
  networks}.
\newblock {\em Physica D: Nonlinear Phenomena 412\/} (2020), 132609.

\bibitem{CheEckstein2025}
{\sc Cheridito, P., and Eckstein, S.}
\newblock {Optimal transport and Wasserstein distances for causal models}.
\newblock {\em Bernoulli 31}, 2 (2025), 1351 -- 1376.

\bibitem{ChowTeicher1997}
{\sc Chow, Y.~S., and Teicher, H.}
\newblock {\em {Probability Theory: Independence, Interchangeability,
  Martingales}}, 3~ed.
\newblock Springer Texts in Statistics. Springer New York, 1997.

\bibitem{ChuaGreen1976}
{\sc Chua, L., and Green, D.}
\newblock {A qualitative analysis of the behavior of dynamic nonlinear
  networks: Steady-state solutions of nonautonomous networks}.
\newblock {\em IEEE Transactions on Circuits and Systems 23}, 9 (1976),
  530--550.

\bibitem{CrauelFlandoli1994}
{\sc Crauel, H., and Flandoli, F.}
\newblock {Attractors for random dynamical systems}.
\newblock {\em Probability Theory and Related Fields 100}, 3 (Sep 1994),
  365--393.

\bibitem{DeshengKloeden2004}
{\sc Desheng, L., and Kloeden, P.~E.}
\newblock {Equi-attraction and the continuous dependence of attractors on
  parameters}.
\newblock {\em Glasgow Mathematical Journal 46}, 1 (2004), 131–141.

\bibitem{Engelking1989}
{\sc Engelking, R.}
\newblock {\em {General Topology}}, vol.~6 of {\em Sigma Series in Pure
  Mathematics}.
\newblock Heldermann Verlag Berlin, 1989.

\bibitem{ErogluLambPereira2017}
{\sc Eroglu, D., Lamb, J. S.~W., and Pereira, T.}
\newblock {Synchronisation of chaos and its applications}.
\newblock {\em Contemporary Physics 58}, 3 (2017), 207--243.

\bibitem{FernandoSojakka2003}
{\sc Fernando, C., and Sojakka, S.}
\newblock {Pattern Recognition in a Bucket}.
\newblock In {\em Advances in Artificial Life\/} (2003), W.~Banzhaf,
  J.~Ziegler, T.~Christaller, P.~Dittrich, and J.~T. Kim, Eds., Springer Berlin
  Heidelberg, pp.~588--597.

\bibitem{FrancqZakoian2019}
{\sc Francq, C., and Zakoian, J.-M.}
\newblock {\em {GARCH models}}, 2~ed.
\newblock John Wiley \& Sons, 2019.

\bibitem{RC12}
{\sc Gonon, L., Grigoryeva, L., and Ortega, J.-P.}
\newblock {Approximation bounds for random neural networks and reservoir
  systems}.
\newblock {\em The Annals of Applied Probability 33}, 1 (2023), 28--69.

\bibitem{RC20}
{\sc Gonon, L., and Ortega, J.-P.}
\newblock {Fading memory echo state networks are universal}.
\newblock {\em Neural Networks 138\/} (2021), 10--13.

\bibitem{Graves2014}
{\sc Graves, A.}
\newblock {Generating Sequences With Recurrent Neural Networks}.
\newblock {\em arXiv:1308.0850v5\/} (2014).

\bibitem{RC18}
{\sc Grigoryeva, L., Hart, A.~G., and Ortega, J.-P.}
\newblock {Chaos on compact manifolds: Differentiable synchronizations beyond
  the Takens theorem}.
\newblock {\em Physical Review E - Statistical Physics, Plasmas, Fluids, and
  Related Interdisciplinary Topics 103\/} (2021), 062204.

\bibitem{RC21}
{\sc Grigoryeva, L., Hart, A.~G., and Ortega, J.-P.}
\newblock {Learning strange attractors with reservoir systems}.
\newblock {\em Nonlinearity 36\/} (2023), 4674--4708.

\bibitem{RC7}
{\sc Grigoryeva, L., and Ortega, J.-P.}
\newblock {Echo state networks are universal}.
\newblock {\em Neural Networks 108\/} (2018), 495--508.

\bibitem{RC6}
{\sc Grigoryeva, L., and Ortega, J.-P.}
\newblock {Universal discrete-time reservoir computers with stochastic inputs
  and linear readouts using non-homogeneous state-affine systems}.
\newblock {\em Journal of Machine Learning Research 19}, 24 (2018), 1--40.

\bibitem{RC9}
{\sc Grigoryeva, L., and Ortega, J.-P.}
\newblock {Differentiable reservoir computing}.
\newblock {\em Journal of Machine Learning Research 20}, 179 (2019), 1--62.

\bibitem{RC16}
{\sc Grigoryeva, L., and Ortega, J.-P.}
\newblock {Dimension reduction in recurrent networks by canonicalization}.
\newblock {\em Journal of Geometric Mechanics 13}, 4 (2021), 647--677.

\bibitem{GuDao2023}
{\sc Gu, A., and Dao, T.}
\newblock {Mamba: Linear-Time Sequence Modeling with Selective State Spaces}.
\newblock {\em arXiv:2312.00752v2\/} (2023).

\bibitem{HewamalageEtal2021}
{\sc Hewamalage, H., Bergmeir, C., and Bandara, K.}
\newblock {Recurrent Neural Networks for Time Series Forecasting: Current
  status and future directions}.
\newblock {\em International Journal of Forecasting 37}, 1 (2021), 388--427.

\bibitem{HoangOlsonRobinson2015}
{\sc Hoang, L.~T., Olson, E.~J., and Robinson, J.~C.}
\newblock {On the continuity of global attractors}.
\newblock {\em Proceedings of the American Mathematical Society 143}, 10
  (2015), 4389--4395.

\bibitem{JacodMemin1981}
{\sc Jacod, J., and Memin, J.}
\newblock {Sur un type de convergence intermediaire entre la convergence en loi
  et la convergence en probabilite}.
\newblock In {\em S{\'e}minaire de Probabilit{\'e}s XV 1979/80\/} (1981),
  J.~Az{\'e}ma and M.~Yor, Eds., vol.~850 of {\em Lecture Notes in
  Mathematics}, Springer Berlin Heidelberg, pp.~529--546.

\bibitem{Jaeger2010}
{\sc Jaeger, H.}
\newblock {The ``echo state'' approach to analysing and training recurrent
  neural networks -- with an Erratum note}.
\newblock Tech. Rep. GMD Report 148, German National Research Center for
  Information Technology, 2010.

\bibitem{JaegerHaas2004}
{\sc Jaeger, H., and Haas, H.}
\newblock {Harnessing nonlinearity: Predicting chaotic systems and saving
  energy in wireless communication}.
\newblock {\em Science 304}, 5667 (2004), 78--80.

\bibitem{JiangLiLiWang2023JML}
{\sc Jiang, H., Li, Q., Li, Z., and Wang, S.}
\newblock {A Brief Survey on the Approximation Theory for Sequence Modelling}.
\newblock {\em Journal of Machine Learning 2}, 1 (2023), 1--30.

\bibitem{JiangSonneEtal2024}
{\sc Jiang, P., Sonne, C., Li, W., You, F., and You, S.}
\newblock {Preventing the Immense Increase in the Life-Cycle Energy and Carbon
  Footprints of LLM-Powered Intelligent Chatbots}.
\newblock {\em Engineering 40\/} (2024), 202--210.

\bibitem{Kallenberg2021}
{\sc Kallenberg, O.}
\newblock {\em {Foundations of Modern Probability}}, 3~ed., vol.~99 of {\em
  Probability Theory and Stochastic Modelling}.
\newblock Springer Nature Switzerland, 2021.

\bibitem{Kalman1960}
{\sc Kalman, R.~E.}
\newblock {A new approach to linear filtering and prediction problems}.
\newblock {\em Journal of Basic Engineering 82}, 1 (1960), 35--45.

\bibitem{KantasEtal2015}
{\sc Kantas, N., Doucet, A., Singh, S.~S., Maciejowski, J., and Chopin, N.}
\newblock {On Particle Methods for Parameter Estimation in State-Space Models}.
\newblock {\em Statistical Science 30}, 3 (2015), 328 -- 351.

\bibitem{KloedenEtal2012}
{\sc Kloeden, P.~E., P{\"o}tzsche, C., and Rasmussen, M.}
\newblock {Limitations of pullback attractors for processes}.
\newblock {\em Journal of Difference Equations and Applications 18}, 4 (2012),
  693--701.

\bibitem{KloedenEtal2013}
{\sc Kloeden, P.~E., P{\"o}tzsche, C., and Rasmussen, M.}
\newblock {Discrete-Time Nonautonomous Dynamical Systems}.
\newblock In {\em Stability and Bifurcation Theory for Non-Autonomous
  Differential Equations}, R.~Johnson and M.~Pera, Eds., vol.~2065 of {\em
  Lecture Notes in Mathematics}. Springer Berlin, Heidelberg, 2013,
  pp.~35--102.

\bibitem{KloedenRasmussen2011}
{\sc Kloeden, P.~E., and Rasmussen, M.}
\newblock {\em {Nonautonomous Dynamical Systems}}, vol.~176 of {\em
  Mathematical Surveys and Monographs}.
\newblock American Mathematical Society, 2011.

\bibitem{KocarevParlitz1996}
{\sc Kocarev, L., and Parlitz, U.}
\newblock {Generalized Synchronization, Predictability, and Equivalence of
  Unidirectionally Coupled Dynamical Systems}.
\newblock {\em Phys. Rev. Lett. 76\/} (Mar 1996), 1816--1819.

\bibitem{KourouEtal2015}
{\sc Kourou, K., Exarchos, T.~P., Exarchos, K.~P., Karamouzis, M.~V., and
  Fotiadis, D.~I.}
\newblock {Machine learning applications in cancer prognosis and prediction}.
\newblock {\em Computational and Structural Biotechnology Journal 13\/} (2015),
  8--17.

\bibitem{LangkvistEtal2014}
{\sc L{\"a}ngkvist, M., Karlsson, L., and Loutfi, A.}
\newblock {A review of unsupervised feature learning and deep learning for
  time-series modeling}.
\newblock {\em Pattern Recognition Letters 42\/} (2014), 11--24.

\bibitem{LargerEtAl2012}
{\sc Larger, L., Soriano, M.~C., Brunner, D., Appeltant, L., Gutierrez, J.~M.,
  Pesquera, L., Mirasso, C.~R., and Fischer, I.}
\newblock {Photonic information processing beyond Turing: an optoelectronic
  implementation of reservoir computing}.
\newblock {\em Opt. Express 20}, 3 (Jan 2012), 3241--3249.

\bibitem{Lauritzen1996}
{\sc Lauritzen, S.}
\newblock {\em {Graphical Models}}, vol.~17 of {\em Oxford Statistical Science
  Series}.
\newblock Clarendon Press, Oxford, 1996.

\bibitem{Lauritzen2024}
{\sc Lauritzen, S.}
\newblock {Total variation convergence preserves conditional independence}.
\newblock {\em Statistics \& Probability Letters 214\/} (2024), 110200.

\bibitem{LeCun2015}
{\sc LeCun, Y., Bengio, Y., and Hinton, G.}
\newblock {Deep learning}.
\newblock {\em Nature 521}, 7553 (2015), 436--444.

\bibitem{LimZohren2021}
{\sc Lim, B., and Zohren, S.}
\newblock {Time-series forecasting with deep learning: a survey}.
\newblock {\em Philosophical Transactions of the Royal Society A 379}, 2194
  (2021), 20200209.

\bibitem{LindstenEtal2012}
{\sc Lindsten, F., Sch\"{o}n, T., and Jordan, M.}
\newblock {Ancestor Sampling for Particle Gibbs}.
\newblock In {\em Advances in Neural Information Processing Systems\/} (2012),
  F.~Pereira, C.~Burges, L.~Bottou, and K.~Weinberger, Eds., vol.~25, Curran
  Associates, Inc., pp.~1--9.

\bibitem{LuEtal2025}
{\sc Lu, Y., Chen, L., Zhang, Y., Shen, M., Wang, H., Wang, X., {van Rechem},
  C., Fu, T., and Wei, W.}
\newblock {Machine Learning for Synthetic Data Generation: A Review}.
\newblock {\em arXiv:2302.04062v10\/} (2025).

\bibitem{LuHuntOtt2018Chaos}
{\sc Lu, Z., Hunt, B.~R., and Ott, E.}
\newblock {Attractor reconstruction by machine learning}.
\newblock {\em Chaos 28}, 6 (2018).

\bibitem{Maass2011}
{\sc Maass, W.}
\newblock {Liquid State Machines: Motivation, Theory, and Applications}.
\newblock In {\em Computability In Context: Computation And Logic In The Real
  World}, S.~B. Cooper and A.~Sorbi, Eds. Imperial College Press, 2011,
  pp.~275--296.

\bibitem{MaassNatschMarkram2002}
{\sc Maass, W., Natschläger, T., and Markram, H.}
\newblock {Real-Time Computing Without Stable States: A New Framework for
  Neural Computation Based on Perturbations}.
\newblock {\em Neural Computation 14}, 11 (Nov 2002), 2531--2560.

\bibitem{Manjunath2020ProcA}
{\sc Manjunath, G.}
\newblock {Stability and memory-loss go hand-in-hand: three results in dynamics
  and computation}.
\newblock {\em Proceedings of the Royal Society A 476\/} (2020), 20200563.

\bibitem{Manjunath2022Nonlin}
{\sc Manjunath, G.}
\newblock {Embedding information onto a dynamical system}.
\newblock {\em Nonlinearity 35}, 3 (Jan 2022), 1131.

\bibitem{ManjunathJaeger2013}
{\sc Manjunath, G., and Jaeger, H.}
\newblock {Echo State Property Linked to an Input: Exploring a Fundamental
  Characteristic of Recurrent Neural Networks}.
\newblock {\em Neural Computation 25}, 3 (2013), 671--696.

\bibitem{ManjunathJaeger2014}
{\sc Manjunath, G., and Jaeger, H.}
\newblock {The Dynamics of Random Difference Equations Is Remodeled by Closed
  Relations}.
\newblock {\em SIAM Journal on Mathematical Analysis 46}, 1 (2014), 459--483.

\bibitem{RC27}
{\sc Manjunath, G., and Ortega, J.-P.}
\newblock {Transport in reservoir computing}.
\newblock {\em Physica D: Nonlinear Phenomena 449\/} (2023), 133744.

\bibitem{QRC1}
{\sc Mart\'{\i}nez-Pe\~na, R., and Ortega, J.-P.}
\newblock {Quantum reservoir computing in finite dimensions}.
\newblock {\em Phys. Rev. E 107\/} (Mar 2023), 035306.

\bibitem{Milnor1985}
{\sc Milnor, J.}
\newblock {On the concept of attractor}.
\newblock {\em Communications in Mathematical Physics 99}, 2 (1985), 177--195.

\bibitem{Nakajima2020}
{\sc Nakajima, K.}
\newblock {Physical reservoir computing—an introductory perspective}.
\newblock {\em Japanese Journal of Applied Physics 59}, 6 (May 2020), 060501.

\bibitem{OljaAshwinRasmussen2024}
{\sc Olja{\v{c}}a, L., Ashwin, P., and Rasmussen, M.}
\newblock {Measure and Statistical Attractors for Nonautonomous Dynamical
  Systems}.
\newblock {\em Journal of Dynamics and Differential Equations 36}, 3 (Sep
  2024), 2375--2411.

\bibitem{RC30}
{\sc Ortega, J.-P., and Rossmannek, F.}
\newblock {Fading Memory and the Convolution Theorem}.
\newblock {\em IEEE Transactions on Automatic Control 70}, 12 (2025),
  7830--7842.

\bibitem{RC28}
{\sc Ortega, J.-P., and Rossmannek, F.}
\newblock {State-space systems as dynamic generative models}.
\newblock {\em Proceedings of the Royal Society A 481}, 2309 (2025), 20240308.

\bibitem{RC32}
{\sc Ortega, J.-P., and Rossmannek, F.}
\newblock {Echoes of the Past: A Unified Perspective on Fading Memory and Echo
  States}.
\newblock {\em Neural Computation 38}, 5 (04 2026), 765--782.

\bibitem{PratelliRigo2023}
{\sc Pratelli, L., and Rigo, P.}
\newblock {A Strong Version of the Skorohod Representation Theorem}.
\newblock {\em Journal of Theoretical Probability 36}, 1 (Mar 2023), 372--389.

\bibitem{Saerkkae2013}
{\sc S{\"a}rkk{\"a}, S.}
\newblock {\em {Bayesian Filtering and Smoothing}}, 1~ed., vol.~3 of {\em
  Institute of Mathematical Statistics Textbooks}.
\newblock Cambridge University Press, 2013.

\bibitem{Schmidhuber2015}
{\sc Schmidhuber, J.}
\newblock {Deep learning in neural networks: An overview}.
\newblock {\em Neural Networks 61\/} (2015), 85--117.

\bibitem{Sontag1990}
{\sc Sontag, E.~D.}
\newblock {\em {Mathematical Control Theory: Deterministic Finite Dimensional
  Systems}}, 1~ed.
\newblock Texts in Applied Mathematics. Springer-Verlag New York, 1990.

\bibitem{TakahashiEtal2019}
{\sc Takahashi, S., Chen, Y., and Tanaka-Ishii, K.}
\newblock {Modeling financial time-series with generative adversarial
  networks}.
\newblock {\em Physica A: Statistical Mechanics and its Applications 527\/}
  (2019), 121261.

\bibitem{Takens1981}
{\sc Takens, F.}
\newblock {Detecting strange attractors in turbulence}.
\newblock In {\em Dynamical Systems and Turbulence}, D.~Rand and L.-S. Young,
  Eds., 1~ed., vol.~898 of {\em Lecture Notes in Mathematics}. Springer Berlin,
  Heidelberg, 1981, pp.~366--381.

\bibitem{TaylorNitschke2018}
{\sc Taylor, L., and Nitschke, G.}
\newblock {Improving Deep Learning with Generic Data Augmentation}.
\newblock In {\em 2018 IEEE Symposium Series on Computational Intelligence
  (SSCI)\/} (2018), pp.~1542--1547.

\bibitem{VerstraetenEtal2007}
{\sc Verstraeten, D., Schrauwen, B., D’Haene, M., and Stroobandt, D.}
\newblock {An experimental unification of reservoir computing methods}.
\newblock {\em Neural Networks 20}, 3 (2007), 391--403.

\bibitem{WenEtal2021}
{\sc Wen, Q., Sun, L., Yang, F., Song, X., Gao, J., Wang, X., and Xu, H.}
\newblock {Time Series Data Augmentation for Deep Learning: A Survey}.
\newblock In {\em Proceedings of the Thirtieth International Joint Conference
  on Artificial Intelligence, {IJCAI-21}\/} (8 2021), Z.-H. Zhou, Ed.,
  International Joint Conferences on Artificial Intelligence Organization,
  pp.~4653--4660.

\bibitem{WiknerEtal2020Chaos}
{\sc Wikner, A., Pathak, J., Hunt, B., Girvan, M., Arcomano, T., Szunyogh, I.,
  Pomerance, A., and Ott, E.}
\newblock {Combining machine learning with knowledge-based modeling for
  scalable forecasting and subgrid-scale closure of large, complex,
  spatiotemporal systems}.
\newblock {\em Chaos: An Interdisciplinary Journal of Nonlinear Science 30}, 5
  (May 2020), 053111.

\bibitem{YildizJaegerKiebel2012}
{\sc Yildiz, I.~B., Jaeger, H., and Kiebel, S.~J.}
\newblock {Re-visiting the echo state property}.
\newblock {\em Neural Networks 35\/} (2012), 1--9.

\end{thebibliography}
\end{document}